\documentclass[12pt]{article}

\usepackage{amsmath,amssymb}
\usepackage[T1]{fontenc}
\usepackage[utf8]{inputenc}
\usepackage{textcomp}
\usepackage{lmodern}
\IfFileExists{upquote.sty}{\usepackage{upquote}}{}
\IfFileExists{microtype.sty}{%
    \usepackage[]{microtype}
    \UseMicrotypeSet[protrusion]{basicmath}
}{}
\makeatletter
\@ifundefined{KOMAClassName}{%
    \IfFileExists{parskip.sty}{%
        \usepackage{parskip}
    }{%
        \setlength{\parindent}{0pt}
        \setlength{\parskip}{6pt plus 2pt minus 1pt}}%
}{%
    \KOMAoptions{parskip=half}}
\makeatother
\usepackage[dvipsnames,svgnames,x11names]{xcolor}
\makeatletter
\ifx\paragraph\undefined\else
    \let\oldparagraph\paragraph
    \renewcommand{\paragraph}{
        \@ifstar
            \xxxParagraphStar
            \xxxParagraphNoStar
    }
    \newcommand{\xxxParagraphStar}[1]{\oldparagraph*{#1}\mbox{}}
    \newcommand{\xxxParagraphNoStar}[1]{\oldparagraph{#1}\mbox{}}
\fi
\ifx\subparagraph\undefined\else
    \let\oldsubparagraph\subparagraph
    \renewcommand{\subparagraph}{
        \@ifstar
            \xxxSubParagraphStar
            \xxxSubParagraphNoStar
    }
    \newcommand{\xxxSubParagraphStar}[1]{\oldsubparagraph*{#1}\mbox{}}
    \newcommand{\xxxSubParagraphNoStar}[1]{\oldsubparagraph{#1}\mbox{}}
\fi
\makeatother

\usepackage{longtable,booktabs,array}
\usepackage{calc}
\usepackage{etoolbox}
\makeatletter
\patchcmd\longtable{\par}{\if@noskipsec\mbox{}\fi\par}{}{}
\makeatother
\IfFileExists{footnotehyper.sty}{\usepackage{footnotehyper}}{\usepackage{footnote}}
\makesavenoteenv{longtable}
\usepackage{graphicx}
\makeatletter
\def\maxwidth{\ifdim\Gin@nat@width>\linewidth\linewidth\else\Gin@nat@width\fi}
\def\maxheight{\ifdim\Gin@nat@height>\textheight\textheight\else\Gin@nat@height\fi}
\makeatother
\setkeys{Gin}{width=\maxwidth,height=\maxheight,keepaspectratio}
\makeatletter
\def\fps@figure{htbp}
\makeatother

\makeatletter
\@ifpackageloaded{caption}{}{\usepackage{caption}}
\AtBeginDocument{%
\ifdefined\contentsname
    \renewcommand*\contentsname{Table of contents}
\else
    \newcommand\contentsname{Table of contents}
\fi
\ifdefined\listfigurename
    \renewcommand*\listfigurename{List of Figures}
\else
    \newcommand\listfigurename{List of Figures}
\fi
\ifdefined\listtablename
    \renewcommand*\listtablename{List of Tables}
\else
    \newcommand\listtablename{List of Tables}
\fi
\ifdefined\figurename
    \renewcommand*\figurename{Figure}
\else
    \newcommand\figurename{Figure}
\fi
\ifdefined\tablename
    \renewcommand*\tablename{Table}
\else
    \newcommand\tablename{Table}
\fi
}
\@ifpackageloaded{float}{}{\usepackage{float}}
\floatstyle{ruled}
\@ifundefined{c@chapter}{\newfloat{codelisting}{h}{lop}}{\newfloat{codelisting}{h}{lop}[chapter]}
\floatname{codelisting}{Listing}

\makeatother
\makeatletter
\@ifpackageloaded{caption}{}{\usepackage{caption}}
\@ifpackageloaded{subcaption}{}{\usepackage{subcaption}}
\makeatother

\usepackage{enumitem}
\usepackage{bbm}
\usepackage{bm}
\usepackage{amsthm}
\usepackage{siunitx}
\usepackage{multirow}
\usepackage[toc,page]{appendix}
\usepackage{comment}
\usepackage{tikz}

\usetikzlibrary{shapes,decorations,arrows,calc,arrows.meta,fit,positioning}
\tikzset{
    -Latex,auto,node distance =1 cm and 1 cm,semithick,
    state/.style ={ellipse, draw, minimum width = 0.7 cm},
    point/.style = {circle, draw, inner sep=0.04cm,fill,node contents={}},
    bidirected/.style={Latex-Latex,dashed},
    el/.style = {inner sep=2pt, align=right, sloped}
}

\usepackage{algorithm}
\usepackage{chngcntr} 
\usepackage[noend]{algpseudocode}
\algrenewcommand\textproc{}
\algdef{SE}[SUBALG]{Indent}{EndIndent}{}{\algorithmicend\ }%
\algtext*{Indent}
\algtext*{EndIndent}

\newcommand{\Var}{{\mbox{Var}}}

\DeclareMathOperator*{\argmin}{arg\,min}

\newcommand\independent{\protect\mathpalette{\protect\independenT}{\perp}}
\def\independenT#1#2{\mathrel{\rlap{$#1#2$}\mkern2mu{#1#2}}}

\newtheorem{theorem}{Theorem}
\newtheorem{prop}{Proposition}
\newtheorem{lemma}{Lemma}

\newtheorem{assumption}{Assumption}[section]
\newtheorem{corollary}{Corollary}

\usepackage{dsfont}

\usepackage[]{natbib}
\usepackage[unicode]{hyperref}
\usepackage{bookmark}

\IfFileExists{xurl.sty}{\usepackage{xurl}}{}
\hypersetup{
    pdftitle={Learning from the Unseen: Offline Reinforcement Learning with Hidden Actions},
    pdfauthor={},
    pdfkeywords={off-policy evaluation, latent action models, semiparametric inference, multiple robustness},
    colorlinks=true,
    linkcolor={blue},
    filecolor={Maroon},
    citecolor={Blue},
    urlcolor={Blue},
    pdfcreator={LaTeX via pandoc}}

\newcommand{\RNum}[1]{\uppercase\expandafter{\romannumeral #1\relax}}

\newcommand{\E}{\mathbb{E}}

\newcommand{\nind}{\mathrel{\not\perp}}

\usepackage{hyperref}

\hypersetup{
	unicode=false,
	pdftoolbar=true,
	pdfmenubar=true,
	pdffitwindow=false,     
	pdfstartview={FitH},    
	pdfkeywords={}, 
	pdfnewwindow=true,      
	colorlinks=true,       
	linkcolor=cyan,          
	citecolor=blue,        
	filecolor=pink,      
	urlcolor=cyan           
}

\usepackage[utf8]{inputenc}

\allowdisplaybreaks

\newcommand{\neighbor}[1]%
{\overline{#1}}

\usepackage{mathtools}
\graphicspath{{fig/}}

\newcommand{\anon}{1}

\begin{document}
\def\spacingset#1{\renewcommand{\baselinestretch}{#1}\small\normalsize}
\spacingset{1}

\if1\anon
{
    \csname title\endcsname{\bf Learning from the Unseen: Offline Reinforcement Learning with Hidden Actions}
 \author{
Zeyu Bian$^{1\dagger}$,
Ying Zhou$^{2\dagger*}$,
and Yifan Cui$^{3*}$\\[0.5em]
$^1$ Department of Statistics, Florida State University \\
 $^2$ Department of Statistics, University of Connecticut \\
$^3$ Center for Data Science, Zhejiang University \\[0.5em]
\small $^\dagger$ Equal contribution. \\
\small $^*$ Corresponding authors: yzhou@uconn.edu; cuiyf@zju.edu.cn.
}
    \date{}
    \maketitle
} \fi

\if0\anon
{
    \bigskip
    \bigskip
    \bigskip
    \begin{center}
        {\LARGE\bf Learning from the Unseen: Offline Reinforcement Learning with Hidden Actions}
    \end{center}
    \medskip
} \fi

\bigskip
\begin{abstract}
Standard offline reinforcement learning (RL) algorithms typically assume that the actions in the dataset are observed without error. However, in many real-world applications, the true actions are unobserved and only noisy proxies are available, causing existing RL methods to yield biased and potentially misleading conclusions. We study off-policy evaluation  in infinite-horizon discounted Markov decision processes with hidden actions. By leveraging the next-state variable as a natural proxy for the unobserved action, we establish identification of the policy value and propose an influence-function-based estimator: LURE (\textbf{L}earning from the \textbf{U}nseen: \textbf{R}obust \textbf{E}stimator). LURE is multiply robust—remaining consistent under several combinations of correctly specified nuisance components—and is asymptotically normal, enabling valid statistical inference. To our knowledge, this is the first work to address offline RL with hidden actions. We demonstrate LURE’s effectiveness via simulations and a sepsis management application using the MIMIC-III database.
\end{abstract}

\noindent%
\vfill

\newpage
\spacingset{1.8}

\section{Introduction} \label{sec:intro}

Reinforcement learning \citep[RL;][]{sutton2018} aims to maximize the long-run cumulative reward in sequential decision-making by learning a policy that maps states to distributions over actions. In recent years, RL has been successfully applied to a wide range of domains, including game playing \citep{silver2016mastering}, robotics \citep{levine2020offline,zhong2026risk}, business strategy \citep{den2015dynamic,bian2024tale,liu2025contextual}, healthcare \citep{liao2022batch,zhou2024federated,chen2026data}, and social science \citep{li2026reinforcement}. Recent advances in large language models have also drawn heavily on RL techniques; see, for example, \citet{ouyang2022training,rafailov2023direct}.

RL is also closely connected to the statistical literature on dynamic treatment regimes and precision medicine \citep{Murphy,gest,qian2011performance,zhao2012estimating,zhang2012robust,wallace2015doubly,wang2018quantile,kitagawa2018should,shi2018high,qi2020multi,athey2021policy,cui2021semiparametric,bian2023variable, shen2024statistical}. This line of work studies how to tailor treatment decisions to individual patients so as to optimize expected clinical outcomes, typically over a finite number of treatment stages. While the goals are closely aligned with those of RL, most existing statistical methods for precision medicine are not designed for large- or infinite-horizon problems.

In this paper, we focus on off-policy evaluation (OPE), a fundamental problem in offline reinforcement learning. The goal of OPE is to estimate the long-run value of a target policy using historical data generated under a potentially different behavior policy. Existing model-free OPE methods broadly fall into three classes: direct methods \citep{ernst2005tree,le2019batch,shi2022statistical,chen2022well}, importance-sampling-based methods \citep{precup2000eligibility,liu2018breaking,nachum2019dualdice,uehara2020minimax,wang2023projected}, and doubly robust methods \citep{jiang2016doubly,thomas2016data,kallus2022efficiently,liao2022batch}. See \citet{uehara2022review} for a comprehensive overview.

Most of the existing literature assumes that the action recorded in the offline dataset is observed without error. In many applications, however, the action of scientific or operational interest may be partially observed or entirely latent \citep{young2018measurement,zhu2022causal}. Such discrepancies can arise from annotation errors, imperfect management systems, indirect action reconstruction from high-dimensional records, privacy-preserving data collection that releases only a coarsened proxy, or an improperly discretized action space. If these discrepancies are ignored, standard OPE methods can target the wrong state--action mechanism and yield misleading policy evaluations.

Our work is motivated by the longitudinal sepsis data from the MIMIC-III \citep{johnson2016mimic}. This large electronic health record database combines multiple heterogeneous sources and is characterized by irregular measurements, substantial missingness, and complex preprocessing pipelines \citep{gupta2022extensive}. As in many EHR-based studies, these features increase the risk of measurement error in downstream analyses \citep{rogers2021medical}. In particular, treatment records may not perfectly reflect the treatment actually delivered because of delayed charting, documentation mistakes, or administrative coding practices. This gap between recorded and true actions makes standard offline RL methods unreliable and motivates principled methods that explicitly account for hidden actions.

The hidden-action problem is related to the measurement-error literature. When the action is continuous and observed with additive noise, one can often leverage classical measurement error methods; see, e.g., \citet{carroll2006measurement}. In contrast, discrete action misclassification is substantially more challenging because the classical additive error framework underlying these methods generally does not apply. A recent related contribution is \citet{zhou2024causal}, which studies causal effect estimation under a hidden treatment in the static setting. Their framework uses a surrogate variable, together with an additional proxy variable related to the hidden treatment, and develops multiply robust estimators using semiparametric efficiency theory. 

Our work also leverages a proxy for the hidden action to identify the value function and construct an estimator. A key feature of the RL setting is that the next state provides such a proxy naturally: because it depends on the current latent action through the transition dynamics, it contains information about the hidden action that can be exploited for identification. Despite this connection, our setting differs substantially from that of \citet{zhou2024causal}. First, hidden actions arise in a sequential decision-making setting. Consequently, the target parameter and nuisance functions depend on the transition dynamics and the target policy, making the problem substantially more complex than in the static setting. Second, unlike the static hidden-treatment setting, which requires an additional proxy variable, the next state in RL serves simultaneously as a proxy for the current hidden action and as part of the future trajectory. This dual role creates technical challenges that are absent in the single-stage setting. As a result, the identification strategy, influence-function derivation, and asymptotic analysis require new arguments that explicitly exploit the sequential nature of RL, rather than constituting a direct extension of the single-stage hidden-treatment framework.

\textbf{\textit{Challenges and contributions}}. Since the true action in the environment is never observed, off-policy evaluation with hidden actions presents three main challenges. First, the state--action distribution underlying standard OPE methods is latent, making it unclear whether the target policy value is identifiable from the observed data. Second, even when identification can be established, constructing valid estimators remains nontrivial, as naively substituting the latent action with its surrogate in existing OPE methods generally leads to biased estimates. Finally, valid uncertainty quantification and policy comparison procedures have yet to be developed in this setting.

To address these challenges, we first establish identification of the target policy value in the presence of hidden actions by leveraging the next-state variable within the RL framework. Second, we derive an observed-data influence function representation and use it to construct LURE (\csname textbf\endcsname{L}earning from the \csname textbf\endcsname{U}nseen: \csname textbf\endcsname{R}obust \csname textbf\endcsname{E}stimator), a multiply robust estimator that remains consistent under several combinations of correctly specified nuisance components. To address the unobserved latent-action problem in the estimation procedure,
we propose a novel expectation--maximization-based algorithm for estimating the nuisance functions in our offline RL setting.
Finally, we establish finite-sample error bounds and the asymptotic normality of the proposed estimator, enabling valid confidence intervals and hypothesis tests for policy value estimation. To the best of our knowledge, this is the first work to study offline reinforcement learning in the hidden-action setting.


\textbf{\textit{Organization}}. The rest of the paper is organized as follows. In Section~\ref{sec: prelim}, we review the standard RL framework and existing OPE methods. In Sections~\ref{sec: latent_action} and \ref{sec: estimation}, we introduce the hidden-action setting, establish identification, and construct the LURE estimator. In Section~\ref{sec: theory}, we present the large-sample theory for LURE. In Section~\ref{sec: sim}, we report simulation results, and in Section~\ref{sec: real} we present the MIMIC-III sepsis analysis. We conclude with a discussion in Section~\ref{sec: discuss}. All proofs are deferred to the Appendix.

\section{Preliminaries} \label{sec: prelim}

In this section, we briefly review the standard RL framework, in which actions are perfectly recorded. 

We consider a time-homogeneous Markov decision process (MDP) defined by the sequence $\{(S_t, A_t, R_t)\}_{t \ge 1}$, where $S_t \in \mathcal{S}$ denotes the state, $A_t \in \mathcal{A}$ the action, and $R_t \in \mathbb{R}$ the reward at time $t$. The system dynamics are governed by a stationary Markov transition kernel $q(\cdot \mid s, a)$. Specifically, given the current state-action pair $(S_t, A_t) = (s, a)$, the next state is sampled according to $S_{t+1} \sim q(\cdot \mid s, a)$, independent of the preceding history. The reward $R_t$ is generated from a time-invariant distribution conditioned on $(S_t, A_t)$ for any time $t$. We define the reward function $\theta_R(s, a)$ as the conditional expectation: \begin{gather*}
    \theta_R(s, a) := \mathbb{E}[R_t \mid S_t = s, A_t = a].
\end{gather*} 
For any function $g:\mathcal{S}\times\mathcal{A}\to\mathbb{R}$, define $g(s,\pi):=\sum_{a\in\mathcal{A}} g(s,a)\pi(a\mid s)$.

\begin{figure}[H]
\centering
\resizebox{\linewidth}{!}{%
\begin{tikzpicture}[
  >=Stealth,
  node distance=1.4cm and 1.4cm,
  every node/.style={text=black, font=\normalsize},
  edge/.style={->, very thick, draw=black}
]

\node (S0) {$S_{t-1}$};
\node (A0) [right=of S0] {$A_{t-1}$};
\node (S1) [right=of A0] {$S_{t}$};
\node (A1) [right=of S1] {$A_{t}$};
\node (S2) [right=of A1] {$S_{t+1}$};

\node (R0) at ($(S0)!0.5!(A0) + (0,1.6cm)$) {$R_{t-1}$};
\node (R1) at ($(S1)!0.5!(A1) + (0,1.6cm)$) {$R_{t}$};

\draw[edge] (S0) -- (A0);
\draw[edge] (A0) -- (S1);
\draw[edge] (S0) to[bend right=41] (S1);
\draw[edge] (S1) -- (A1);
\draw[edge] (A1) -- (S2);
\draw[edge] (S1) to[bend right=41] (S2);

\draw[edge] (S0) -- (R0);
\draw[edge] (A0) -- (R0);

\draw[edge] (S1) -- (R1);
\draw[edge] (A1) -- (R1);

\end{tikzpicture}%
}
\caption{Data-generating process in standard MDP.}
\label{fig:mdp}
\end{figure}
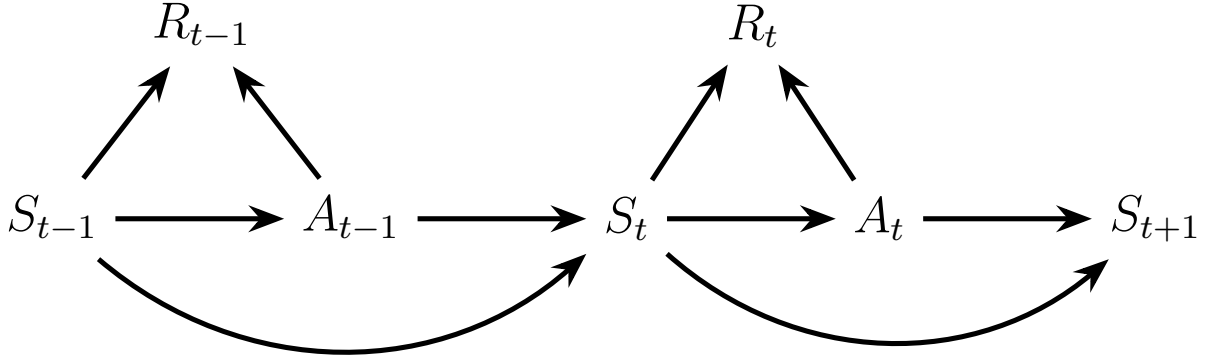

As is standard in the reinforcement learning literature, throughout, we assume the environment is stationary, meaning that both the transition kernel $q$ and the reward function $\theta_R$ are Markovian and remain constant over time. This assumption lays the foundation for many state-of-the-art algorithms \citep{sutton2018}.
Under this setting, it is sufficient to focus on stationary policies $\pi$ under which $A_t \sim \pi(\cdot \mid S_t)$ \citep{puterman2014markov}. Such a stationary policy specifies how the agent selects actions at each decision time, mapping the current state to a distribution over actions. Figure \ref{fig:mdp} summarizes the data-generating process in the standard MDP setting described above.

Policy evaluation aims to estimate the expected reward under a target policy $\pi$. This is quantified by the value function, which measures the long-term performance of a policy and is defined as: \begin{gather} \label{eq:value}
    V(\pi) := \mathbb{E}_{S_1 \sim p_e} \left[ \mathbb{E}^\pi \left(\sum_{t=1}^\infty \gamma^{t-1} R_t \,\middle\vert\, S_1 \right) \right],
\end{gather} 
where $\gamma \in [0, 1)$ is the discount factor that governs the trade-off between immediate and future rewards, and $p_e$ is a known initial state distribution.

In this paper, we study off-policy evaluation, which aims to estimate the value of a target policy $\pi$ using only pre-collected historical data. This offline dataset consists of $N$ independent trajectories of length $T$, summarized as $$\{(S_{i,t}, A_{i,t}, R_{i,t}) : 1\le i\le N,\ 1\le t\le T\}.$$ These data are generated by a potentially unknown behavior policy $b$, such that $A_{i,t} \sim b(\cdot\mid S_{i,t})$, which typically differs from the target policy $\pi$.

To estimate the policy value $V(\pi)$, a popular approach is the direct method. It is based on estimating the state-value function and the $Q$-function, defined as
\begin{align*}
V^\pi(s)
&:= \E^\pi\!\left[\sum_{t=1}^\infty \gamma^{t-1} R_t\,\middle\vert\, S_1=s\right],
\\
Q^\pi(s,a)
&:= \E^\pi\!\left[\sum_{t=1}^\infty \gamma^{t-1} R_t\,\middle\vert\, S_1=s,\,A_1=a\right],
\end{align*}
so that $V^\pi(s)=Q^\pi(s,\pi)=\sum_{a\in\mathcal A} \pi(a\mid s)Q^\pi(s,a)$.
The policy value is then given by
\begin{align*}
V(\pi):=\E_{S\sim p_e}\big[V^\pi(S)\big].
\end{align*}

Another approach is importance-sampling-based estimation. Let $p_t^\pi(s,a)$ denote the marginal state--action distribution at time $t$ induced by the policy $\pi$. For reference, the distribution $p_t^\pi(s,a)$ admits the standard expansion
\begin{gather*}
 p_t^\pi(s_t,a_t)
 = \sum_{a_1,\dots,a_{t-1}}\int p_e(s_1)
 \left( \prod_{k=1}^{t-1} \pi(a_k\mid s_k)\, p(s_{k+1}\mid s_k, a_k) \right)
 \pi(a_t\mid s_t)\, ds_{1}\dots ds_{t-1}.
\end{gather*}
In addition, define the discounted visitation
\begin{align*}
 p^*(s,a)
 &:=(1-\gamma)\sum_{t=1}^\infty \gamma^{t-1} p_t^\pi(s,a),
\end{align*}
and assume that for all $s$ and $a$, the marginal density ratio
\begin{align*}
\omega(s,a):=\frac{p^*(s,a)}{p^b(s,a)}
\end{align*} exists,  where $
 p^b(s,a)$ is the marginal state--action distribution in the offline data.
 Then,
\begin{align*}
    &V(\pi)
    =(1-\gamma)^{-1} \int_{s,a,r} r\, p(r\mid s,a) p^*(s,a)\, ds\, da\, dr
    \\
    = & \;(1-\gamma)^{-1} \E\left[\frac{p^*(S,A)}{p^b(S,A)}R\right]
    =(1-\gamma)^{-1} \E\left[\omega(S,A)R\right].
\end{align*} 

Thus, to estimate the value function $V(\pi)$, it suffices to estimate the density ratio $\omega(\cdot,\cdot)$.

Finally, we introduce a doubly robust representation \citep{kallus2022efficiently} of the value function, which combines the ideas of the previous two methods. Specifically, we have 
    \begin{align*}
       V(\pi)= \E_{p_e} (V^\pi(S))+(1-\gamma)^{-1} \E \Big[\omega(S,A)\left[ R+ \gamma V^\pi(S')-Q^\pi(S,A)\right] \Big].
    \end{align*} One can then construct an estimator accordingly, which requires estimation of both the $Q$-function and the density ratio $\omega$. The key advantage of this doubly robust formulation is that the resulting estimator remains consistent if either the $Q$-function or the density ratio is correctly specified.


The aforementioned OPE identification results and estimation methods all rely on the assumption that actions are perfectly measured in the offline data. In many applications, however,  the decision-maker only has access to an error-prone observed action $\widetilde A_t \in \mathcal{A}$, which is a noisy measurement of the underlying true action $A_t$. Naively substituting $\widetilde A_t$ for $A_t$ in any of the above OPE methods will generally produce biased estimates, since the $Q$-function and density ratio are both defined with respect to the latent true action. We address this setting in the following section.

\section{RL with Unobserved Action} \label{sec: latent_action}

In our framework, the true actions $A_{i,t}$ are not directly recorded. Instead, we observe trajectories of the form $\{(S_{i,t},\widetilde A_{i,t},R_{i,t}) : 1\le i\le N,\ 1\le t\le T\}$, where $\widetilde A_{i,t}$ is an error-prone surrogate for the hidden action $A_{i,t}$. Figure~\ref{fig:dgp} summarizes the data-generating process: at each time point, the true action $A_{i,t}$ is generated by a behavior policy $b$, while the observed surrogate $\widetilde A_{i,t}$ is generated conditional on the state $S_{i,t}$ and the true action $A_{i,t}$. This setting commonly arises in studies using electronic health records. For example, $A_{i,t}$ may indicate whether a clinician actually administered a medication at time $t$, whereas $\widetilde A_{i,t}$ is the medication documented in the record. The two can differ because of delayed charting, coding mistakes, or discrepancies between an order and the treatment ultimately delivered. The misrecording rate may also depend on patient acuity through the current state $S_{i,t}$.

\begin{figure}[t]
\centering
\resizebox{\linewidth}{!}{%
\begin{tikzpicture}[
  >=Stealth,
  node distance=1.4cm and 1.4cm,
  every node/.style={text=black, font=\normalsize},
  edge/.style={->, very thick, draw=black}
]

\node (S0) {$S_{t-1}$};
\node (A0) [right=of S0] {\textcolor{magenta}{$A_{t-1}$}};
\node (S1) [right=of A0] {$S_{t}$};
\node (A1) [right=of S1] {\textcolor{magenta}{$A_{t}$}};
\node (S2) [right=of A1] {$S_{t+1}$};

\node (R0) at ($(S0)!0.5!(A0) + (0,1.6cm)$) {$R_{t-1}$};
\node (R1) at ($(S1)!0.5!(A1) + (0,1.6cm)$) {$R_{t}$};

\node (At0) at ($(S0)!0.5!(A0)+ (0,-1.6cm)$) {$\widetilde A_{t-1}$};
\node (At1) at ($(S1)!0.5!(A1) + (0,-1.6cm)$) {$\widetilde A_{t}$};

\draw[edge] (S0) -- (A0);
\draw[edge] (A0) -- (S1);
\draw[edge] (S0) to[bend right=32] (S1);
\draw[edge] (S1) -- (A1);
\draw[edge] (A1) -- (S2);
\draw[edge] (S1) to[bend right=32] (S2);

\draw[edge] (S0) -- (R0);
\draw[edge] (A0) -- (R0);

\draw[edge] (S1) -- (R1);
\draw[edge] (A1) -- (R1);

\draw[edge] (A0) -- (At0);
\draw[edge] (S0) -- (At0);
\draw[edge] (A1) -- (At1);
\draw[edge] (S1) -- (At1);

\end{tikzpicture}%
}
\caption{Data-generating process with latent actions. The underlying latent true action \textcolor{magenta}{$A_t$} governs the transition to the next state and the reward. The observed surrogate $\widetilde A_t$ is generated conditional on the state and the true action, and may differ from the true action due to misclassification.}
\label{fig:dgp}
\end{figure}
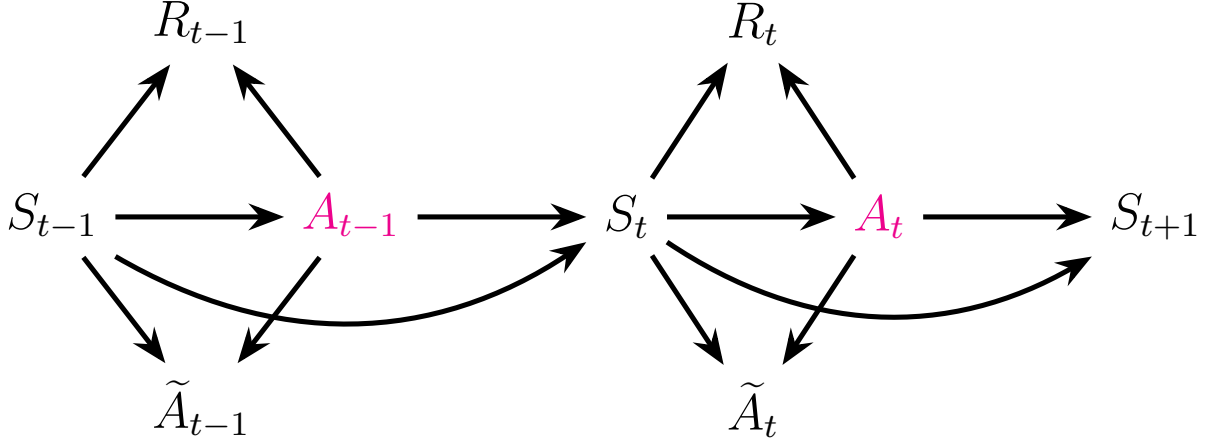

Under this setting, the goal remains the evaluation of the value function defined in Equation \eqref{eq:value}. However, the focus is on estimating the value of the target policy that governs the true action-generating mechanism $\pi(A\mid S)$, rather than that of the policy acting on the observed surrogate $\pi(\widetilde{A}\mid S)$. This distinction introduces substantial difficulty for two reasons. First, it is not immediate whether $V(\pi)$ is identifiable from the observed data, since the true action is never directly recorded. Second, even when identification can be established, standard OPE estimating techniques cannot be applied directly, as they involve the latent action explicitly. We first address the identification challenge, beginning by describing our assumptions. Throughout the paper, we focus on binary actions for clarity, while noting that the proposed framework can be extended to categorical action spaces (see Section \ref{sec:multi-action} in the Appendix). Because we focus on the stationary setting, we omit the time index hereafter where it causes no confusion.

\begin{assumption}
\label{assump:latent_action_id}
\begin{enumerate}[label=(\roman*)]
\item (Coverage) For all $s$ and $a$, we have \begin{gather*}
    p^b(s,a)= 0 \implies p^*(s,a) = 0.
\end{gather*}
\item (Relevance) $R\nind A\mid S; \widetilde A\nind A\mid S; S'\nind A\mid S$.
\item (Conditional independence) $R\independent \widetilde A \independent S' \mid (S,A)$.
\item (Surrogate)  $\Pr(\widetilde A=0\mid A=1, S=s)< \Pr(\widetilde A=0\mid A=0, S=s)$, for any $s$. \label{cdn:1}

\end{enumerate}
\end{assumption}

Assumption~\ref{assump:latent_action_id}(i) is standard in OPE literature \citep{uehara2022review}. It ensures that the support of the target policy is contained in the support of the behavior policy.
Assumption~\ref{assump:latent_action_id}(ii) and (iii) are relevance and conditional independence assumptions. They are natural consequences of the causal structure depicted in Figure~\ref{fig:dgp}: the surrogate $\widetilde A$ is generated based on $(S, A)$ and does not directly affect either the reward or the transition. Assumption~\ref{assump:latent_action_id}(iv) is mild, requiring that the surrogate be more informative than a random guess. Similar assumptions have been imposed in \citet{hu2008identification} and \citet{zhou2024causal}.

\begin{theorem}[Identification]
\label{thm:identification}
Under Assumption~\ref{assump:latent_action_id}, the policy value $V(\pi)$ is uniquely identified from the distribution of the observed data $O$.
\end{theorem}

Having established identification, we next turn to estimation and develop an influence-function-based estimator for $V(\pi)$. Specifically, we derive an influence function of $V(\pi)$ and use it to construct LURE (\textbf{L}earning from the \textbf{U}nseen: \textbf{R}obust \textbf{E}stimator), a multiply robust estimator that accommodates the latent-action structure.
We denote the observed transition tuple by \(O=(S,\widetilde A,R,S')\).  For any measurable function $f:\mathcal{S}\to\mathbb{R}$, we denote the Markov operator by
$(\mathcal{M}f)(s,a):=\E\!\left[f(S')\mid S=s, A=a\right].$
Let \(\theta_{\widetilde A}(s,a) := \E(\widetilde A \mid S=s, A=a)\) denote the surrogate-action measurement model, and let \(\theta_{S'}(s,a) := \E\left[l(S') \mid S=s, A=a\right]\) denote the conditional expectation of the proxy, where \(l:\mathbb{R}^d \to \mathbb{R}\) is a known function.  For example, one may take \(l(s) = s^{(j)}\), i.e., the $j$-th component of the state vector. 
Finally, recall that the reward function is given by \(\theta_R(s,a) := \E(R \mid S=s, A=a)\). Throughout, we assume that the reward $R$ and $l(S')$ are uniformly bounded.

\begin{theorem}\label{thm:influence_function}
Under Assumption~\ref{assump:latent_action_id}, the function $\varphi^\pi(O)$ defined below is an influence function for the value function $V(\pi)$: 
    \begin{align}\label{eq:IF}
        \nonumber \varphi^\pi(O)=&\E_{p_e} (V^\pi(S))+(1-\gamma)^{-1}\sum_a  g_{a}' (S',\widetilde A,S) \omega(S,a) \left[ R-\theta_R(S,a) \right]\\
        &+\frac{\gamma}{1-\gamma}\sum_a  g_{a} (R,\widetilde A,S) \omega(S,a) \left[ V^\pi(S')-\mathcal{M}V^\pi(S,a)  \right]-V(\pi), \\
  \text{ where } 
   \nonumber & g_{a}(R,\widetilde A,S) =  \frac{\widetilde A-\theta_{\widetilde A}(S,1-a)}{\theta_{\widetilde A}(S,a)-\theta_{\widetilde A}(S,1-a)} \frac{R-\theta_{R}(S,1-a)}{\theta_{R}(S,a)-\theta_{R}(S,1-a)},\\
 \nonumber & g_{a}'(S',\widetilde A,S) =  \frac{\widetilde A-\theta_{\widetilde A}(S,1-a)}{\theta_{\widetilde A}(S,a)-\theta_{\widetilde A}(S,1-a)} \frac{l(S')-\theta_{S'}(S,1-a)}{\theta_{S'}(S,a)-\theta_{S'}(S,1-a)},
\end{align} for $a=0,1$.
\end{theorem}


Throughout, we assume  $\Var\!\left[\varphi^\pi(O)\right]<\infty$. It can be seen that the above influence function is a function of the observed data. In addition, for any $a$, the quantities $g_a(R,\widetilde A,S)$, $g_a'(S',\widetilde A,S)$, $\theta_R(S,a)$, and $\mathcal{M}V^\pi(s,a)$ are all identified, following the proof of Theorem \ref{thm:identification}.
In particular, $g_a$ and $g_a'$ play the role of the indicator function $\mathds{1}(A=a)$, as we can show that (see Lemma \ref{lemma:indicator} in the Appendix)
\begin{gather*}
    \E\!\left[g_a(R,\widetilde A,S)\mid S,A\right]=\E\!\left[g_a'(S',\widetilde A,S)\mid S,A\right]=\mathds{1}(A=a).
\end{gather*}
Consequently,
\begin{align*}
    &\E\!\left[\sum_a g_a(R,\widetilde A,S)\,\omega(S,a)\,\middle|\,S,A\right]
    =\E\!\left[\sum_a g_a'(S',\widetilde A,S)\,\omega(S,a)\,\middle|\,S,A\right]\\
    = & \sum_a \mathds{1}(A=a)\,\omega(S,a)
    =\omega(S,A).
\end{align*}
Projecting the observed influence function in Equation \eqref{eq:IF} onto the full-data space yields 
\begin{align*}
\E_{p_e} (V^\pi(S))+(1-\gamma)^{-1} \Big[\omega(S,A)\left( R+ \gamma V^\pi(S')-Q^\pi(S,A)\right) \Big]-V(\pi),
\end{align*} which coincides with the full efficient influence function in \citet{kallus2022efficiently}.


Based on the influence function in Equation \eqref{eq:IF}, we construct the LURE estimator of $V(\pi)$ by replacing the nuisance components with estimators and taking the sample mean.
Let $\widehat g_a(R,\widetilde A,S)$ and $\widehat g_a'(S',\widetilde A,S)$ be obtained by plugging in estimates of $\theta_{\widetilde A}$, $\theta_{S'}$, and $\theta_R$.
Let $\widehat\omega(s,a)$ be any estimator of the marginal density ratio. 
Also let $\widehat Q^\pi(s,a)$ be an estimator of the action-value function, and define $\widehat V^\pi(s):=\widehat Q^\pi(s,\pi)=\sum_{a\in\mathcal A}\pi(a\mid s)\widehat Q^\pi(s,a)$.
Recall that ${\mathcal M}V(s,a)=\E[V^\pi(S')\mid S=s,A=a]$. By the Bellman equation, we have
$$Q^\pi(s,a)=\theta_R(s,a)+\gamma \mathcal M V^\pi(s,a).$$
Consequently, $\widehat{\mathcal M V^{\pi}}(S_{i,t},a)$ can be obtained by
\begin{gather*}
    \widehat{\mathcal M V^{\pi}}(S_{i,t},a)= \gamma^{-1} \left[\widehat Q^\pi(S_{i,t},a)-\widehat\theta_R(S_{i,t},a)\right].
\end{gather*}

The resulting value estimator is then
\begin{align} \label{eq:IF_estimator} 
& \E_{p_e}\big(\widehat V^\pi(S)\big)
+(1-\gamma)^{-1}\frac{1}{NT}\sum_{i=1}^N\sum_{t=1}^T\sum_{a\in\mathcal A}
\widehat g_a'(S'_{i,t},\widetilde A_{i,t},S_{i,t})\,\widehat\omega(S_{i,t},a)\,\Big[R_{i,t}-\widehat\theta_R(S_{i,t},a)\Big]
\nonumber\\
&+\frac{\gamma}{1-\gamma}\frac{1}{NT}\sum_{i=1}^N\sum_{t=1}^T\sum_{a\in\mathcal A}
\widehat g_a(R_{i,t},\widetilde A_{i,t},S_{i,t})\,\widehat\omega(S_{i,t},a)\,\Big[\widehat V^\pi(S_{i,t+1})-\widehat {\mathcal M V^{\pi}}(S_{i,t},a)\Big].
\end{align}


\section{Estimation Procedure} \label{sec: estimation}

Recall that by Equation \eqref{eq:IF_estimator},
constructing the LURE estimator requires estimating the following nuisance components: $\omega(s,a)$, $\theta_{\widetilde A}(s,a)$, $\theta_{R}(s,a)$, $\theta_{S'}(s,a)$, and $Q^\pi(s,a)$ (equivalently, $\mathcal{M}V^\pi(s,a)$). This section describes how we estimate these nuisance functions. The core challenge is that the true action $A$ is unobserved, so standard estimating methods in existing RL cannot be applied directly. We address this via a novel EM-style \citep{dempster1977maximum} iterative procedure that alternates between
(i) estimating nuisance functions given the current posterior probability (the ``M-step''), and
(ii) updating the posterior probability given the current nuisance estimates (the ``E-step'').
We start by initializing the posterior probability function
\[
\eta(O;a)
:=
\Pr\big(A=a\mid O\big),
\qquad
O=(S,\widetilde A,R,S').
\]
Suppose for the moment that the true value of $\eta(O;a)$ is known. Then the key nuisance functions can be estimated via weighted regression or weighted likelihood maximization. We now provide the details.

\csname textbf\endcsname{\csname textit\endcsname{Behavior policy}} $b$. The behavior policy satisfies
\begin{align*}
 b(a\mid S)
 &=\E\big[\mathds{1}(A=a)\mid S\big]
 =\E\Big[\E\big[\mathds{1}(A=a)\mid O\big] \,\Big|\,S\Big]
 =\E\big[\eta(O;a)\mid S\big].
\end{align*}
Thus, we can estimate $b(a\mid s)$ by regressing $\eta(O;a)$ on $S$.

\textit{\textbf{Measurement model}} $\theta_{\widetilde A}$. 
By Bayes' rule and iterated expectations,
\begin{align*}
\theta_{\widetilde A}(S,a)
&=\frac{\Pr(\widetilde A=1, A=a\mid S)}{\Pr(A=a\mid S)}
=\frac{\E\big[\widetilde A\mathds{1}(A=a)\mid S\big]}{b(a\mid S)}
=\frac{\E\big[\widetilde A\,\eta(O;a)\mid S\big]}{\E\big[\eta(O;a)\mid S\big]}.
\end{align*}
Equivalently, for any $a=0,\,1$,
\begin{gather*}
\E\left[ \eta(O;a)\left(\widetilde A-\theta_{\widetilde A}(S,a)\right) \middle| S \right]=0.
\end{gather*}
We can therefore estimate $\theta_{\widetilde A}(S,a)$ by regressing $\widetilde A$ on $S$ with observation weights $\eta(O;a)$.

\textit{\textbf{Reward and transition model}} $\theta_R$ and $\theta_{S'}$. Similarly, we estimate
$\theta_R(s,a)=\E(R\mid S=s,A=a)$ and
$\theta_{S'}(s,a)=\E(l(S')\mid S=s,A=a)$
by weighted regression with observation weights $\eta(O;a)$.

\textbf{\textit{Marginal density ratio}} $\omega$.
We now discuss how to estimate the marginal density ratio $\omega(S,a)$.
By Lemma \ref{lemma:balancing} in Appendix, for any function $f:\mathcal S\times\mathcal A\to\mathbb R$, the underlying true density ratio $\omega$ satisfies
\begin{equation}
\E\Big[
\omega(S,A)
\big[
f(S,A) - \gamma \, f(S',\pi)
\big]
\Big]
=
(1-\gamma)\,\mathbb E_{S\sim p_e}\Big[f(S,\pi)\Big].
\label{eq:balance}
\end{equation}
When the action $A$ can be fully observed, \citet{uehara2020minimax} propose to
estimate $\omega$ by solving the following minimax problem
\begin{align*}
\min_{\omega \in \Omega} \max_{f \in \mathcal{F}}
\Bigg[
\frac{1}{NT}\sum_{i=1}^N\sum_{t=1}^T 
\omega(S_{i,t},A_{i,t})
\Big(
f(S_{i,t},A_{i,t})
-
\gamma f(S_{i,t+1},\pi)
\Big)
-(1-\gamma)\E_{p_e}[f(S,\pi)]
\Bigg]^2,
\end{align*} with $\mathcal{F}$ serving as a discriminator function class. Such a minimax formulation is commonly employed in reinforcement learning and, more broadly, in conditional moment restriction problems; see, e.g., \citet{liu2018breaking,dikkala2020minimax,shi2021deeply}.

However, in our setting, the true action $A$ is unobserved. As a result, the above estimation strategy cannot be applied directly, since expectations involving $(S, A, S')$ cannot be approximated by empirical averages based on the observed data. To address this, we integrate
over the latent action using the posterior probabilities/weights
$\eta(O;a)$. Specifically, for any measurable function
$m(O,A)$, we have 
\begin{align*}
\E\!\left[m(O,A)\right]
&= \E\!\left[ \sum_{a \in \mathcal{A}} m(O,a)\,\mathds{1}(A=a) \right] = \E\!\left[ \E\!\left( \sum_{a \in \mathcal{A}} m(O,a)\,\mathds{1}(A=a) \,\middle|\, O \right) \right] \\
&= \E\!\left[ \sum_{a \in \mathcal{A}} m(O,a)\, \E\!\left(\mathds{1}(A=a) \mid O \right) \right] = \E\!\left[ \sum_{a \in \mathcal{A}} \eta(O;a)\, m(O,a) \right],
\end{align*}
and hence can be approximated by \begin{gather} \label{eq:computation}
    \frac{1}{NT}\sum_{i=1}^N\sum_{t=1}^T\sum_{a\in\mathcal A}
\widehat \eta(O_{i,t};a)\,m(O_{i,t},a),
\end{gather} where $\widehat \eta(o ;a)$ is some estimator for $\eta(o;a)$.

Hence, combining Equations \eqref{eq:balance} and \eqref{eq:computation}, we can
estimate $\omega$ by solving the following weighted minimax problem
\begin{align}
\min_{\omega \in \Omega} \max_{f \in \mathcal{F}}
\Bigg[
\frac{1}{NT}\sum_{i=1}^N\sum_{t=1}^T\sum_{a\in\mathcal A}\widehat \eta(O_{i,t};a)\,
\omega(S_{i,t},a)
\Big(
f(S_{i,t},a)
-
\gamma f(S_{i,t+1},\pi)
\Big)
-(1-\gamma)\E_{p_e}[f(S,\pi)]
\Bigg]^2,
\label{eq:latent_dice_combined}
\end{align}
\noindent which generalizes the estimation idea of \citet{liu2018breaking} and \citet{uehara2020minimax} to handle the latent action space.

\csname textbf\endcsname{\csname textit\endcsname{Q-function}} $Q^\pi$.
We next discuss how to estimate the $Q$-function $Q^\pi(S,A)$. Combining the Bellman equation
\begin{equation}
Q^\pi(s,a) = \mathbb E\!\left[ R + \gamma V^\pi(S') \mid S=s, A=a \right],
\label{eq:bellman_Q}
\end{equation} with Equation \eqref{eq:computation}, we obtain \begin{gather*}
    \E\!\left[\sum_{a\in\mathcal A}
\eta(O;a)\left(R+\gamma V^\pi(S')-Q^\pi(S,a) \right)\right]=0.
\end{gather*} This motivates us to employ a weighted version of fitted $Q$-evaluation \citep[FQE,][]{ernst2005tree}.
 We summarize the estimation procedure in Algorithm \ref{alg:fqe}. 

\begin{algorithm}[H]
\caption{Weighted FQE}
\label{alg:fqe}
\begin{algorithmic}[1]
\State \textbf{Initialize:} An initial estimator $\widehat Q^{(0)}$.
\For{$m=0,1,2,\ldots$ until convergence}
    \State Compute
    $
    Y_{i,t}^{(m)}
    =
    R_{i,t}
    +
    \gamma
    \sum_{a\in\mathcal A}
    \pi(a\mid S_{i,t+1})
    \widehat Q^{(m)}(S_{i,t+1},a)
    $
    for all $(i,t)$.
    \State Update $\widehat Q^{(m+1)}$ by solving
    \Statex
    \[
    \widehat Q^{(m+1)}
    \in
    \arg\min_{Q\in\mathcal Q}
    \frac{1}{NT}
    \sum_{i=1}^N\sum_{t=1}^T
    \sum_{a\in\mathcal A}
    \widehat \eta(O_{i,t};a)
    \left\{
    Y_{i,t}^{(m)} - Q(S_{i,t},a)
    \right\}^2 .
    \]
\EndFor
\State \textbf{Output:} $\widehat Q^\pi \gets \widehat Q^{(m)}$.
\end{algorithmic}
\end{algorithm}

\textbf{\textit{Posterior probability}} $\eta$.
The preceding nuisance estimators all treat $\eta(O;a)$ as given, so it remains to describe its estimation. Under Bayes' rule and the conditional independence  in Assumption~\ref{assump:latent_action_id}, the posterior probability of the latent action is
\begin{align}
\eta(O;a)
=\frac{
b(a\mid S)\,
\mu(\widetilde A\mid S,a)\,
h(R\mid S,a)\,
q(S'\mid S,a)
}{
\sum_{a'\in\mathcal A}
b(a'\mid S)\,
\mu(\widetilde A\mid S,a')\,
h(R\mid S,a')\,
q(S'\mid S,a')
},
\label{eq:posterior_bayes}
\end{align}
where $\mu(\widetilde A\mid S,a)$, $h(R\mid S,a)$, and $q(S'\mid S,a)$ denote the conditional probabilities or densities of $\widetilde A$, $R$, and $S'$ given $(S,A=a)$, respectively. For binary $\widetilde A$, we have
\[
\mu(\widetilde A\mid S,a)
=
\widetilde A\,\theta_{\widetilde A}(S,a)
+
(1-\widetilde A)\left[1-\theta_{\widetilde A}(S,a)\right].
\]
Equation~\eqref{eq:posterior_bayes} shows that estimating $\eta$ reduces to estimating the generative model components $b$, $\mu$, $h$, and $q$. However, each of these components itself depends on $\eta$ through the weighted regressions described above. We therefore adopt an iterative scheme that alternates between updating the generative components given the current posterior (the M-step) and recomputing the posterior given the updated components (the E-step). In the M-step, the behavior policy, measurement model, and conditional means of the reward and next state are updated via weighted least-squares regressions with weights $\widehat\eta^{(k)}$, while the conditional densities $q$ and $h$ are updated via weighted maximum likelihood. In the E-step, the posterior is refreshed by plugging the updated estimates into Equation \eqref{eq:posterior_bayes}. The full procedure is summarized in Algorithm~\ref{alg:latent_action_em}.

\textbf{\textit{Label alignment}}.
Although Assumption~\ref{assump:latent_action_id} allows us to identify the action labels at the population level, the numerical labels produced by Algorithm~\ref{alg:latent_action_em} are arbitrary: component \(0\) may correspond to the true action \(A=0\), or it
may correspond to the true action \(A=1\). Therefore, before evaluating the final estimator, we apply a label-alignment step.
Before label alignment, the index \(a\in\{0,1\}\) in the output of Algorithm~\ref{alg:latent_action_em} should be viewed only as an algorithmic component index. Under Assumption~\ref{assump:latent_action_id}(iv), the component with the larger conditional probability of observing \(\widetilde A=1\) is interpreted as the true action \(A=1\). We therefore define the population label selector
$d=\arg\max_{a\in\{0,1\}}\E[\theta_{\widetilde A}(S,a)],
$
and its empirical analogue
\[
\widehat d
=\arg\max_{a\in\{0,1\}}\frac{1}{NT}\sum_{i=1}^N\sum_{t=1}^T \widehat\theta_{\widetilde A}(S_{i,t},a).
\]
Thus, the component indexed by \(\widehat d\) is assigned to the true action \(A=1\), while the component indexed by \(1-\widehat d\) is assigned to the true action \(A=0\).

Let \(\widehat V_a(\pi)\), \(a\in\{0,1\}\), denote the LURE estimator obtained by treating component \(a\) as the true action \(A=1\) and component \(1-a\) as the true action \(A=0\), with all action-indexed nuisance estimates relabeled accordingly. The final label-aligned estimator is then
\[
\widetilde V(\pi)=\widehat V_{\widehat d}(\pi)
\]

\textbf{\textit{Cross-fitting procedure}}.
Finally, we introduce the LURE estimator with a $K$-fold cross-fitting procedure \citep{schick1986asymptotically,chernozhukov2018double}, which mitigates overfitting bias and helps control the empirical-process remainder. Specifically, partition the $N$ trajectories into $K$ folds $\{I_k\}_{k=1}^K$ of approximately equal size. For each fold $k$, estimate the nuisance functions $\theta_{\widetilde A}$, $\theta_{S'}$, $\theta_R$, $\omega$, and $ Q^\pi$ using the training data $\{O_{i,t}: i \notin I_k\}$, and then evaluate the value function on the held-out fold $I_k$. The final estimator is
\[
\widehat V(\pi) = \frac{1}{K}\sum_{k=1}^K \widetilde V^{(k)}(\pi),
\]
where $\widetilde V^{(k)}(\pi)$ denotes the fold-specific estimator constructed with nuisance estimates fit off fold $k$ and evaluated on fold $k$.

\begin{algorithm}[t]
\caption{Iterative latent-action reweighting}
\label{alg:latent_action_em}
\begin{algorithmic}[1]
\State \textbf{Initialize} $\{\widehat \eta^{(0)}(O_{i,t};a)\}_{i,t}$ such that $\sum_{a\in\mathcal A}\widehat \eta^{(0)}(O_{i,t};a)=1$.
\Repeat
\State \textbf{M-step (weighted supervised learning):} given $\{\widehat \eta^{(k)}(O_{i,t};a)\}_{i,t}$, $\forall a,$ update
\[
\begin{gathered}
\widehat b^{(k+1)}
=\argmin_{b \in \mathcal{B}}\; \frac{1}{NT} \sum_{i=1}^N\sum_{t=1}^T \big[\widehat \eta^{(k)}(O_{i,t};a)-b(a\mid S_{i,t})\big]^2, \\
\widehat \theta_{\widetilde A}^{(k+1)}
=\argmin_{\theta_{\widetilde A} \in \Theta_{\widetilde A}}\; \frac{1}{NT} \sum_{i=1}^N\sum_{t=1}^T \widehat \eta^{(k)}(O_{i,t};a)\big[\widetilde A_{i,t}-\theta_{\widetilde A}(S_{i,t},a)\big]^2, \\
\widehat \theta_{S'}^{(k+1)}
=\argmin_{\theta_{ S'} \in \Theta_{ S'}}\; \frac{1}{NT} \sum_{i=1}^N\sum_{t=1}^T \widehat \eta^{(k)}(O_{i,t};a)\big[ l(S_{i,t+1})-\theta_{S'}(S_{i,t},a)\big]^2, \\
\widehat\theta_R^{(k+1)}
=\argmin_{\theta_R\in \Theta_R}\; \frac{1}{NT}\sum_{i=1}^N\sum_{t=1}^T \widehat \eta^{(k)}(O_{i,t};a)\big[R_{i,t}-\theta_R(S_{i,t},a)\big]^2, \\
\widehat q^{(k+1)}
=\argmin_{q\in\mathcal Q}\; -\frac{1}{NT}\sum_{i=1}^N\sum_{t=1}^T \sum_{a \in \mathcal A} \widehat \eta^{(k)}(O_{i,t};a)\log q(S_{i,t+1}\mid S_{i,t},a), \\
\widehat h^{(k+1)}
=\argmin_{h\in\mathcal H}\; -\frac{1}{NT}\sum_{i=1}^N\sum_{t=1}^T \sum_{a \in \mathcal A} \widehat \eta^{(k)}(O_{i,t};a)\log h(R_{i,t}\mid S_{i,t},a).
\end{gathered}
\]
\State \textbf{E-step (posterior update):} for each $(i,t)$ and $a\in\mathcal A$, let 
\begin{gather*}
\nonumber \widehat \mu^{(k+1)}(\widetilde A_{i,t}\mid S_{i,t},a)=\widetilde A_{i,t}\widehat \theta_{\widetilde A}^{(k+1)}(S_{i,t},a)+(1-\widetilde A_{i,t})\left(1-\widehat \theta_{\widetilde A}^{(k+1)}(S_{i,t},a)\right) \\
\widehat \eta^{(k+1)}(O_{i,t};a)
=
\frac{
\widehat b^{(k+1)}(a\mid S_{i,t})\,
\widehat \mu^{(k+1)}(\widetilde A_{i,t}\mid S_{i,t},a)\,
\widehat h^{(k+1)}(R_{i,t}\mid S_{i,t},a)\,
\widehat q^{(k+1)}(S_{i,t+1}\mid S_{i,t},a)
}{
\sum_{a' \in\mathcal A}
\widehat b^{(k+1)}(a'\mid S_{i,t})\,
\widehat \mu^{(k+1)}(\widetilde A_{i,t}\mid S_{i,t},a')\,
\widehat h^{(k+1)}(R_{i,t}\mid S_{i,t},a')\,
\widehat q^{(k+1)}(S_{i,t+1}\mid S_{i,t},a')
}.
\end{gather*}
\Until{convergence}
\State \textbf{Output} $\widehat \eta(O_{i,t};a)$ and nuisance estimates
$\widehat b$, $\widehat\theta_{S'}$, $\widehat\theta_{\widetilde A}$, and $\widehat\theta_R$.
\end{algorithmic}
\end{algorithm}


\section{Theory} \label{sec: theory}

In this section, we establish the theoretical properties of the cross-fitted LURE estimator $\widehat V(\pi)$, showing that it achieves $\sqrt{NT}$-consistency and asymptotic normality. These results enable the construction of asymptotically valid confidence intervals and hypothesis tests for the policy value $V(\pi)$.

For later use, we introduce some notation. Let \(\|\cdot\|\) denote the \(L_2(P)\) norm under the data-generating distribution, and let \(\|\cdot\|_\infty\) denote the supremum norm. For each action-indexed estimated nuisance function
\(\widehat f\in\{\widehat\omega,\widehat\theta_{\widetilde A},\widehat\theta_R,\widehat\theta_{S'}\}\), define 
\(\delta \widehat f(\cdot,a)=\widehat f(\cdot,a)-f(\cdot,a)\), \(a\in\{0,1\}\).
Define $e_M(S,a)=\mathbb E[\widehat V^\pi(S')\mid S,A=a]-\widehat{\mathcal M V^{\pi}}(S,a).$ We begin with the following technical conditions.

\begin{assumption}[Regularity conditions]
\label{assump:regularity}
(i) \text{(Boundedness)} All the nuisance estimators are uniformly bounded. (ii) \textit{(Separation)}
There exists some universal constant \(c>0\) such that, for all $a$,
    $\inf_s\left|\theta_{\widetilde A}(s,a)-\theta_{\widetilde A}(s,1-a)\right|\ge c,$ $\inf_s\left|\theta_R(s,a)-\theta_R(s,1-a)\right|\ge c$ and $\inf_s\left|\theta_{S'}(s,a)-\theta_{S'}(s,1-a)\right|\ge c.$
The same lower bounds hold almost surely when \(\theta_{\widetilde A}\), \(\theta_R\), and \(\theta_{S'}\) are replaced by their estimated counterparts.

\end{assumption}

Assumption~\ref{assump:regularity}(i) is a standard boundedness condition. Assumption~\ref{assump:regularity}(ii) is a separation condition; such conditions are common in causal inference and policy learning, see, e.g., \citet{qian2011performance,zhou2024causal}. It ensures that the denominators appearing in the influence function are bounded away from zero. Intuitively, it requires the hidden actions to induce meaningfully different conditional laws for the surrogate action, reward, and next-state proxy.

\begin{theorem}[Error Bound]\label{thm:error}
Under Assumptions~\ref{assump:latent_action_id} and~\ref{assump:regularity}, we have
\[
\widehat V(\pi)-V(\pi) = \frac{1}{NT} \sum_{i=1}^N\sum_{t=1}^T\varphi^\pi(O_{i,t})+O_p(L_n+I_n) + o_p((NT)^{-1/2}),
\]where $L_n=\max_{a\in\{0,1\}}\|\delta \widehat \theta_{\widetilde A}(\cdot,a)\|^\kappa,$ for some $\kappa>0,$ is the label-selection error, and
\begin{align*}
&I_n = \sum_{a\in\{0,1\}}\E\Big[
\left\{|\delta\widehat\theta_{\widetilde A}(S,a)|
+|\delta\widehat\theta_{S'}(S,a)|
+|\delta\widehat\omega(S,a)|\right\}
|\delta\widehat\theta_R(S,a)| \\
&+ \left\{|\delta\widehat\theta_{\widetilde A}(S,a)|
+|\delta\widehat\theta_R(S,a)|
+|\delta\widehat\omega(S,a)|
\right\}|e_M(S,a)| \\
&+ |\delta\widehat\theta_{\widetilde A}(S,a)|
|\delta\widehat\theta_{S'}(S,a)| + 
|\delta\widehat\theta_{\widetilde A}(S,1-a)|
|\delta\widehat\theta_{S'}(S,1-a)|
|\delta\widehat\theta_R(S,a)| \\
&+
|\delta\widehat\omega(S,a)|
|\delta\widehat\theta_{\widetilde A}(S,1-a)|
|\delta\widehat\theta_{S'}(S,1-a)|
+
|\delta\widehat\omega(S,a)|
|\delta\widehat\theta_{\widetilde A}(S,1-a)|
|\delta\widehat\theta_R(S,1-a)|
\Big] 
\end{align*}
is the second order remainder.
\end{theorem}

Theorem~\ref{thm:error} decomposes the policy evaluation error into the sample average of the influence function, together with a label-selection remainder $L_n$, a second-order remainder $I_n$, and a negligible remainder of order $o_p((NT)^{-1/2})$. We make several remarks. First, the exponent $\kappa>0$ governs the decay rate of the label-selection remainder $L_n$; its exact value depends on the regularity of the label selector and only affects the rate at which $L_n$ converges to zero. Second, under suitable convergence rate conditions on the nuisance function estimators, both $L_n$ and $I_n$ are $o_p((NT)^{-1/2})$, so that the influence-function term dominates the asymptotic expansion. Finally, the remainder $I_n$ consists of products of nuisance estimation errors, which gives rise to a multiple robustness property for consistency, as formalized below.

\begin{corollary}[Multiple Robustness] \label{cor: multiple robust}
Suppose that the assumptions in Theorem \ref{thm:error} hold and the label-alignment step is consistent.
Then we have \begin{gather*}
    \widehat V(\pi) - V(\pi) = o_p(1),
\end{gather*} provided that for each \(a\in\{0,1\}\), at least one of the following conditions is satisfied:\\
    (i) \(\| \widehat\theta_{\widetilde A}(\cdot,a)- \theta_{\widetilde A}(\cdot,a)\|=o_p(1)\), \(\| \widehat\theta_{R}(\cdot,a)- \theta_{R}(\cdot,a)\|=o_p(1)\),
and \(\|e_M(\cdot,a)\|=o_p(1)\);\\
    (ii) \(\| \widehat\theta_{S'}(\cdot,a)- \theta_{S'}(\cdot,a)\|=o_p(1)\), \(\| \widehat\theta_{R}(\cdot,a)- \theta_{R}(\cdot,a)\|=o_p(1)\),
and \(\|e_M(\cdot,a)\|=o_p(1)\);\\
    (iii) \(\| \widehat\omega(\cdot,a)- \omega(\cdot,a)\|=o_p(1)\), \(\| \widehat\theta_{R}(\cdot,a)- \theta_{R}(\cdot,a)\|=o_p(1)\), and \(\| \widehat\theta_{\widetilde A}(\cdot,a)- \theta_{\widetilde A}(\cdot,a)\|=o_p(1)\);\\
    (iv) \(\| \widehat\omega(\cdot,a)- \omega(\cdot,a)\|=o_p(1)\), \(\| \widehat\theta_{\widetilde A}(\cdot,a)- \theta_{\widetilde A}(\cdot,a)\|=o_p(1)\), \(\| \widehat\theta_{S'}(\cdot,a)- \theta_{S'}(\cdot,a)\|=o_p(1)\), and \(\|e_M(\cdot,a)\|=o_p(1)\).
\end{corollary}

The four cases in Corollary~\ref{cor: multiple robust} highlight two complementary routes to consistency. Cases~(i) and~(ii) correspond to a direct-estimation strategy: consistency is achieved when the reward bridge $\theta_R$ is estimated consistently together with either the latent-action regression $\theta_{\widetilde A}$ or the next-state regression $\theta_{S'}$, provided that the continuation-value error $e_M$ is negligible. In contrast, Cases~(iii) and~(iv) correspond to an MIS-based strategy: once $\omega$ and $\theta_{\widetilde A}$ are consistently estimated, consistency is preserved as long as either $\theta_R$ or $\theta_{S'}$ is correctly specified. This multiple robustness property allows misspecification in one component of the debiasing structure to be offset by correct specification of another.

\textbf{\textit{Asymptotic normality}}. We next strengthen the consistency result to an asymptotic normality result. Let $\alpha_\omega,\alpha_{\widetilde A},\alpha_R,\alpha_{S'},\alpha_M, \alpha_{\min} \geq 0$ such that
\[
\begin{gathered}
\max_{a\in\{0,1\}}\|\delta \widehat \omega(\cdot,a)\|=O_p((NT)^{-\alpha_\omega}),\qquad
\max_{a\in\{0,1\}}\|\delta \widehat \theta_{\widetilde A}(\cdot,a)\|
    =O_p((NT)^{-\alpha_{\widetilde A}}), \\
\max_{a\in\{0,1\}}\|\delta \widehat \theta_R(\cdot,a)\|=O_p((NT)^{-\alpha_R}),\qquad
\max_{a\in\{0,1\}}\|\delta \widehat \theta_{S'}(\cdot,a)\|
    =O_p((NT)^{-\alpha_{S'}}), \\
\max_{a\in\{0,1\}}\|e_M(\cdot,a)\|=O_p((NT)^{-\alpha_M}), \qquad \alpha_{\min}=\min\left(\alpha_\omega,\alpha_{\widetilde A},\alpha_R,\alpha_{S'},\alpha_M\right).
\end{gathered}
\]
Under additional rate conditions $\alpha_{\min}>1/4$, the LURE estimator admits the expansion
\begin{align*}
\widehat V(\pi)-V(\pi)
=
\frac{1}{NT} \sum_{i=1}^N\sum_{t=1}^T
\varphi^\pi(O_{i,t})
+o_p((NT)^{-1/2}).
\end{align*} Consequently, asymptotic normality follows by applying the martingale central limit theorem \citep{mcleish1974dependent} to the leading term.

\begin{theorem}[Asymptotic normality]
\label{thm:asymptotic_normality}
Under Assumptions~\ref{assump:latent_action_id} and \ref{assump:regularity}, suppose further that $\alpha_{\min}>1/4$ and, for some $\kappa>0$, $\max_{a\in\{0,1\}}\|\delta \widehat \theta_{\widetilde A}(\cdot,a)\|^\kappa=o_p((NT)^{-1/2}).$
We have
\[
\sqrt{NT}\big(\widehat V(\pi) - V(\pi)\big) \xrightarrow{d} \mathcal{N}\big(0,\sigma^2_\pi\big),
\]
where $\sigma^2_\pi = \Var(\varphi^\pi(O))$ is the variance of the influence function in Equation~\eqref{eq:IF}.
\end{theorem}

Theorem~\ref{thm:asymptotic_normality} establishes that the proposed estimator is asymptotically normal under the stated rate conditions. This enables the construction of Wald-type confidence intervals:
\[
\widehat V(\pi) \pm z_{\alpha/2} \cdot \frac{\widehat\sigma_\pi}{\sqrt{NT}},
\]
where $z_{\alpha/2}$ is the $(1-\alpha/2)$-quantile of the standard normal distribution, and $\widehat\sigma_\pi^2$ can be obtained by computing the sample variance of the estimated influence function, i.e., \[
\widehat\sigma_\pi^2 = \frac{1}{NT}\sum_{i=1}^N\sum_{t=1}^T \big(\widehat\varphi^\pi(O_{i,t})\big)^2.
\] The following proposition guarantees that $\widehat\sigma_\pi^2$ is consistent for the true asymptotic variance $\sigma_\pi^2$, ensuring that the confidence intervals have correct asymptotic coverage.

\begin{prop}[Consistent Variance Estimator] \label{prop:var}
Under the conditions of Theorem~\ref{thm:asymptotic_normality}, $$\sigma_\pi^2= \widehat\sigma_\pi^2+o_p(1).$$
\end{prop}

\section{Numerical Studies} \label{sec: sim}

In this section, we evaluate the finite-sample performance of the proposed LURE estimator through synthetic experiments and a benchmark dataset. 

\textbf{Competing methods}. We compare our proposal with representative direct, importance-sampling, and doubly robust estimators that ignore the latent-action structure and instead treat the observed surrogate $\widetilde A$ as if it were the true action. Specifically, we consider the following competing methods: sequential importance sampling \citep[SIS,][]{precup2000eligibility}, fitted $Q$ evaluation \citep[FQE,][]{ernst2005tree}, least-squares temporal difference \citep[LSTD,][]{shi2022statistical}, marginalized importance sampling \citep[MIS,][]{uehara2020minimax}, and double reinforcement learning \citep[DRL,][]{kallus2022efficiently}. 
Crucially, all five baselines use the observed surrogate $\widetilde A$ in place of the true action, so they are expected to exhibit systematic bias when misclassification is present. 

Throughout, we set the discount factor to $\gamma=0.7$, generate $N=50$ independent offline trajectories of length $T=50$ in each replication, and vary the misclassification rate over four scenarios $ \Pr(A \neq \widetilde A) \in\{0.05,0.1,0.2,0.3\}$. The surrogate action is subject to misclassification, generated according to $\Pr(\widetilde A\neq A)=\tau$, independently of $(S,A)$; equivalently, $\theta_{\widetilde A}(s,1)=1-\tau$ and $\theta_{\widetilde A}(s,0)=\tau$ for all $s$. We consider $\tau\in\{0.05,0.1,0.2,0.3\}$, corresponding to Scenarios 1--4.

\textbf{\textit{Proxy variable selection}}. Recall that $\theta_{S'}(s,a)=\E\!\left[l(S')\mid S=s,A=a\right]$ denotes the conditional expectation of a known feature mapping $l(\cdot)$ of the next state, given the current state and action. In implementation, we restrict attention to coordinate-wise features $l(s) = s^{(j)}$, the $j$-th component of the state variable, so that $\theta_{S'}(s,a)=\E\!\left[S'^{(j)}\mid S=s,A=a\right]$ represents the corresponding conditional mean. To construct $g_a'(O)$, we select the coordinate $j$ based on the Pearson partial correlation between $S'^{(j)}$ and observed $\widetilde A$ conditional on $S$, and choose the coordinate with the largest absolute value.

\subsection{Tabular MDP}

\textbf{Data-generating process.}
We consider a tabular MDP with $|\mathcal{S}|=3$ states, binary actions $\mathcal{A}=\{0,1\}$. All environment parameters are fixed across replications and specified as follows. The transition kernel is given by
\[
q(\cdot\mid s, 0)=
\begin{pmatrix}
0.2 & 0.5 & 0.3\\
0.1 & 0.3 & 0.6\\
0.05 & 0.25 & 0.7
\end{pmatrix},\qquad
q(\cdot\mid s, 1)=
\begin{pmatrix}
0.7 & 0.2 & 0.1\\
0.4 & 0.4 & 0.2\\
0.2 & 0.5 & 0.3
\end{pmatrix},
\]
where each row specifies the distribution of the next state $S'$. The reward function is $\theta_R(s,a)$ with $\theta_R(\cdot,0)=(1.0,\,0.5,\,0.0)^\top$ and $\theta_R(\cdot,1)=(2.0,\,1.5,\,0.5)^\top$, and observed rewards follow
\[
R_{i,t}=\theta_R(S_{i,t},A_{i,t})+\varepsilon_{i,t},\qquad \varepsilon_{i,t}\sim\mathcal{N}(0,0.5^2).
\]
\begin{figure}[h]
    \centering
\includegraphics[width=1\textwidth]{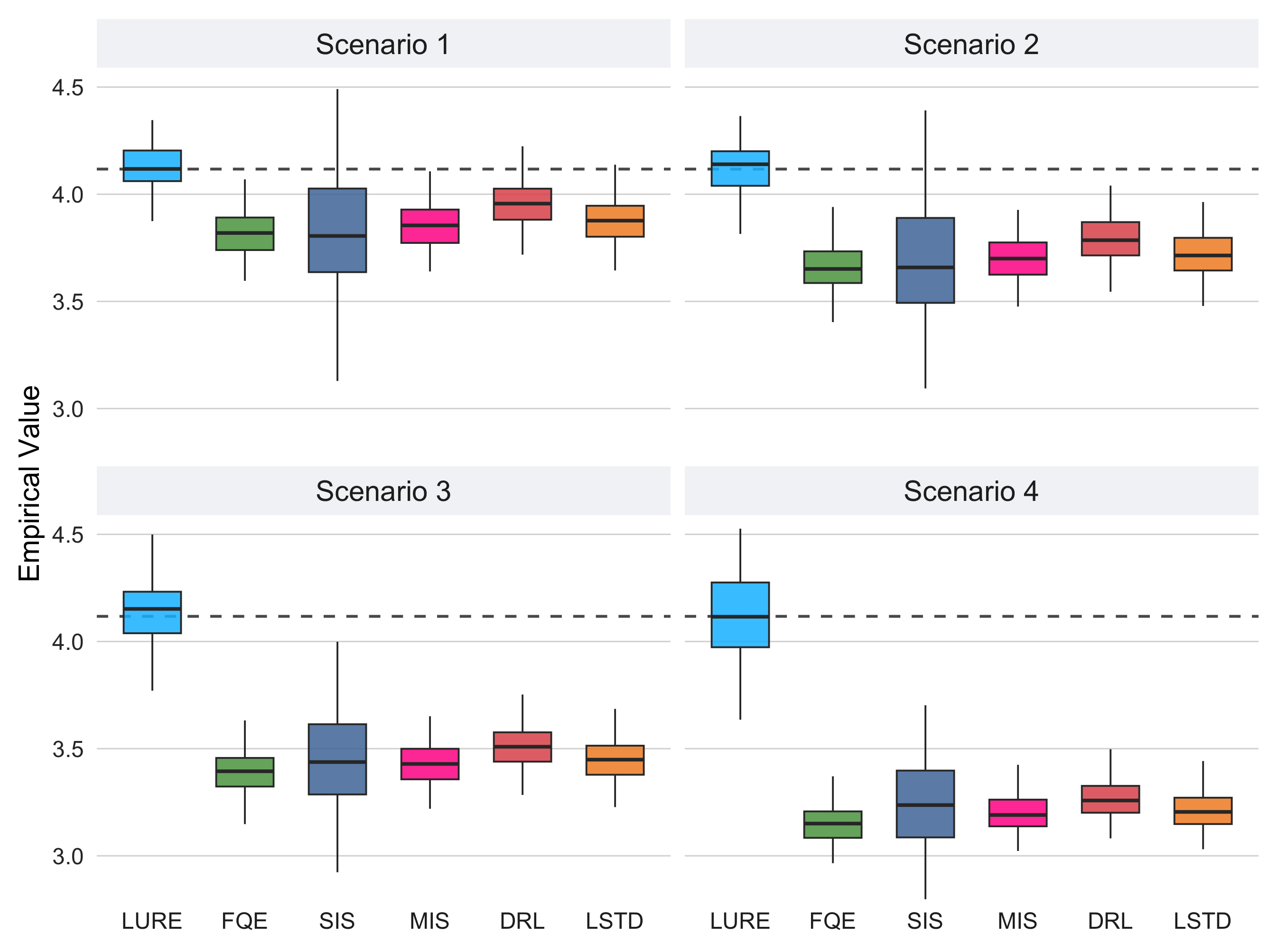}
    \caption{Empirical distributions of the estimated policy value under Scenarios~1 ($\tau=0.05$), 2 ($\tau=0.1$), 3 ($\tau=0.2$), and 4 ($\tau=0.3$), based on $N=50$ trajectories of length $T=50$ over $100$ replications. The dotted vertical line denotes the true policy value.}
    \label{fig:tab}
\end{figure}
The behavior policy satisfies $b(1\mid s)=(0.3,\,0.5,\,0.7)$ for $s=1,2,3$, while the target policy is given by $\pi(1\mid s)=0.8$ for all $s$.  The initial state distribution $p_e$ is uniform over $\mathcal{S}$.
In this tabular setting, the true policy value can be computed analytically via the Bellman equations, and all nuisance functions in LURE are estimated nonparametrically using weighted empirical frequencies and conditional means within each state or state--action cell.

\begin{table}[H]
\centering
\caption{Empirical coverage rates of the proposed $95\%$ confidence intervals under the Tabular MDP, General MDP, and OpenAI Gym environments across four scenarios.}
\label{tab:coverage}
\begin{tabular}{cS[table-format=1.2]S[table-format=1.2]S[table-format=1.2]}
\toprule
Scenario & {Tabular MDP} & {General MDP} & {OpenAI Gym} \\
\midrule
1 & 0.96 & 0.94 & 0.94 \\
2 & 0.98 & 0.95 & 0.96 \\
3 & 0.92 & 0.94 & 0.95 \\
4 & 0.96 & 0.91 & 0.94 \\
\bottomrule
\end{tabular}
\end{table}

Figure~\ref{fig:tab} displays the empirical distributions of the estimated policy value across the four scenarios. All competing methods exhibit noticeable bias, which becomes more pronounced as the misclassification rate increases. In contrast, LURE remains well-centered around the true policy value across all settings. Table~\ref{tab:coverage} also reports the empirical coverage rates of the proposed $95\%$ confidence intervals for the LURE estimator, which are consistently close to the nominal level across all scenarios. We note that, among the competing methods, both LSTD and DRL can in principle be accompanied by confidence intervals. However, because these estimators exhibit substantial bias in our simulations, their empirical coverage is extremely low. Therefore, we report coverage only for the proposed LURE estimator.

\subsection{General MDP}

\begin{figure}[H]
    \centering
    \includegraphics[width=1\textwidth]{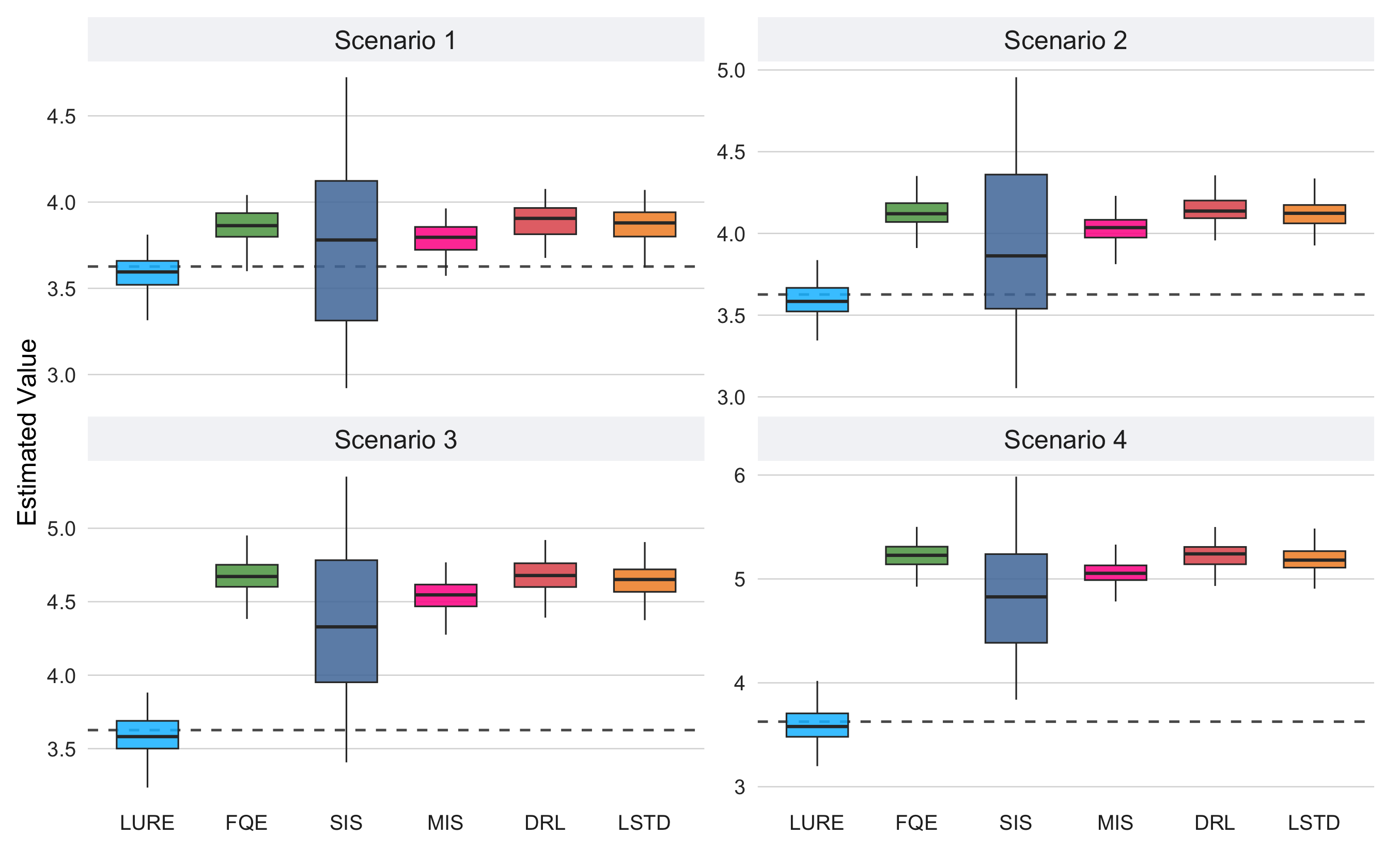}
    \caption{Empirical distributions of the estimated policy value under Scenarios~1 ($\tau=0.05$), 2 ($\tau=0.10$), 3 ($\tau=0.20$), and 4 ($\tau=0.30$), with $N=50$ trajectories of length $T=50$.}
    \label{fig:linear}
\end{figure}
We next consider a continuous-state MDP following a setup similar to that in \citet{shi2022statistical} and \citet{wang2023projected}, but with the addition of surrogate action misclassification. The state space is $\mathcal{S}=\mathbb{R}^2$, so that the state at time $t$ in trajectory $i$ is denoted by $S_{i,t}=\left(S_{i,t}^{(1)},S_{i,t}^{(2)}\right)^\top\in\mathbb{R}^2$, with binary actions $\mathcal{A}=\{0,1\}$ and discount factor $\gamma=0.7$. The initial state is drawn from independent Gaussians,
$S_{1}^{(1)}\sim \mathcal{N}(0.25,0.75^2),
$ and $S_{1}^{(2)}\sim \mathcal{N}(0.05,0.75^2).$
The offline behavior policy is state-independent with $b(1\mid s)=0.5$.
Given the current state $S_t$ and true action $A_t$, the dynamics are \begin{gather*}
    S_{t+1}^{(1)}
= \frac{1}{2}S_t^{(1)}+\frac{3}{5}(2A_t-1)+0.4S_t^{(1)}A_t+\varepsilon_{t}^{(1)},\\
\mbox{and } S_{t+1}^{(2)}= \frac{1}{3}S_t^{(2)}-\frac{3}{5}(2A_t-1)-0.3S_t^{(2)}A_t+\varepsilon_{t}^{(2)},
\end{gather*}
where $\varepsilon_{t}^{(1)},\varepsilon_{t}^{(2)}\stackrel{\mathrm{i.i.d.}}{\sim}\mathcal{N}(0,0.5^2)$. The reward is generated according to
$R_t=1+S_t^{(1)}+\frac{1}{2}S_t^{(2)}+\frac{3}{2}A_t+\varepsilon_{t},$ where $\varepsilon_{t}\sim\mathcal{N}(0,0.5^2).$
 We consider the following deterministic target policy $\pi(1\mid s)=\mathds{1}\left(s^{(1)}\ge 0.25,\ s^{(2)}\ge -0.10\right).$ For the LURE estimator, we use linear features to estimate the nuisance functions. The same linear specification is also used for all the competing methods.

Figure~\ref{fig:linear} shows the empirical distributions of the estimated policy value across the four scenarios. Similar to the tabular case, all competing methods exhibit bias that worsens with higher misclassification rates, while LURE remains well-centered around the true policy value. The empirical coverage rates of the proposed confidence intervals are also close to the nominal level across all scenarios, as reported in Table~\ref{tab:coverage}.

\subsection{OpenAI Gym CartPole}

We evaluate the proposed method in the CartPole environment using offline data with horizon $T=50$ and sample size $N=50$. The state is four-dimensional, denoted by $S=(S^{(1)},S^{(2)},S^{(3)},S^{(4)})^\top$, where $S^{(1)}$ is the cart position, $S^{(2)}$ is the cart velocity, $S^{(3)}$ is the pole angle, and $S^{(4)}$ is the pole angular velocity. The action $A\in\{0,1\}$ corresponds to applying a force to the cart in one of two directions. As in the previous setting, the behavior policy is state-independent, with $b(1\mid s)=0.5$ for all states $s$, and the target policy is deterministic,
\[
\pi(1\mid S)=\mathds{1}\left(S^{(1)}>-0.5,\; S^{(3)}<0.1\right).
\]

Similarly to \citet{uehara2020minimax} and \citet{shi2021deeply}, we modify the standard CartPole environment as follows. We add Gaussian noise with standard deviation $0.05$ to the next state and use the resulting noisy state in all subsequent transitions. In addition, the deterministic next state generated by the CartPole dynamics is shifted by $0.8A$ in each coordinate before the state noise is added. The reward is generated according to
\[
R_t = 1-\frac{\left(S_t^{(1)}\right)^2}{11.52}-\frac{\left(S_t^{(3)}\right)^2}{288}+0.8A_t+\varepsilon_t,\qquad \varepsilon_t\sim \mathcal{N}(0,0.2^2),
\] for all $t$. Across all nuisance functions and all implemented methods, including the proposed method and the competing approaches, we use the same additive cubic B-spline specification.

\begin{figure}[H]
    \centering
    \includegraphics[width=1\textwidth]{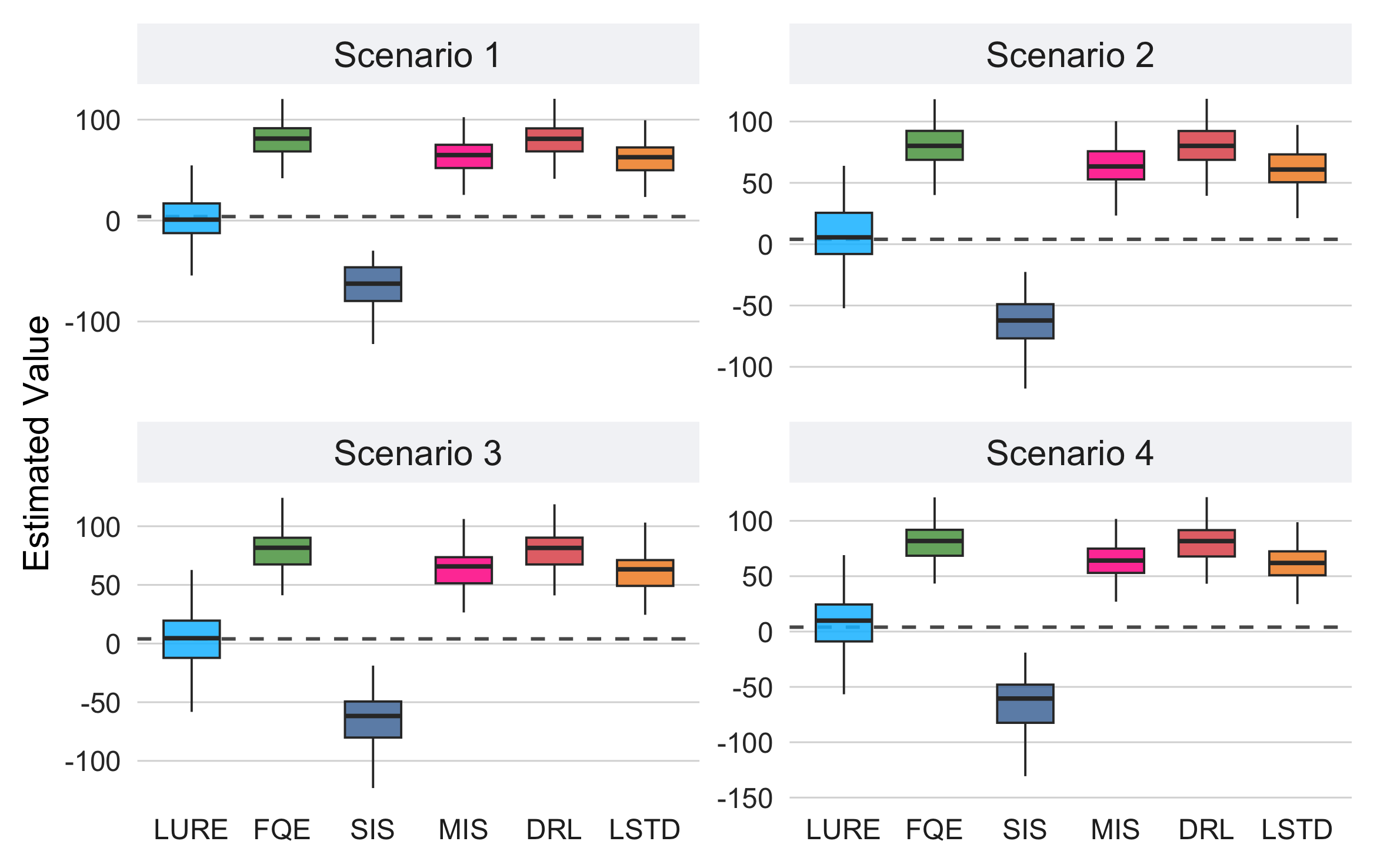}
    \caption{Empirical distributions of the estimated policy value under Scenarios~1 ($\tau=0.05$), 2 ($\tau=0.10$), 3 ($\tau=0.20$), and 4 ($\tau=0.30$), with $N=50$ trajectories of length $T=50$.}
    \label{fig:gym}
\end{figure}

Figure~\ref{fig:gym} summarizes the estimated policy value across the four scenarios for all the methods. As in the previous settings, the LURE estimator outperforms all competing methods. The empirical coverage rates of the proposed confidence intervals are also close to the nominal level across all scenarios, as reported in Table~\ref{tab:coverage}. Unlike the previous two settings, the performance of the competing methods remains relatively stable across different misclassification rates, which may be attributed to the more complex dynamics and reward structure in this environment that can partially mitigate the bias from treating $\widetilde A$ as the true action. Nevertheless, LURE still achieves superior performance by properly accounting for the latent-action structure.

\section{Real Data Analysis}\label{sec: real}

In this section, we apply the proposed method to a sepsis dataset from MIMIC-III \citep{johnson2016mimic}, a database containing information on critical care patients from Beth Israel Deaconess Medical Center. We focus on ICU stays with complete observations over 20 decision stages, yielding $T=20$ and $N=500$. The outcome is the sequential organ failure assessment (SOFA) score \citep{jones2009sequential}, where a larger score indicates more severe organ dysfunction and a higher risk of mortality. Following \citet{bian2025off}, we study a binary intravenous-fluid intervention to illustrate the method in a clinically interpretable setting. Specifically, at each decision time $t$, we consider $A_t \in \{0,1\}$, where $A_t=1$ indicates that intravenous fluid is administered and $A_t=0$ indicates no intravenous fluid. The state vector includes 46 demographic, physiological, laboratory, and clinical variables, such as age, gender, weight, and vital signs.

We evaluate three target policies. The first policy always administers IV, $\pi_{\mathrm{iv}}(1\mid s)=1.$ The second policy never administers IV, namely $\pi_{\mathrm{noiv}}(1\mid s)=0.$ Similarly to \citet{bian2025off}, we also evaluate a SOFA-tailored policy that administers IV when the current SOFA score is at least 3 and no IV otherwise, $\pi_{\mathrm{ta}}(1\mid s)=\mathds{1}\left(\mathrm{SOFA}\geq 3\right).$ This policy is motivated by the finding that a SOFA score greater than 2 reflects an overall mortality risk of approximately 10\% in a general hospital population with suspected infection \citep{singer2016third}.

We apply the proposed LURE estimator with discount factor $\gamma=0.9$. The nuisance functions are estimated using linear features of the state variables. For comparison, we also report the DRL estimator based on FQE and MIS. As in the simulation studies, we construct $\theta_{S'}(s,a)$ using the Pearson partial correlation between each next-state component $S'^{(j)}$ and the observed treatment $\widetilde A$ conditional on $S$, and choose as the proxy variable the component with the largest absolute partial correlation.

\begin{table}[t]
\centering
\caption{Estimated policy values for the MIMIC-III sepsis analysis using DRL and LURE. Smaller values indicate lower expected SOFA scores.}
\label{tab:mimic}
\begin{tabular}{lcccc}
\toprule
Policy  & DRL estimate & DRL 95\% CI & LURE estimate & LURE 95\% CI \\
\midrule
Always IV  & 66.29 & [63.78, 68.79] & 62.10 & [60.16, 64.04] \\
No IV  & 67.15 & [62.42, 71.88] & 68.63 & [67.03, 70.22]\\
SOFA-tailored rule & 66.91 & [64.27, 69.55] & 65.92 & [64.30, 67.53] \\
\bottomrule
\end{tabular}
\end{table}

Table~\ref{tab:mimic} reports the estimated policy values. Because smaller values correspond to better outcomes, the LURE estimator favors the always-IV policy, with an estimated SOFA score of 62.10, followed by the SOFA-tailored rule at 65.92 and the no-IV policy at 68.63. The 95\% confidence interval for the always-IV policy, [60.16, 64.04], does not overlap with the intervals for either the SOFA-tailored rule or the no-IV policy. By contrast, the DRL confidence intervals overlap substantially across all three policies. These results suggest that accounting for the hidden-action structure may lead to different policy evaluations in this sepsis application, highlighting the potential impact of hidden actions on empirical conclusions.

\section{Discussion} \label{sec: discuss}

This paper studies off-policy evaluation when the true action is hidden and only a noisy surrogate is observed. By exploiting the next state as an internal proxy for the hidden action, we establish identification of the policy value and derive an observed-data influence function. The resulting LURE estimator is multiply robust and asymptotically normal, which supports valid large-sample inference under flexible nuisance estimation. A novel EM-based procedure is proposed for estimating the nuisance functions.
Across tabular, continuous-state, and OpenAI Gym environments, the simulations show that explicitly accounting for hidden actions can reduce bias relative to standard OPE methods that treat the surrogate action as the truth. 


\spacingset{1.6}

\bibliographystyle{apalike}

\bibliography{references}

\appendix

\spacingset{1.8}

\newpage

\section{Auxiliary Results and Proofs}

\subsection{Lemmas}

Throughout, we use $c$ and $C$ to denote a generic constant that can vary from line to line. Moreover, to simplify notation, we sometimes write $\varphi^\pi(O;\Theta_0)$ as $\varphi^\pi(O)$. We use $\varphi^\pi(O;\widehat\Theta)$ to denote the influence function obtained by replacing the true nuisance functions with their estimated counterparts.  We also write \(n=NT\).

\begin{lemma} \label{lemma:balancing}
For any measurable function $f(s,a)$, the following holds.
    \begin{gather*}
        \E_{p_e}(f(S;\pi))+(1-\gamma)^{-1}\E\left[\omega(S,A)(\gamma f(S';\pi)-f(S,A)) \right]=0. 
    \end{gather*}
\end{lemma}

\begin{proof}
    \begin{align*}
        &\E_{p_e}(f(S;\pi))=\int_s \sum_a f(s,a) \underbrace{\pi(a|s) p_e(s)}_{p_1^\pi(s,a)} ds\\
       =& \; \int_s \sum_a f(s,a) \left[\underbrace{p_1^\pi(s,a)+\sum_{t=2}^\infty \gamma^{t-1} p_t^\pi(s,a)}_{(1-\gamma)^{-1}p^*(s,a)}-\sum_{t=2}^\infty \gamma^{t-1} p_t^\pi(s,a) \right] ds\\
       =& \; (1-\gamma)^{-1}\E\left[\omega(S,A)f(S,A) \right]-\int_s \sum_a f(s,a)\sum_{t=2}^\infty \gamma^{t-1} p_t^\pi(s,a) ds.
    \end{align*} It thus remains to show that \begin{gather*}
        \int_s \sum_a f(s,a)\sum_{t=2}^\infty \gamma^{t-1} p_t^\pi(s,a) ds=(1-\gamma)^{-1}\E\left[\gamma \omega(S,A) f(S';\pi)\right].
    \end{gather*} It follows that 
\begin{align*}
       &\int_s \sum_a f(s,a)\sum_{t=2}^\infty \gamma^{t-1} p_t^\pi(s,a) ds=\int_{s'} \sum_{a'} f(s',a')\sum_{t=2}^\infty \gamma^{t-1} p_{t}^\pi(s',a') ds'\\
       =& \; \int_{s'} \sum_{a'} f(s',a')\sum_{t=1}^\infty \gamma^{t} p_{t+1}^\pi(s',a') ds=\int_{s,s'} \sum_{a,a'} f(s',a')\sum_{t=1}^\infty \gamma^{t} p_{t+1}^\pi(s',a'|s,a)p_{t}^\pi(s,a) ds' ds\\
       =& \; \int_{s,s'} \sum_{a,a'} f(s',a')\sum_{t=1}^\infty \gamma^{t} \pi(a'|s')p(s'|s,a)p_{t}^\pi(s,a) ds' ds\\
         =& \; \int_{s,s'}  \sum_{a} f(s',\pi)p(s'|s,a) \underbrace{\sum_{t=1}^\infty \gamma^{t} p_{t}^\pi(s,a)}_{\gamma(1-\gamma)^{-1}p^*(s,a)} ds' ds=(1-\gamma)^{-1}\E\left[\gamma \omega(S,A) f(S';\pi)\right].
     \end{align*} This completes the proof.
\end{proof}

\begin{lemma} \label{lemma:indicator}
    \begin{gather*}
\E[g_a(R,\widetilde A, S)\mid S,A]=\E[g_a'(S',\widetilde A, S)\mid S,A]=\mathds 1(A=a).
    \end{gather*} 
\end{lemma}

\begin{proof}
The proof is based on the proof of Theorem 4 in \citet{zhou2024causal}.
    \begin{align*}
& \E[g_a(R,\widetilde A, S)\mid S,A]= \E\left[\frac{\widetilde A-\theta_{\widetilde A}(S,1-a)}{\theta_{\widetilde A}(S,a)-\theta_{\widetilde A}(S,1-a)} \frac{R-\theta_{R}(S,1-a)}{\theta_{R}(S,a)-\theta_{R}(S,1-a)} \middle| S,A\right]\\
        = & \E\left[\frac{\widetilde A-\theta_{\widetilde A}(S,1-a)}{\theta_{\widetilde A}(S,a)-\theta_{\widetilde A}(S,1-a)}  \middle| S,A\right] \E\left[\frac{R-\theta_{R}(S,1-a)}{\theta_{R}(S,a)-\theta_{R}(S,1-a)} \middle| S,A\right],
    \end{align*}where the last equality holds by the conditional independence assumption $R\independent \widetilde A| S,A$.

    It follows that 
  \begin{align*}
\E[g_a(R,\widetilde A, S)\mid S,A]= \frac{\theta_{\widetilde A}(S,A)-\theta_{\widetilde A}(S,1-a)}{\theta_{\widetilde A}(S,a)-\theta_{\widetilde A}(S,1-a)}  \frac{\theta_{R}(S,A)-\theta_{R}(S,1-a)}{\theta_{R}(S,a)-\theta_{R}(S,1-a)} =\mathds 1(A=a)
    \end{align*}

Following the same derivation, we also have 
\begin{gather*}
\E[g_a'(S',\widetilde A, S)\mid S,A]=\mathds 1(A=a).
    \end{gather*}
This completes the proof.
\end{proof} 

\begin{lemma} \label{lemma:IF_condition}
A function $H(\widetilde{A},R,S',S)$ is an influence function for $V(\pi)$ in the observed data model if and only if there exists $\text{IF}^F$, an influence function for the full data model where $A$ is observed, such that 
\begin{align}
\mathbb E(H\mid \widetilde A,S,A) &= \mathbb E(\text{IF}^F\mid \widetilde A,S,A), \nonumber\\
\mathbb E(H\mid R,S,A) &= \mathbb E(\text{IF}^F\mid R,S,A), \\
\mathbb E(H\mid S',S,A) &= \mathbb E(\text{IF}^F\mid S',S,A). \nonumber
\end{align}
\end{lemma}

\begin{proof}
Let \(O^F=(S,A,\widetilde A,R,S')\) denote the full data, where \(A\) is observed, and let \(O=(S,\widetilde A,R,S')\) denote the observed data, where \(A\) is hidden. Let \(\mathcal G=\sigma(O)\). Under the full-data model, the density factorizes as
\[
p(o^F)=
p_S(s)p_{A\mid S}(a\mid s)
p_{\widetilde A\mid S,A}(\widetilde a\mid s,a)
p_{R\mid S,A}(r\mid s,a)
p_{S'\mid S,A}(s'\mid s,a).
\]
Hence, any full-data score can be written as
\[
S^F(O^F)=S_S(S)+S_{A\mid S}(A,S)+S_{\widetilde A\mid S,A}(\widetilde A,S,A)
+S_{R\mid S,A}(R,S,A)+S_{S'\mid S,A}(S',S,A),
\]
where the components satisfy:
\(\mathbb E[S_S(S)]=0\),
\(\mathbb E[S_{A\mid S}(A,S)\mid S]=0\),
\(\mathbb E[S_{\widetilde A\mid S,A}(\widetilde A,S,A)\mid S,A]=0\),
\(\mathbb E[S_{R\mid S,A}(R,S,A)\mid S,A]=0\), and
\(\mathbb E[S_{S'\mid S,A}(S',S,A)\mid S,A]=0\).

Let \(S(O)=\mathbb{E}\!\left(S^F\mid O\right)\) denote the observed-data score induced by the full-data score \(S^F\). Since \(H=H(O)\) is measurable with respect to the observed data,
\[
\mathbb{E}\{H S(O)\}=\mathbb{E}\!\left[H\,\mathbb{E}(S^F\mid O)\right]
=\mathbb{E}(H S^F).
\]

By the defining equation for an influence function, \(H\) is an observed-data influence function for \(V(\pi)\) if, for every regular parametric submodel,
\[
\mathbb{E}\{H S(O)\}=\left.\frac{d}{d\varepsilon}V_{P_\varepsilon}(\pi)\right|_{\varepsilon=0}.
\]
If \(IF^F\) is a full-data influence function, then
\[
\left.\frac{d}{d\varepsilon}V_{P_\varepsilon}(\pi)
\right|_{\varepsilon=0}=\mathbb{E}(IF^F S^F).
\]
Therefore, \(H\) is an observed-data influence function if and only if
\[
\mathbb{E}\left[\left\{H-IF^F\right\}S^F\right]=0
\]
for every admissible full-data score \(S^F\).

Write \(D=H-IF^F\). Expanding against the score decomposition gives
\[
\begin{aligned}
\mathbb E(DS^F)
&=\mathbb E[\mathbb E(D\mid S)S_S(S)]+\mathbb E[\mathbb E(D\mid S,A)S_{A\mid S}(A,S)]\\
&\quad+\mathbb E[\mathbb E(D\mid \widetilde A,S,A)S_{\widetilde A\mid S,A}(\widetilde A,S,A)]\\
&\quad+\mathbb E[\mathbb E(D\mid R,S,A)S_{R\mid S,A}(R,S,A)]\\
&\quad+\mathbb E[\mathbb E(D\mid S',S,A)S_{S'\mid S,A}(S',S,A)].
\end{aligned}
\]

We now compare score components. Since \(S_S(S)\) is arbitrary subject to \(\mathbb E[S_S(S)]=0\), the condition \(\mathbb E[\mathbb E(D\mid S)S_S(S)]=0\) for all such scores implies that \(\mathbb E(D\mid S)\) is constant. Since both \(H\) and \(IF^F\) are taken to be mean-zero influence functions, \(\mathbb E(D)=0\), and hence \(\mathbb E(D\mid S)=0\). Next, since \(S_{A\mid S}\) is arbitrary subject to \(\mathbb E(S_{A\mid S}\mid S)=0\), we obtain \(\mathbb E(D\mid S,A)=\mathbb E(D\mid S)=0\).

Next, since \(S_{\widetilde A\mid S,A}\) is arbitrary subject to \(\mathbb E(S_{\widetilde A\mid S,A}\mid S,A)=0\), we obtain \(\mathbb E(D\mid \widetilde A,S,A)=\mathbb E(D\mid S,A)=0\). Equivalently,
\[
\mathbb E(H\mid \widetilde A,S,A)=\mathbb E(IF^F\mid \widetilde A,S,A).
\]
The same argument applied to the \(R\mid S,A\) and \(S'\mid S,A\) score components gives
\[
\mathbb E(H\mid R,S,A)=\mathbb E(IF^F\mid R,S,A),
\qquad
\mathbb E(H\mid S',S,A)=\mathbb E(IF^F\mid S',S,A).
\]
This proves the necessity of the three conditional equations.

Conversely, suppose there exists a full-data influence function \(IF^F\) satisfying the three conditional equations. Then, with \(D=H-IF^F\), the first equation gives \(\mathbb E(D\mid \widetilde A,S,A)=0\), and hence by iterated expectation \(\mathbb E(D\mid S,A)=0\) and \(\mathbb E(D\mid S)=0\). The other two equations give \(\mathbb E(D\mid R,S,A)=0\) and \(\mathbb E(D\mid S',S,A)=0\). Therefore every term in the expansion of \(\mathbb E(DS^F)\) is zero, so \(\mathbb E(DS^F)=0\) for every full-data score \(S^F\). Consequently,
\[
\mathbb E(HS^F)=\mathbb E(IF^F S^F)=\left.\frac{d}{d\varepsilon}V_{P_\varepsilon}(\pi)\right|_{\varepsilon=0}.
\]
Since \(\mathbb E(HS^F)=\mathbb E\{H\,\mathbb E(S^F\mid \mathcal G)\}\), this is exactly the defining condition for \(H\) to be an observed-data influence function. The proof is complete.
\end{proof}

\begin{lemma} \label{lem:EP}
    The empirical-process remainder
\[
E_n=(\mathbb P_n-\mathbb P)\{\varphi^\pi(O;\widehat\Theta)-\varphi^\pi(O;\Theta_0)\}=o_p((NT)^{-1/2}).
\]
\end{lemma}

\begin{proof}
Let the total sample of size $n$ be partitioned into $K$ folds $\mathcal{I}_1,\ldots,\mathcal{I}_K$, where fold $k$ has size $n_k$. Let $\widehat{\Theta}^{(-k)}$ denote the nuisance estimator constructed using observations outside fold $k$, and define
\[
g_k(O)=\varphi^\pi\!\left(O;\widehat{\Theta}^{(-k)}\right)-\varphi^\pi(O;\Theta_0).
\]

For fold $k$, let
\[
\mathbb{P}_{n,k}f=\frac{1}{n_k}\sum_{i\in\mathcal{I}_k} f(O_i),
\]
and define
\[
E_{n,k}=(\mathbb{P}_{n,k}-\mathbb{P})g_k.
\]

By cross-fitting, $\widehat{\Theta}^{(-k)}$ is estimated using observations outside $\mathcal{I}_k$. Hence, conditional on the training sample, $g_k$ is fixed and independent of the observations in fold $k$.
Therefore,
\[
\mathbb{E}\left(E_{n,k}\mid\widehat{\Theta}^{(-k)}\right)=0,
\]
and
\[
\operatorname{Var}\left(E_{n,k}\mid\widehat{\Theta}^{(-k)}\right)
\leq\frac{1}{n_k}\|g_k\|^2.
\]

By Chebyshev's inequality,
\[
E_{n,k}=O_p\left(n_k^{-1/2}\|g_k\|\right).
\]

By Assumption~\ref{assump:regularity} and the nuisance-function convergence rates,
\[
\|g_k\|=\left\|\varphi^\pi\left(\cdot;\widehat{\Theta}^{(-k)}\right)
-\varphi^\pi(\cdot;\Theta_0)\right\|=O_p\left(n^{-\alpha_{\min}}\right)=o_p(1)
\]
for each $k$. Hence,
\(E_{n,k}=o_p(n_k^{-1/2}).\)

The full cross-fitted empirical-process term is
\[
E_n=\sum_{k=1}^K\frac{n_k}{n}E_{n,k}.
\]
Since $K$ is fixed and $n_k\asymp n$ for every $k$,
\[
E_n=o_p(n^{-1/2}).
\]
\end{proof}

\subsection{Proof of Theorem \ref{thm:identification}}

\begin{proof}
Under Assumption \ref{assump:latent_action_id}, \citet{zhou2024causal} proved that \begin{gather*}
    \Pr(\widetilde A=a'\mid A=a,S=s)
\end{gather*} is identifiable for all $s$, $a'$ and $a$.

Thus, we have \begin{gather} \label{eq:bayes}
    \Pr(\widetilde A=a'\mid S=s)=\sum_{a} \Pr(\widetilde A=a'\mid A=a, S=s)\Pr(A=a\mid  S=s).
\end{gather} Writing Equation \eqref{eq:bayes} in matrix form, we have \begin{gather*}
    P_{\widetilde A}(s)=P_{\widetilde A,A}(s) P_{A}(s),
\end{gather*} where \begin{gather*}
    P_{\widetilde A}(s)=\left[\Pr(\widetilde A=0\mid S=s),\Pr(\widetilde A=1\mid S=s)\right]^\top,
\end{gather*}
\begin{equation*}
   P_{\widetilde A,A}(s)= \begin{pmatrix}
        \Pr(\widetilde A=0\mid A=0, S=s) & \Pr(\widetilde A=0\mid A=1, S=s)\\
         \Pr(\widetilde A=1\mid A=0, S=s) & \Pr(\widetilde A=1\mid A=1, S=s)
    \end{pmatrix},
\end{equation*} and
\begin{gather*}
    P_{A}(s)=\left[\Pr( A=0\mid S=s),\Pr( A=1\mid S=s)\right]^\top,
\end{gather*} By Assumption \ref{assump:latent_action_id} (iv), the matrix $P_{\widetilde A,A}(s)$ is non-singular for all $s$, thus
we have \begin{gather*}
    P_{A}(s)=P_{\widetilde A,A}^{-1}(s) P_{\widetilde A}(s),
\end{gather*} for all $s$. Since $P_{\widetilde A,A}^{-1}(s)$ is identifiable, and $P_{\widetilde A}(s)$ is a function of the observed data, $P_{A}(s)$ is identifiable for all $s$.

Recall the definition of the discounted visitation
\begin{align*}
 p^*(s,a)
 &:=(1-\gamma)\sum_{t=1}^\infty \gamma^{t-1} p_t^\pi(s,a).
\end{align*}
Under the coverage condition in Assumption \ref{assump:latent_action_id} (i), the marginal density ratio
\begin{align*}
\omega(s,a):=\frac{p^*(s,a)}{p^b(s,a)}
\end{align*} is well defined and exists for all $s$ and $a$. Moreover, since $\pi(a|s)$ is known, and $P_{A}(s)$ is identifiable for all $s$, the density ratio $\omega(s,a)$ is also identifiable from the observed data. Then,
\begin{align*}
    &V(\pi)
    =(1-\gamma)^{-1} \int_{s,a,r} r\, p(r\mid s,a) p^*(s,a)\, ds\, da\, dr
    \\
    = & \;(1-\gamma)^{-1} \E\left[\frac{p^*(S,A)}{p^b(S,A)}R\right]
    =(1-\gamma)^{-1} \E\left[\omega(S,A)R\right].
\end{align*} 



\end{proof}

\subsection{Proof of Theorem \ref{thm:influence_function}}

\begin{proof}
Let \(O^F=(S,A,\widetilde A,R,S')\) denote the full data, where the true action \(A\) is observed, and let \(O=(S,\widetilde A,R,S')\) denote the observed data. 
Consider the full-data influence function
\[
IF^F=\mathbb E_{p_e}\{V^\pi(S)\} + \frac{1}{1-\gamma}\omega(S,A)\{R-\theta_R(S,A)\}
+\frac{\gamma}{1-\gamma}\omega(S,A)\{V^\pi(S')-(\mathcal M V^\pi)(S,A)\}-V(\pi).
\]
By Lemma \ref{lemma:IF_condition}, it suffices to show that \(\varphi^\pi(O)\) and \(IF^F\) have the same conditional expectations given \((\widetilde A,S,A)\), \((R,S,A)\), and \((S',S,A)\).

Let
\begin{align*}
\alpha_a(\widetilde A,S) &= \frac{\widetilde A-\theta_{\widetilde A}(S,1-a)}{\theta_{\widetilde A}(S,a)-\theta_{\widetilde A}(S,1-a)}, \\
\beta_a^{S'}(S',S) &= \frac{l(S')-\theta_{S'}(S,1-a)}{\theta_{S'}(S,a)-\theta_{S'}(S,1-a)},\\
\beta_a^R(R, S) &= \frac{R-\theta_{R}(S,1-a)}{\theta_{R}(S,a)-\theta_{R}(S,1-a)}.
\end{align*}
Then
\begin{align*}
g_a(R,\widetilde A,S) = \alpha_a(\widetilde A,S)\beta_a^R(R, S),\quad
g'_a(S',\widetilde A,S) = \alpha_a(\widetilde A,S) \beta_a^{S'}(S',S).
\end{align*}
Also note that
\begin{align*}
\E[\alpha_a(\widetilde A,S)\mid R,A,S] &= \E[\alpha_a(\widetilde A,S)\mid S',A,S] = \E[\alpha_a(\widetilde A,S)\mid A,S] = \mathds{1}(A=a), \\
\E[\beta^R_a(R,S)\mid \widetilde A,A,S] &= \E[\beta^R_a(R,S)\mid S',A,S] = \E[\beta^R_a(R,S)\mid A,S] = \mathds{1}(A=a),\\
\E[\beta^{S'}_a(S',S)\mid \widetilde A,A,S] &= \E[\beta^{S'}_a(S',S)\mid R,A,S] = \E[\beta^{S'}_a(S',S)\mid A,S] = \mathds{1}(A=a).
\end{align*}

We now verify the three conditional equalities required by Lemma \ref{lemma:IF_condition}.

First, condition on $(\widetilde A,S,A)$. For the reward component of
$\varphi^\pi(O)$,
\[
\begin{aligned}
&\mathbb E\left[\sum_a
g_a'(S',\widetilde A,S)\omega(S,a)\{R-\theta_R(S,a)\}\middle|\widetilde A,S,A\right] \\
&=\sum_a\alpha_a(\widetilde A,S)\omega(S,a)\mathbb E\{\beta_a^{S'}(S',S)\mid S,A\}\mathbb E\{R-\theta_R(S,a)\mid S,A\} \\
&=\sum_a\alpha_a(\widetilde A,S)\omega(S,a)\mathbbm 1(A=a)\{\theta_R(S,A)-\theta_R(S,a)\} \\
&=0.
\end{aligned}
\]

Similarly, for the transition component,
\[
\begin{aligned}
&\mathbb E\left[\sum_ag_a(R,\widetilde A,S)\omega(S,a)\{V^\pi(S')-(\mathcal M V^\pi)(S,a)\}\middle|\widetilde A,S,A\right] \\
&=\sum_a\alpha_a(\widetilde A,S)\omega(S,a)\mathbb E\{\beta_a^R(R,S)\mid S,A\} 
\mathbb E\left[V^\pi(S')-(\mathcal M V^\pi)(S,a)\mid S,A\right] \\
&=\sum_a\alpha_a(\widetilde A,S)\omega(S,a)\mathbbm 1(A=a)\{(\mathcal M V^\pi)(S,A)-(\mathcal M V^\pi)(S,a)\} \\
&=0.
\end{aligned}
\]

Therefore,
\[
\mathbb E\{\varphi^\pi(O)\mid \widetilde A,S,A\}=\mathbb E_{p_e}\{V^\pi(S)\}-V(\pi).
\]
The two residual terms in $IF^F$ also have conditional mean zero given $(\widetilde A,S,A)$. Hence,
\[
\mathbb E\{\varphi^\pi(O)\mid \widetilde A,S,A\}=\mathbb E\{IF^F\mid \widetilde A,S,A\}.
\]

Second, condition on $(R,S,A)$. For the reward component,
\[
\begin{aligned}
&\mathbb E\left[\sum_ag_a'(S',\widetilde A,S)\omega(S,a)\{R-\theta_R(S,a)\}\middle|R,S,A\right] \\
&=\sum_a\omega(S,a)\{R-\theta_R(S,a)\}\mathbb E\{g_a'(S',\widetilde A,S)\mid R,S,A\} \\
&=\sum_a\omega(S,a)\{R-\theta_R(S,a)\}\mathbbm 1(A=a) \\
&=\omega(S,A)\{R-\theta_R(S,A)\}.
\end{aligned}
\]

For the transition component,
\[
\begin{aligned}
&\mathbb E\left[\sum_ag_a(R,\widetilde A,S)\omega(S,a)\{V^\pi(S')-(\mathcal M V^\pi)(S,a)\}
\middle|R,S,A\right] \\
&=\sum_a\beta_a^R(R,S)\omega(S,a)
\mathbb E\{\alpha_a(\widetilde A,S)\mid S,A\} 
\mathbb E\left[V^\pi(S')-(\mathcal M V^\pi)(S,a)\mid S,A\right] \\
&=\sum_a\beta_a^R(R,S)\omega(S,a)\mathbbm 1(A=a)\{(\mathcal M V^\pi)(S,A)-(\mathcal M V^\pi)(S,a)\} \\
&=0.
\end{aligned}
\]

Thus,
\[
\begin{aligned}
\mathbb E\{\varphi^\pi(O)\mid R,S,A\}
&=\mathbb E_{p_e}\{V^\pi(S)\}-V(\pi) +\frac{1}{1-\gamma}\omega(S,A)\{R-\theta_R(S,A)\} \\
&=\mathbb E\{IF^F\mid R,S,A\}.
\end{aligned}
\]

Finally, condition on $(S',S,A)$. For the transition component,
\[
\begin{aligned}
&\mathbb E\left[\sum_ag_a(R,\widetilde A,S)\omega(S,a)\{V^\pi(S')-(\mathcal M V^\pi)(S,a)\}\middle| S',S,A\right] \\
&=\sum_a\omega(S,a)\{V^\pi(S')-(\mathcal M V^\pi)(S,a)\}\mathbb E\{g_a(R,\widetilde A,S)\mid S',S,A\} \\
&=\sum_a\omega(S,a)\{V^\pi(S')-(\mathcal M V^\pi)(S,a)\}\mathbbm 1(A=a) \\
&=\omega(S,A)\{V^\pi(S')-(\mathcal M V^\pi)(S,A)\}.
\end{aligned}
\]

For the reward component,
\[
\begin{aligned}
&\mathbb E\left[
\sum_ag_a'(S',\widetilde A,S)\omega(S,a)\{R-\theta_R(S,a)\}
\middle| S',S,A\right] \\
&=\sum_a\beta_a^{S'}(S',S)\omega(S,a)\mathbb E\{\alpha_a(\widetilde A,S)\mid S,A\} 
\mathbb E\{R-\theta_R(S,a)\mid S,A\} \\
&=\sum_a\beta_a^{S'}(S',S)\omega(S,a)\mathbbm 1(A=a)\{\theta_R(S,A)-\theta_R(S,a)\} \\
&=0.
\end{aligned}
\]

Therefore,
\[
\begin{aligned}
\mathbb E\{\varphi^\pi(O)\mid S',S,A\}
&=\mathbb E_{p_e}\{V^\pi(S)\}-V(\pi)  +\frac{\gamma}{1-\gamma}\omega(S,A)\{V^\pi(S')-(\mathcal M V^\pi)(S,A)\} \\
&=\mathbb E\{IF^F\mid S',S,A\}.
\end{aligned}
\]

Thus all three conditional equalities in Lemma \ref{lemma:IF_condition} are satisfied. Therefore, by Lemma \ref{lemma:IF_condition}, \(\varphi^\pi(O)\) is an influence function for \(V(\pi)\) in the observed-data model.
\end{proof}

\subsection{Proof of Theorem \ref{thm:error}}
Let the nuisance collection be \(\Theta=(\theta_R, M,\omega,\theta_{\widetilde A},\theta_{S'})\), where $M(s,a) := \mathcal{M}V^\pi(s,a).$

For $a\in\{0,1\}$, define the action-specific contrasts
\[
\Delta_{\tilde A,a}(s):=\theta_{\tilde A}(s,a)-\theta_{\tilde A}(s,1-a),
\quad
\Delta_{R,a}(s):=\theta_R(s,a)-\theta_R(s,1-a),
\quad
\Delta_{S',a}(s):=\theta_{S'}(s,a)-\theta_{S'}(s,1-a).
\]

Define
\begin{align*}
\alpha_a(\widetilde A,S;\Theta)=&\frac{\widetilde A-\theta_{\tilde A}(S,1-a)}{\Delta_{\tilde A,a}(S)},\\
\beta^R_a(R,S;\Theta)=\frac{R-\theta_R(S,1-a)}{\Delta_{R,a}(S)}, & \qquad
\beta^{S'}_a(S',S;\Theta)=\frac{l(S')-\theta_{S'}(S,1-a)}{\Delta_{S',a}(S)}.
\end{align*}
Then
$g_a(R,\widetilde A,S;\Theta) = \alpha_a(\widetilde A,S;\Theta)\beta^R_a(R,S;\Theta),\ 
g'_a(S',\widetilde A,S;\Theta) = \alpha_a(\widetilde A,S;\Theta)\beta^{S'}_a(S',S;\Theta).$

Define 
\begin{align}
\phi^\pi(O;\Theta)&=
\mathbb{E}_{p_e}[V^\pi_\Theta(S)]+\frac{1}{1-\gamma}\sum_a g'_a(O;\Theta)\,\omega(S,a)\,\{R-\theta_R(S,a)\} \nonumber\\
&\quad +\frac{\gamma}{1-\gamma}\sum_a g_a(O;\Theta)\,\omega(S,a)\,\{V^\pi_\Theta(S')-M(S,a)\}.
\label{eq:Psi}
\end{align}

When the nuisance functions take their true values \(\Theta_0\), the expectation of this quantity equals the target value \(V(\pi)\):
\[
V(\pi)=\mathbb{E}\{\phi^\pi(O;\Theta_0)\}.
\]
The corresponding influence function is
$\varphi^\pi(O)=\phi^\pi(O;\Theta_0)-V(\pi).$

\paragraph*{Step 1: Empirical-process term and label-selection error.}
We first separate the error due to estimating the labeling from the estimation error under the population-correct labeling. Recall that the final label-aligned estimator is
\[
\widehat V(\pi)=\widehat V_{\widehat d}(\pi)=\sum_{a=0,1} \mathds 1\{\widehat d=a\}\widehat V_a(\pi).
\]
We decompose
\begin{align}\label{eq:decomposition}
\widehat V(\pi)-V(\pi)=\{\widehat V_d(\pi)-V(\pi)\}+\{\widehat V_{\widehat d}(\pi)-\widehat V_d(\pi)\}.
\end{align}


The first term in \eqref{eq:decomposition} is the estimation error under the population-correct labeling, while the second term captures the effect of using the empirical label selector \(\widehat d\).
Under the population label selector $d$, write
\[
\widehat V_d(\pi)=\mathbb P_n\{\phi_d^\pi(O;\widehat\Theta)\}, \qquad V(\pi)=\mathbb P\{\phi_d^\pi(O;\Theta_0)\}.
\]

Adding and subtracting
\(\mathbb P_n\{\phi_d^\pi(O;\Theta_0)\}\) and
\(\mathbb P\{\phi_d^\pi(O;\widehat\Theta)\}\), we obtain
\[
\begin{aligned}
\widehat V_d(\pi)-V(\pi)
&=\mathbb P_n[\phi_d^\pi(O;\widehat\Theta)]-\mathbb P[\phi_d^\pi(O;\Theta_0)] \\
&=(\mathbb P_n-\mathbb P)\phi_d^\pi(O;\Theta_0)
+\mathbb P[\phi_d^\pi(O;\widehat\Theta)-\phi_d^\pi(O;\Theta_0)]
+(\mathbb P_n-\mathbb P)[\phi_d^\pi(O;\widehat\Theta)-\phi_d^\pi(O;\Theta_0)]\\
&=(\mathbb P_n-\mathbb P)\phi^\pi(O;\Theta_0)
+\mathbb P[\phi^\pi(O;\widehat\Theta)-\phi^\pi(O;\Theta_0)]
+(\mathbb P_n-\mathbb P)[\phi^\pi(O;\widehat\Theta)-\phi^\pi(O;\Theta_0)].
\end{aligned}
\]
Since
\(
\varphi^\pi(O;\Theta_0)=\phi^\pi(O;\Theta_0)-V(\pi),
\)
the first term equals
\(
(\mathbb P_n-\mathbb P)\varphi^\pi(O;\Theta_0).
\)
Thus,
\[
\widehat V_d(\pi)-V(\pi)=(\mathbb P_n-\mathbb P)\varphi^\pi(O;\Theta_0)+\mathcal D(\widehat\Theta)+E_n,
\]
where
\[
\mathcal D(\widehat\Theta)=\mathbb P[\phi^\pi(O;\widehat\Theta)-\phi^\pi(O;\Theta_0)]
\]
is the population drift, and
\[
E_n=(\mathbb P_n-\mathbb P)[\phi^\pi(O;\widehat\Theta)-\phi^\pi(O;\Theta_0)]
=(\mathbb P_n-\mathbb P)[\varphi^\pi(O;\widehat\Theta)-\varphi^\pi(O;\Theta_0)]
\]
is the empirical-process remainder.

The second term in \eqref{eq:decomposition} is the label-selection error. Since there are only two candidate
labelings,
\[
\left|\widehat V_{\widehat d}(\pi)-\widehat V_d(\pi)\right|\le|\widehat V_1(\pi)-\widehat V_0(\pi)|\mathds 1\{\widehat d\neq d\}.
\]
By Assumption~\ref{assump:regularity}, we have that
\(
\max_{a=0,1}|\widehat{V}_a(\pi)| = O_p(1),
\)
and consequently
\[
\widehat V_{\widehat d}(\pi)-\widehat V_d(\pi)=O_p(1)\mathds 1\{\widehat d\neq d\}.
\]

It remains to control the indicator \(\mathds 1\{\widehat d\neq d\}\). Recall
that
\[
\widehat d=\arg\max_{a\in\{0,1\}}\mathbb P_n[\widehat\theta_{\widetilde A}(S,a)].
\]
By Assumption~\ref{assump:regularity}, there exists \(c>0\) such that
\[
\theta_{\widetilde A}(s,d)-\theta_{\widetilde A}(s,1-d)\geq c
\quad\text{for all }s.
\]
Therefore,
\[
\mathbb P_n[\theta_{\widetilde A}(\cdot,d)-\theta_{\widetilde A}(\cdot,1-d)]\geq c.
\]
On the event \(\{\widehat d\neq d\}\), the definition of the empirical label selector gives
\[
\mathbb P_n[\widehat\theta_{\widetilde A}(\cdot,1-d)]
\geq
\mathbb P_n[\widehat\theta_{\widetilde A}(\cdot,d)].
\]
Thus, on \(\{\widehat d\neq d\}\),
\[
\begin{aligned}
c
&\leq\mathbb P_n[\theta_{\widetilde A}(\cdot,d)-\theta_{\widetilde A}(\cdot,1-d)] \\
&\leq\mathbb P_n[\theta_{\widetilde A}(\cdot,d)-\widehat\theta_{\widetilde A}(\cdot,d)]+\mathbb P_n
[\widehat\theta_{\widetilde A}(\cdot,1-d)-\theta_{\widetilde A}(\cdot,1-d)] \\
& \leq2\max_{a\in\{0,1\}}\left[\mathbb P_n[\widehat{\theta}_{\widetilde A}(\cdot,a) - \theta_{\widetilde A}(\cdot,a)]^2\right]^{1/2}.
\end{aligned}
\]
Consequently,
\[
\mathds 1\{\widehat d\neq d\}\leq\mathds 1\left\{\max_{a\in\{0,1\}}\left[\mathbb P_n[\widehat{\theta}_{\widetilde A}(\cdot,a) - \theta_{\widetilde A}(\cdot,a)]^2\right]^{1/2}\geq c/2\right\}.
\]
For any fixed \(\kappa>0\), the right-hand side is bounded by
\[
\left(\frac{2}{c}\right)^\kappa
\max_{a\in\{0,1\}}\left[\mathbb P_n[\widehat{\theta}_{\widetilde A}(\cdot,a) - \theta_{\widetilde A}(\cdot,a)]^2\right]^{\kappa/2} .
\]
By cross-fitting, for each validation fold the fitted function \(\widehat\theta_{\widetilde A}(\cdot,a)\) is constructed using observations outside that fold. Conditional on the corresponding training sample, the fitted function is fixed and independent of the validation observations. Hence, for each fold,
\[
\mathbb E\left[\mathbb P_{n,k}\left\{\widehat\theta_{\widetilde A}^{(-k)}(\cdot,a)-\theta_{\widetilde A}(\cdot,a)\right\}^2\mid \mathcal T_k\right]=\|\delta \widehat \theta^{(-k)}_{\widetilde A}(\cdot,a)\|_2^2,
\]
where \(\mathcal T_k\) denotes the training folds used to estimate the nuisance
functions evaluated at fold $k$. By conditional Markov's inequality,
\[
\mathbb P_{n,k}\left[\widehat\theta_{\widetilde A}^{(-k)}(\cdot,a)-\theta_{\widetilde A}(\cdot,a)\right]^2=O_p\left(\|\delta \widehat \theta^{(-k)}_{\widetilde A}(\cdot,a)\|_2^2\right).
\]
Applying this argument fold by fold and combining over the fixed number of folds gives, 
\[
\mathbb P_{n}\left[\widehat\theta_{\widetilde A}(\cdot,a)-\theta_{\widetilde A}(\cdot,a)\right]^2=O_p\left(\|\delta \widehat \theta_{\widetilde A}(\cdot,a)\|_2^2\right).
\]
Taking the \(\kappa/2\)-power and using that \(a\in\{0,1\}\) is finite, we obtain
\[
\max_{a\in\{0,1\}}\left[\mathbb P_n[\widehat\theta_{\widetilde A}(\cdot,a)-\theta_{\widetilde A}(\cdot,a)]^2\right]^{\kappa/2}
=O_p\left(\max_{a\in\{0,1\}}\|\delta \widehat \theta_{\widetilde A}(\cdot,a)\|_2^\kappa\right).
\]

Hence, with
\(
L_n=\max_{a\in\{0,1\}}\|\delta \widehat \theta_{\widetilde A}(\cdot,a)\|_2^\kappa,
\)
we have
\(
\mathds 1\{\widehat d\neq d\}=O_p(L_n).
\)
Therefore,
\[
\widehat V_{\widehat d}(\pi)-\widehat V_d(\pi)=O_p(L_n).
\]

Combining the two terms in \eqref{eq:decomposition}, we obtain
\[
\widehat V(\pi)-V(\pi)=(\mathbb P_n-\mathbb P)\varphi^\pi(O;\Theta_0)+\mathcal D(\widehat\Theta)+E_n+O_p(L_n).
\]
By Lemma \ref{lem:EP}, $E_n =o_p(n^{-1/2})$.
Therefore, it remains to bound the population drift 
\[
\mathcal D(\widehat\Theta)=\E\left[\phi^\pi(O;\widehat\Theta)-\phi^\pi(O;\Theta_0)\right].
\]

\paragraph*{Step 2: Decomposition of the population drift.}
We now decompose the population drift $\mathcal D(\widehat\Theta)$.
By the definition of \(\phi^\pi(O;\Theta)\), write
\[
\mathcal D(\widehat\Theta)=D_0(\widehat\Theta)+D_1(\widehat\Theta)+D_2(\widehat\Theta),
\]
where
\[
D_0(\widehat\Theta)=\mathbb E_{p_e}\{\widehat V^\pi(S)-V^\pi(S)\},
\]
\[
D_1(\widehat\Theta)=\frac{1}{1-\gamma}\sum_{a\in\{0,1\}}
\mathbb E\left[\widehat g_a'(S',\widetilde A,S)\widehat\omega(S,a)
\{R-\widehat\theta_R(S,a)\}-g_a'(S',\widetilde A,S)\omega(S,a)\{R-\theta_R(S,a)\}\right],
\]
and
\[
D_2(\widehat\Theta)=\frac{\gamma}{1-\gamma}\sum_{a\in\{0,1\}}\mathbb E\left[\widehat g_a(R,\widetilde A,S)\widehat\omega(S,a)
\{\widehat V^\pi(S')-\widehat{M}(S,a)\}-g_a(R,\widetilde A,S)\omega(S,a)\{V^\pi(S')- M(S,a)\}\right].
\]

For \(D_0(\widehat\Theta)\), note that
\[
V^\pi(s)=\sum_a \pi(a\mid s)\left\{\theta_R(s,a)+\gamma M(s,a)\right\}.
\]
Also recall that
\[
\widehat V^\pi(s)=\sum_a \pi(a\mid s)\widehat Q^\pi(s,a)=\sum_a \pi(a\mid s)\left\{\widehat\theta_R(s,a)+\gamma\widehat{M}(s,a)\right\},
\]
where
\[
\widehat{M}(s,a)=\gamma^{-1}\left\{\widehat Q^\pi(s,a)-\widehat\theta_R(s,a)\right\}.
\]

Therefore,
\[
D_0(\widehat\Theta)=\mathbb E_{p_e}\left[\sum_a\pi(a\mid S)\left\{\delta \widehat \theta_R(S,a)+\gamma\delta \widehat{M}(S,a)\right\}\right],
\]
where
\[
\delta \widehat{M}(s,a)=\widehat{M}(s,a)-M(s,a).
\]

The remaining two terms are obtained by comparing the hatted and true products inside the reward and transition corrections. We repeatedly use the elementary identity
\[
\begin{aligned}
\widehat X\widehat Y\widehat Z-XYZ&=(\delta \widehat X)YZ+X(\delta \widehat Y)Z+XY(\delta \widehat Z) \\
&\quad+(\delta \widehat X)(\delta \widehat Y)Z+(\delta \widehat X)Y(\delta \widehat Z)+X(\delta \widehat Y)(\delta \widehat Z)+(\delta \widehat X)(\delta \widehat Y)(\delta \widehat Z),
\end{aligned}
\]
where \(\widehat X=X+\delta \widehat X\), \(\widehat Y=Y+\delta \widehat Y\), and \(\widehat Z=Z+\delta \widehat Z\). The first line is linear in the perturbations, while the second line is second order and higher.

For $D_1$, fix \(a\) and set \(X_a=g'_a(S',\widetilde A,S)\), \(Y_a=\omega(S,a)\), and \(Z_a=R-\theta_R(S,a)\). Since \(\widehat Z_a=Z_a-\delta \widehat \theta_R(S,a)\), applying the preceding identity and summing over \(a\) gives
\[
D_1(\widehat\Theta)=L_1(\widehat\Theta)+\mathcal R_1(\widehat\Theta),
\]
where
\[
\begin{aligned}
L_1(\widehat\Theta)=\frac{1}{1-\gamma}\E\sum_a\Big[\delta\widehat g'_a\,\omega(S,a)\{R-\theta_R(S,a)\}+g'_a\delta \widehat \omega(S,a)\{R-\theta_R(S,a)\} -g'_a\omega(S,a)\delta \widehat \theta_R(S,a)
\Big],
\end{aligned}
\]
and
\[
\begin{aligned}
\mathcal R_1(\widehat\Theta)=\frac{1}{1-\gamma}\E\sum_a\Big[
&\delta \widehat g'_a\,\delta \widehat \omega(S,a)\{R-\theta_R(S,a)\}-\delta \widehat g'_a\,\omega(S,a)\delta \widehat \theta_R(S,a) \\
&-g'_a\delta \widehat \omega(S,a)\delta \widehat \theta_R(S,a)-\delta\widehat  g'_a\,\delta \widehat \omega(S,a)\delta \widehat \theta_R(S,a)\Big].
\end{aligned}
\]

For $D_2$, fix \(a\) and set \(X_a=g_a(R,\widetilde A,S)\), \(Y_a=\omega(S,a)\), and \(Z_a=V^\pi(S')-M(S,a)\). Since \(\widehat Z_a=Z_a+\delta\widehat  V^\pi(S')-\delta\widehat M(S,a)\), the same expansion gives
\[
D_2(\widehat\Theta)=L_2(\widehat\Theta)+\mathcal R_2(\widehat\Theta),
\]
where
\[
\begin{aligned}
L_2(\widehat\Theta)=\frac{\gamma}{1-\gamma}\E\sum_a\Big[&\delta \widehat g_a\,\omega(S,a)\{V^\pi(S')-M(S,a)\}+g_a\delta \widehat \omega(S,a)\{V^\pi(S')-M(S,a)\} \\&+g_a\omega(S,a)\{\delta \widehat V^\pi(S')-\delta \widehat M(S,a)\}\Big],
\end{aligned}
\]
and
\[
\begin{aligned}
\mathcal R_2(\widehat\Theta)=\frac{\gamma}{1-\gamma}\E\sum_a\Big[
&\delta \widehat g_a\,\delta \widehat \omega(S,a)\{V^\pi(S')-M(S,a)\} +\delta \widehat g_a\,\omega(S,a)\{\delta \widehat V^\pi(S')-\delta \widehat M(S,a)\} \\
&+g_a\delta \widehat \omega(S,a)\{\delta \widehat V^\pi(S')-\delta \widehat M(S,a)\} +\delta \widehat g_a\,\delta \widehat \omega(S,a)\{\delta \widehat V^\pi(S')-\delta \widehat M(S,a)\}\Big].
\end{aligned}
\]

Therefore, with
\[
L(\widehat\Theta)=D_0(\widehat\Theta)+L_1(\widehat\Theta)+L_2(\widehat\Theta),
\qquad
\mathcal R(\widehat\Theta)=\mathcal R_1(\widehat\Theta)+\mathcal R_2(\widehat\Theta),
\]
we have the exact decomposition
\[
\mathcal D(\widehat\Theta)=L(\widehat\Theta)+\mathcal R(\widehat\Theta).
\]
Here \(L(\widehat\Theta)\) is linear in \(\delta \widehat \theta_R,\delta \widehat M,\delta \widehat \omega,\delta \widehat g_a,\delta \widehat g'_a\), while \(\mathcal R(\widehat\Theta)\) contains only products of two or more of these perturbations.

\paragraph*{Step 3: Cancellation of the linear drift.}
From Step 2, the first-order part of the population drift is
\[
\begin{aligned}
L(\widehat\Theta)&=\E_{p_e}\{\delta \widehat V^\pi(S)\}\\
&\quad+\frac{1}{1-\gamma}\E\sum_a\Big[\delta \widehat g'_a\,\omega(S,a)\{R-\theta_R(S,a)\}+g'_a\delta \widehat \omega(S,a)\{R-\theta_R(S,a)\}-g'_a\omega(S,a)\delta \widehat \theta_R(S,a)\Big]\\
&\quad+\frac{\gamma}{1-\gamma}\E\sum_a\Big[\delta \widehat g_a\,\omega(S,a)\{V^\pi(S')-M(S,a)\}+g_a(R,\widetilde A,S)\delta \widehat \omega(S,a)\{V^\pi(S')-M(S,a)\}\\
&\qquad\qquad\qquad\qquad+g_a(R,\widetilde A,S)\omega(S,a)\{\delta \widehat V^\pi(S')-\delta \widehat M(S,a)\}\Big].
\end{aligned}
\]

We now show that 
\[
L(\widehat\Theta)= \frac{1}{1-\gamma}
\E\sum_a \delta \widehat g'_a(S',\widetilde A,S)\omega(S,a)\{R-\theta_R(S,a)\}+\frac{\gamma}{1-\gamma}\E\sum_a \delta \widehat g_a(R,\widetilde A,S)\omega(S,a)\{V^\pi(S')-M(S,a)\}.
\] 
First, for each \(a\in\{0,1\}\), 
\[
\begin{aligned}
&\mathbb{E}\left[g_a'(S',\widetilde A,S)\delta\widehat{\omega}(S,a)\{R-\theta_R(S,a)\}\right] \\
&\qquad=\mathbb{E}\left[\delta\widehat{\omega}(S,a)
\mathbb{E}\{g_a'(S',\widetilde A,S)\mid S,A\}\mathbb{E}\{R-\theta_R(S,a)\mid S,A\}\right]\\
&\qquad=\mathbb{E}\left[\delta\widehat{\omega}(S,a)
\mathds{1}(A=a)\mathbb{E}\{R-\theta_R(S,a)\mid S,A\}\right] = 0.
\end{aligned}
\]
Similarly,
\[
\mathbb{E}\left[
g_a(R,\widetilde A,S)\delta\widehat{\omega}(S,a)\{V^\pi(S')-M(S,a)\}\right]=0,
\]
using
\[
\mathbb{E}\{g_a(R,\widetilde A,S)\mid S,A\}=\mathbbm{1}(A=a)
\quad\text{ and }\quad
\mathbb{E}\{V^\pi(S')-M(S,a)\mid S,A=a\}=0.
\]

Summing over \(a\in\{0,1\}\) gives
\[
\mathbb{E}\left[\sum_ag_a'(S',\widetilde A,S)\delta\widehat{\omega}(S,a)\{R-\theta_R(S,a)\}\right]=0,
\]
and
\[
\mathbb{E}\left[\sum_ag_a(R,\widetilde A,S)\delta\widehat{\omega}(S,a)\{V^\pi(S')-M(S,a)\}\right]=0.
\]

Therefore the linear drift reduces to
\begin{align*}
L(\widehat\Theta)&=\E_{p_e}\{\delta \widehat V^\pi(S)\}-\frac{1}{1-\gamma}
\E\sum_a g'_a\omega(S,a)\delta \widehat \theta_R(S,a)
+\frac{\gamma}{1-\gamma}\E\sum_a g_a\omega(S,a)\{\delta \widehat V^\pi(S')-\delta \widehat M(S,a)\}\\
&+\frac{1}{1-\gamma}
\E\sum_a \delta \widehat g'_a\omega(S,a)\{R-\theta_R(S,a)\}+\frac{\gamma}{1-\gamma}\E\sum_a \delta \widehat g_a\omega(S,a)\{V^\pi(S')-M(S,a)\} \\
&= L_0 +\frac{1}{1-\gamma}
\E\sum_a \delta \widehat g'_a\omega(S,a)\{R-\theta_R(S,a)\}+\frac{\gamma}{1-\gamma}\E\sum_a \delta \widehat g_a\omega(S,a)\{V^\pi(S')-M(S,a)\}, 
\end{align*}
where we define
\[
L_0 = \E_{p_e}\{\delta V^\pi(S)\}-\frac{1}{1-\gamma}
\E\sum_a g'_a\omega(S,a)\delta \widehat \theta_R(S,a)
+\frac{\gamma}{1-\gamma}\E\sum_a g_a\omega(S,a)\{\delta \widehat V^\pi(S')-\delta \widehat M(S,a)\}.
\]

Next, we show $L_0 = 0$.
By Lemma \ref{lemma:indicator},
\[
\E\{g'_a(S',\widetilde A,S)\mid S,A\}=\mathds 1(A=a),
\qquad
\E\{g_a(R,\widetilde A,S)\mid S,A\}=\mathds 1(A=a).
\]
Thus, for any measurable function \(h(S,a)\),
\[
\E\sum_a g'_a(S',A,S)\omega(S,a)h(S,a)=\E\{\omega(S,A)h(S,A)\},
\]
and the same identity holds with \(g_a\) in place of \(g'_a\). Applying this to $L_0$ gives
\begin{align*}
L_0=\E_{p_e}\{\delta \widehat V^\pi(S)\}-\frac{1}{1-\gamma}\E\{\omega(S,A)\delta \widehat \theta_R(S,A)\}
+\frac{\gamma}{1-\gamma}\E\{\omega(S,A)\delta \widehat V^\pi(S')\}-\frac{\gamma}{1-\gamma}\E\{\omega(S,A)\delta \widehat M(S,A)\}
\end{align*}

Finally, by Lemma \ref{lemma:balancing}, for any measurable \(f\),
\[
\E_{p_e}\{f(S)\}=\frac{1}{1-\gamma}\E\left[\omega(S,A)\{f(S)-\gamma \E(f(S')\mid S,A)\}\right].
\]
Taking \(f=\delta \widehat V^\pi\), we obtain
\[
\E_{p_e}\{\delta \widehat V^\pi(S)\}=\frac{1}{1-\gamma}\E\left[\omega(S,A)\{\delta \widehat V^\pi(S)-\gamma \E(\delta \widehat V^\pi(S')\mid S,A)\}\right].
\]
Substituting this into the expression for \(L_0\), and using iterated expectation to cancel
\[
\E\{\omega(S,A)\E(\delta \widehat V^\pi(S')\mid S,A)\}=\E\{\omega(S,A)\delta \widehat V^\pi(S')\},
\]
yields
\[
L_0=\frac{1}{1-\gamma}\E\left[\omega(S,A)\{\delta \widehat V^\pi(S)-\delta \widehat \theta_R(S,A)-\gamma\delta \widehat M(S,A)\}\right].
\]

By the definition of \(V^\pi\),
\[
\delta \widehat V^\pi(S)=\sum_a \pi(a\mid S)\{\delta \widehat \theta_R(S,a)+\gamma\delta \widehat M(S,a)\}.
\]
Therefore
\[
\delta \widehat V^\pi(S)-\delta \widehat \theta_R(S,A)-\gamma\delta \widehat M(S,A)=\sum_a\pi(a\mid S)\{\delta \widehat \theta_R(S,a)+\gamma\delta \widehat M(S,a)\}-\{\delta \widehat \theta_R(S,A)+\gamma\delta \widehat M(S,A)\}.
\]
The balancing identity for \(\omega\) implies that, for any measurable \(h(S,a)\),
\[
\E\left[\omega(S,A)\left\{\sum_a\pi(a\mid S)h(S,a)-h(S,A)\right\}\right]=0.
\]
Applying this identity with $h(S,a)=\delta \widehat \theta_R(S,a)+\gamma\delta \widehat M(S,a)$ gives
\(
L_0=0.
\)
Therefore, 
\[
L(\widehat\Theta)= \frac{1}{1-\gamma}\E\sum_a \delta \widehat g'_a(S',\widetilde A,S)\omega(S,a)\{R-\theta_R(S,a)\}+\frac{\gamma}{1-\gamma}\E\sum_a \delta \widehat g_a(R,\widetilde A,S)\omega(S,a)\{V^\pi(S')-M(S,a)\}.
\] 

We bound this term by conditioning on \((S,A)\). For the reward part, the contribution from \(A=a\) vanishes because
\[
\mathbb E\{R-\theta_R(S,a)\mid S,A=a\}=0.
\]
Hence only the stratum \(A=1-a\) contributes. By the definitions of \(g'_a\),
\[
\mathbb E\{\delta \widehat g'_a(S',\widetilde A,S)\mid S,A=1-a\}=
\frac{\delta \widehat \theta_{\widetilde A}(S,1-a)\delta \widehat \theta_{S'}(S,1-a)}{\widehat\Delta_{\widetilde A,a}(S)\widehat\Delta_{S',a}(S)},
\]
and
\[
\mathbb E\{R-\theta_R(S,a)\mid S,A=1-a\}=\theta_R(S,1-a)-\theta_R(S,a).
\]
Therefore, under Assumption \ref{assump:regularity},
\[
\left|\mathbb E\sum_a\delta \widehat g'_a(O)\omega(S,a)\{R-\theta_R(S,a)\}\right|\lesssim\sum_a\mathbb E\left[\left|\delta \widehat \theta_{\widetilde A}(S,1-a)\delta \widehat \theta_{S'}(S,1-a)\right|\right].
\]
By Cauchy--Schwarz,
\[
\left|\mathbb E\sum_a\delta \widehat g'_a(O)\omega(S,a)\{R-\theta_R(S,a)\}\right|\lesssim\sum_a\|\delta \widehat \theta_{\widetilde A}(\cdot,1-a)\|\|\delta \widehat \theta_{S'}(\cdot,1-a)\|.
\]

Similarly, for the transition part, the contribution from \(A=a\) vanishes because
\[
\mathbb E\{V^\pi(S')-M(S,a)\mid S,A=a\}=0.
\]
Only the stratum \(A=1-a\) contributes. Moreover,
\[
\mathbb E\{\delta \widehat g_a(O)\mid S,A=1-a\}=\frac{\delta \widehat \theta_{\widetilde A}(S,1-a)\delta \widehat \theta_R(S,1-a)}{\widehat\Delta_{\widetilde A,a}(S)\widehat\Delta_{R,a}(S)},
\]
and
\[
\mathbb E\{V^\pi(S')-M(S,a)\mid S,A=1-a\}=M(S,1-a)-M(S,a).
\]
Thus, again by Assumption \ref{assump:regularity},
\[
\left|\mathbb E\sum_a\delta \widehat g_a(O)\omega(S,a)\{V^\pi(S')-M(S,a)\}
\right|\lesssim \sum_a \|\delta \widehat \theta_{\widetilde A}(\cdot,1-a)\| \|\delta \widehat \theta_R(\cdot,1-a)\|.
\]

Combining the two bounds gives
\[
|L(\widehat\Theta)|\lesssim\sum_a\|\delta \widehat \theta_{\widetilde A}(\cdot,1-a)\|
\|\delta \widehat \theta_{S'}(\cdot,1-a)\|+\sum_a\|\delta \widehat \theta_{\widetilde A}(\cdot,1-a)\|\|\delta \widehat \theta_R(\cdot,1-a)\|.
\]

\paragraph*{Step 4: Second- and higher-order bias term.}
In this step, we focus on $\mathcal R_1(\widehat\Theta)$ and $\mathcal R_2(\widehat\Theta)$ in
\[
\mathcal D(\widehat\Theta)= L(\widehat\Theta) + \mathcal R(\widehat\Theta)
=L(\widehat\Theta) +  \mathcal R_1(\widehat\Theta)+\mathcal R_2(\widehat\Theta),
\]
where all un-hatted quantities are evaluated at the true nuisance collection
\(\Theta_0\). Recall that
\[
\begin{aligned}
\mathcal R_1(\widehat\Theta)
=\frac{1}{1-\gamma}\E\sum_a\Big[
&(\delta \widehat g'_a)\delta \widehat \omega(S,a)\{R-\theta_R(S,a)\}-(\delta \widehat g'_a)\omega(S,a)\delta \widehat \theta_R(S,a) \\
&-g'_a\delta \widehat \omega(S,a)\delta \widehat \theta_R(S,a)-(\delta \widehat g'_a)\delta \widehat \omega(S,a)\delta \widehat \theta_R(S,a)\Big],
\end{aligned}
\]
and
\[
\begin{aligned}
\mathcal R_2(\widehat\Theta)=\frac{\gamma}{1-\gamma}\E\sum_a
\Big[
&(\delta \widehat g_a)\delta \widehat \omega(S,a)\{V^\pi(S')-M(S,a)\} +(\delta \widehat g_a)\omega(S,a)\{\delta \widehat V^\pi(S')-\delta \widehat M(S,a)\} \\
&+g_a\delta \widehat \omega(S,a)\{\delta \widehat V^\pi(S')-\delta \widehat M(S,a)\} +(\delta \widehat g_a)\delta \widehat \omega(S,a)\{\delta \widehat V^\pi(S')-\delta \widehat M(S,a)\}\Big].
\end{aligned}
\]

We first consider the leading residual term in \(\mathcal R_1(\widehat\Theta)\),
\[
T_{1,1}=\E\sum_a(\delta \widehat g'_a)\delta \widehat \omega(S,a)\{R-\theta_R(S,a)\}.
\]
Write \(R-\theta_R(S,a)=\{R-\theta_R(S,A)\}+\{\theta_R(S,A)-\theta_R(S,a)\}\). The first part vanishes after conditioning on \((S,A)\), by conditional independence and \(\E\{R-\theta_R(S,A)\mid S,A\}=0\). Hence
\[
T_{1,1}=\sum_a\E\left[\delta \widehat \omega(S,a)\E\{\delta \widehat g'_a\mid S,A\}\{\theta_R(S,A)-\theta_R(S,a)\}\right].
\]
Since \(g'_a\) satisfies \(\E(g'_a\mid S,A)=\mathds 1(A=a)\), we have
\[
\E\{\delta \widehat g'_a\mid S,A\}=\E\{\widehat g'_a\mid S,A\}-\mathds 1(A=a).
\]
The term involving \(\mathds 1(A=a)\) drops out because \(\mathds 1(A=a)\{\theta_R(S,A)-\theta_R(S,a)\}=0\). Therefore,
\[
T_{1,1}=\sum_a\E\left[\delta \widehat \omega(S,a)\E\{\widehat g'_a\mid S,A\}\{\theta_R(S,A)-\theta_R(S,a)\}\right].
\]
Now use the product form \(\widehat g'_a=\widehat\alpha_a\widehat\beta^{S'}_a\). By the conditional independence assumption, \(\widetilde A\perp S'\mid(S,A)\), and boundedness of the denominators in $\widehat\alpha_a$ and $\widehat\beta^{S'}_a$,
\[
|\E\{\widehat g'_a\mid S,A\}|\lesssim|\theta_{\widetilde A}(S,A)-\widehat\theta_{\widetilde A}(S,1-a)|
|\theta_{S'}(S,A)-\widehat\theta_{S'}(S,1-a)|.
\]
The factor \(\theta_R(S,A)-\theta_R(S,a)\) is zero when \(A=a\). Hence only the case \(A=1-a\) contributes, and on this event
\[
\theta_{\widetilde A}(S,A)-\widehat\theta_{\widetilde A}(S,1-a)=-\delta \widehat \theta_{\widetilde A}(S,1-a),\qquad\theta_{S'}(S,A)-\widehat\theta_{S'}(S,1-a)=-\delta \widehat \theta_{S'}(S,1-a).
\]
Thus,
\[
|T_{1,1}|\lesssim\sum_a\E\left[|\delta \widehat \omega(S,a)\delta \widehat \theta_{\widetilde A}(S,1-a)\delta \widehat \theta_{S'}(S,1-a)|\right].
\]
The remaining terms in \(\mathcal R_1(\widehat\Theta)\) can also be bounded by conditioning on \((S,A)\). Consider
\[
T_{1,2}=\E\sum_a(\delta \widehat g'_a)\omega(S,a)\delta \widehat \theta_R(S,a).
\]
Since \(\omega(S,a)\delta \widehat \theta_R(S,a)\) is measurable with respect to \((S,A)\),
\[
T_{1,2}=\E\sum_a\omega(S,a)\delta \widehat \theta_R(S,a)\E(\delta \widehat g'_a\mid S,A).
\]
We next bound \(\E(\delta \widehat g'_a\mid S,A)\). Recall that \(g'_a=\alpha_a\beta^{S'}_a\), where
\[
\alpha_a=\frac{\widetilde A-\theta_{\widetilde A}(S,1-a)}{\theta_{\widetilde A}(S,a)-\theta_{\widetilde A}(S,1-a)},\qquad\beta^{S'}_a=\frac{l(S')-\theta_{S'}(S,1-a)}{\theta_{S'}(S,a)-\theta_{S'}(S,1-a)}.
\]
The estimated quantities \(\widehat\alpha_a\) and \(\widehat\beta^{S'}_a\) are defined analogously by replacing the nuisance functions with their estimators. By conditional independence of
\(\widetilde A\) and \(S'\) given \((S,A)\),\[\E(\widehat g'_a\mid S,A=b)=\E(\widehat\alpha_a\mid S,A=b)\E(\widehat\beta^{S'}_a\mid S,A=b),\qquad b\in\{a,1-a\}.
\]
For \(b=a\), the two conditional means satisfy
\[
\E(\widehat\alpha_a\mid S,A=a)=
1-\frac{\delta \widehat \theta_{\widetilde A}(S,a)}{\widehat\theta_{\widetilde A}(S,a)-\widehat\theta_{\widetilde A}(S,1-a)},\quad 
\E(\widehat\beta^{S'}_a\mid S,A=a) =1-\frac{\delta \widehat \theta_{S'}(S,a)}{\widehat\theta_{S'}(S,a)-\widehat\theta_{S'}(S,1-a)}.
\]
Since \(\E(g'_a\mid S,A=a)=1\), the separation condition in Assumption \ref{assump:regularity} implies
\[
|\E(\delta \widehat g'_a\mid S,A=a)|\lesssim|\delta \widehat \theta_{\widetilde A}(S,a)|+|\delta \widehat \theta_{S'}(S,a)|+|\delta \widehat \theta_{\widetilde A}(S,a)\delta \widehat \theta_{S'}(S,a)|.
\]
For \(b=1-a\), we have \(\E(g'_a\mid S,A=1-a)=0\). Moreover,
\[
\E(\widehat\alpha_a\mid S,A=1-a)=-\frac{\delta \widehat \theta_{\widetilde A}(S,1-a)}
{\widehat\theta_{\widetilde A}(S,a)-\widehat\theta_{\widetilde A}(S,1-a)},\quad
\E(\widehat\beta^{S'}_a\mid S,A=1-a)=-\frac{\delta \widehat \theta_{S'}(S,1-a)}{\widehat\theta_{S'}(S,a)-\widehat\theta_{S'}(S,1-a)}.
\]
Hence,
\[
|\E(\delta \widehat g'_a\mid S,A=1-a)|\lesssim|\delta \widehat \theta_{\widetilde A}(S,1-a)\delta \widehat \theta_{S'}(S,1-a)|.
\]

Substituting these conditional bounds into the expression for \(T_{1,2}\), and using the boundedness of \(\omega\) and Assumption \ref{assump:regularity}, gives
\[
\begin{aligned}
|T_{1,2}|&\lesssim\sum_a\mathbb E\left[\left\{
|\delta\widehat\theta_{\widetilde A}(S,a)|+
|\delta\widehat\theta_{S'}(S,a)|+
|\delta\widehat\theta_{\widetilde A}(S,a)
  \delta\widehat\theta_{S'}(S,a)|\right\}
|\delta\widehat\theta_R(S,a)|\right]  \\
&\quad+\sum_a\mathbb E\left[|\delta\widehat\theta_{\widetilde A}(S,1-a)
  \delta\widehat\theta_{S'}(S,1-a)
  \delta\widehat\theta_R(S,a)|\right] \\
&\lesssim\sum_a\mathbb E  \left[\left\{
|\delta\widehat\theta_{\widetilde A}(S,a)|+
|\delta\widehat\theta_{S'}(S,a)|
\right\}
|\delta\widehat\theta_R(S,a)| +
|\delta\widehat\theta_{\widetilde A}(S,1-a)
  \delta\widehat\theta_{S'}(S,1-a)
  \delta\widehat\theta_R(S,a)|\right]
\end{aligned}
\]

For the third term in \(\mathcal R_1(\widehat\Theta)\), conditioning on \((S,A)\) gives
\[
\begin{aligned}
T_{1,3} = \E\sum_a g'_a(O)\delta \widehat \omega(S,a)\delta \widehat \theta_R(S,a)
&=\E\sum_a\delta \widehat \omega(S,a)\delta \widehat \theta_R(S,a)\E(g'_a\mid S,A)  \\
&=\E\sum_a\delta \widehat \omega(S,a)\delta \widehat \theta_R(S,a)\mathds 1(A=a).
\end{aligned}
\]
Therefore,
\[
|T_{1,3}|\le\sum_a \E\left[|\delta \widehat \omega(\cdot,a)\delta \widehat \theta_R(\cdot,a)|\right].
\]

It remains to bound the term involving all three perturbations,
\[
T_{1,4}
=
\mathbb E\sum_a
\delta\widehat g_a'\,
\delta\widehat\omega(S,a)\,
\delta\widehat\theta_R(S,a).
\]
Since \(\delta\widehat\omega(S,a)\delta\widehat\theta_R(S,a)\) is measurable with respect to
\((S,A)\), conditioning on \((S,A)\) gives
\[
T_{1,4}
=
\mathbb E\sum_a
\delta\widehat\omega(S,a)\delta\widehat\theta_R(S,a)
\mathbb E\{\delta\widehat g_a'\mid S,A\}.
\]
Using the conditional bounds for
\(\mathbb E\{\delta\widehat g_a'\mid S,A\}\) derived above, together with the separation condition in Assumption \ref{assump:regularity}, we obtain
\[
\begin{aligned}
|T_{1,4}|
&\lesssim\sum_a\mathbb E\Big[
|\delta\widehat\omega(S,a) \delta\widehat\theta_R(S,a)|
\Big\{
|\delta\widehat\theta_{\widetilde A}(S,a)|
+|\delta\widehat\theta_{S'}(S,a)|
+|\delta\widehat\theta_{\widetilde A}(S,a)
  \delta\widehat\theta_{S'}(S,a)|
\Big\}\Big] \\
&\quad+\sum_a\mathbb E\Big[
|\delta\widehat\omega(S,a) \delta\widehat\theta_R(S,a) \delta\widehat\theta_{\widetilde A}(S,1-a)
  \delta\widehat\theta_{S'}(S,1-a)|
\Big]\\
&\lesssim\sum_a\mathbb E\Big[
|\delta\widehat\omega(S,a) \delta\widehat\theta_R(S,a)|
\Big\{
|\delta\widehat\theta_{\widetilde A}(S,a)|
+|\delta\widehat\theta_{S'}(S,a)|
\Big\}
\Big] \\
&\quad+\sum_a\mathbb E\Big[
|\delta\widehat\omega(S,a) \delta\widehat\theta_R(S,a) \delta\widehat\theta_{\widetilde A}(S,1-a)
  \delta\widehat\theta_{S'}(S,1-a)|
\Big]
\end{aligned}
\]
Combining the bounds for $T_{1,1},\ldots, T_{1,4}$, the bound for \(\mathcal R_1(\widehat\Theta)\) is therefore
\[
\begin{aligned}
&|\mathcal R_1(\widehat\Theta)| \\
\lesssim &\sum_a
\mathbb E\!\left[
|\delta\widehat\omega(S,a) \delta\widehat\theta_{\widetilde A}(S,1-a) \delta\widehat\theta_{S'}(S,1-a)|
\right] 
+\sum_a\mathbb E\!\left[\left\{
|\delta\widehat\theta_{\widetilde A}(S,a)|
+|\delta\widehat\theta_{S'}(S,a)|\right\}
|\delta\widehat\theta_R(S,a)|
\right] \\
&\quad+\sum_a\mathbb E\!\left[
|\delta\widehat\theta_{\widetilde A}(S,1-a)
\delta\widehat\theta_{S'}(S,1-a) \delta\widehat\theta_R(S,a)|\right] 
+\sum_a\mathbb E\!\left[
|\delta\widehat\omega(S,a) \delta\widehat\theta_R(S,a)|
\right] \\
\lesssim &\sum_a\mathbb E\!\left[\left\{
|\delta\widehat\theta_{\widetilde A}(S,a)|
+|\delta\widehat\theta_{S'}(S,a) + |\delta\widehat\omega(S,a)|\right\}
|\delta\widehat\theta_R(S,a)|
\right] \\
&\quad + \sum_a
\mathbb E\!\left[
|\delta\widehat\omega(S,a) \delta\widehat\theta_{\widetilde A}(S,1-a) \delta\widehat\theta_{S'}(S,1-a)|
\right] 
+\sum_a\mathbb E\!\left[
|\delta\widehat\theta_{\widetilde A}(S,1-a)
\delta\widehat\theta_{S'}(S,1-a) \delta\widehat\theta_R(S,a)|\right] 
\end{aligned}
\]

We next consider \(\mathcal R_2(\widehat\Theta)\). The leading residual term is
\[
T_{2,1}=\E\sum_a(\delta \widehat g_a)\delta \widehat \omega(S,a)\{V^\pi(S')-M(S,a)\}.
\]
The same argument applies, now using \(g_a=\alpha_a\beta^R_a\). Write
\[
V^\pi(S')-M(S,a)=\{V^\pi(S')-M(S,A)\}+\{M(S,A)-M(S,a)\}.
\]
The first part has conditional mean zero given \((S,A)\). Since \(\E(g_a\mid S,A)=\mathds 1(A=a)\), the term involving \(\mathds 1(A=a)\) again vanishes after multiplication by \(M(S,A)-M(S,a)\). Therefore the contribution is driven by \(\E(\widehat g_a\mid S,A)\{M(S,A)-M(S,a)\}\). Using the product form \(\widehat g_a=\widehat\alpha_a\widehat\beta^R_a\), conditional independence, and bounded denominators gives
\[
|\E\{\widehat g_a\mid S,A\}|\lesssim|\theta_{\widetilde A}(S,A)-\widehat\theta_{\widetilde A}(S,1-a)||\theta_R(S,A)-\widehat\theta_R(S,1-a)|.
\]
The factor \(M(S,A)-M(S,a)\) is zero when \(A=a\), so only \(A=1-a\) contributes. Thus,
\[
|T_{2,1}|\lesssim\sum_a\E\left[|\delta \widehat \omega(S,a)\delta \widehat \theta_{\widetilde A}(S,1-a)\delta \widehat \theta_R(S,1-a)|\right].
\]
For the remaining terms in \(\mathcal R_2(\widehat\Theta)\), we condition on \((S,A)\). Define 
\[
e_M(S,a)=\E\{\delta \widehat V^\pi(S')\mid S,A=a\}-\delta \widehat M(S,a).
\]
Consider first
\[
T_{2,2}=\E\sum_a(\delta \widehat g_a)\omega(S,a)\{\delta \widehat V^\pi(S')-\delta \widehat M(S,a)\}.
\]
By conditioning on \((S,A)\), and using the conditional independence of \((\widetilde A,R)\) and \(S'\) given \((S,A)\), we obtain
\[
T_{2,2}=\E\sum_a\omega(S,a)\E(\delta \widehat g_a\mid S,A)\left[\E\{\delta \widehat V^\pi(S')\mid S,A\}-\delta \widehat M(S,a)\right].
\]
On the event \(A=a\), the last bracket equals \(e_M(S,a)\). Moreover, \(g_a=\alpha_a\beta_a^R\) gives
\[
|\E(\delta \widehat g_a\mid S,A=a)|\lesssim |\delta \widehat \theta_{\widetilde A}(S,a)|+|\delta \widehat \theta_R(S,a)\}|+|\delta \widehat \theta_{\widetilde A}(S,a)\delta \widehat \theta_R(S,a)|.
\]
On the event \(A=1-a\), the last bracket is not \(e_M(S,a)\), but $\delta \widehat g_a$ gives the sharper bound
\[
|\E(\delta \widehat g_a\mid S,A=1-a)|\lesssim |\delta \widehat \theta_{\widetilde A}(S,1-a)\delta \widehat \theta_R(S,1-a)|.
\]
Therefore, under boundedness of the remaining factors,
\[
\begin{aligned}
|T_{2,2}|\lesssim
\sum_a\mathbb E\!\left[\left\{|\delta\widehat\theta_{\widetilde A}(S,a)|
+|\delta\widehat\theta_R(S,a)|\right\}|e_M(S,a)|\right] 
+\sum_a\mathbb E\!\left[|\delta\widehat\theta_{\widetilde A}(S,1-a)\delta\widehat\theta_R(S,1-a)|\right].
\end{aligned}
\]

Next consider
\[
T_{2,3}=\E\sum_ag_a(O)\delta \widehat \omega(S,a)\{\delta \widehat V^\pi(S')-\delta \widehat M(S,a)\}.
\]
Conditioning on \((S,A)\) and using \(\E(g_a\mid S,A)=\mathds 1(A=a)\), we get
\[
T_{2,3}=\E\sum_a\delta \widehat \omega(S,a)\mathds 1(A=a) e_M(S,a).
\]
Hence
\[
|T_{2,3}|\le\sum_a\E\left\{|\delta \widehat \omega(\cdot,a) e_M(\cdot,a)|\right\}.
\]

Finally, consider the term involving all three perturbations,
\[
T_{2,4}=\E\sum_a(\delta \widehat g_a)\delta \widehat \omega(S,a)\{\delta \widehat V^\pi(S')-\delta \widehat M(S,a)\}.
\]
Using the same conditioning argument as above, we obtain
\[
T_{2,4}=\mathbb E\sum_a\delta\widehat\omega(S,a)
\mathbb E\!\left\{\delta\widehat g_a \mid S,A\right\}
\left[\mathbb E\!\left\{\delta \widehat V^\pi(S') \mid S,A\right\}-\delta \widehat M(S,a)
\right].
\]
On the event \(A=a\), the term in the last bracket equals \(e_M(S,a)\). Moreover,
\[
\mathbb E\!\left\{\delta\widehat g_a \mid S,A=a\right\}
\lesssim |\delta\widehat\theta_{\widetilde A}(S,a)|
+|\delta\widehat\theta_R(S,a)|
+|\delta\widehat\theta_{\widetilde A}(S,a)\delta\widehat\theta_R(S,a)|.
\]
Therefore, the \(A=a\) contribution is bounded by
\[
\begin{aligned}
&\sum_a\mathbb E\!\left[|\delta\widehat\omega(S,a)|
\left\{|\delta\widehat\theta_{\widetilde A}(S,a)|+|\delta\widehat\theta_R(S,a)|
+|\delta\widehat\theta_{\widetilde A}(S,a)\delta\widehat\theta_R(S,a)|\right\}|e_M(S,a)|\right].
\end{aligned}
\]
On the event \(A=1-a\), the last bracket is not \(e_M(S,a)\) (but is bounded). However,
the structure of \(g_a\) gives
\[
\mathbb E\!\left\{\delta\widehat g_a \mid S,A=1-a\right\}
\lesssim |\delta\widehat\theta_{\widetilde A}(S,1-a)\delta\widehat\theta_R(S,1-a)|.
\]
Hence the \(A=1-a\) contribution is bounded by
\[
\begin{aligned}
\sum_a\mathbb E\!\left[|\delta\widehat\omega(S,a)
\delta\widehat\theta_{\widetilde A}(S,1-a)\delta\widehat\theta_R(S,1-a)|
\right].
\end{aligned}
\]
Combining the two cases gives
\[
\begin{aligned}
|T_{2,4}|\lesssim
&\sum_a\mathbb E\!\left[
|\delta\widehat\omega(S,a)|
\left\{|\delta\widehat\theta_{\widetilde A}(S,a)|
+|\delta\widehat\theta_R(S,a)|
+|\delta\widehat\theta_{\widetilde A}(S,a)\delta\widehat\theta_R(S,a)|
\right\}|e_M(S,a)|
\right] \\
&\quad+\sum_a\mathbb E\!\left[
|\delta\widehat\omega(S,a)
\delta\widehat\theta_{\widetilde A}(S,1-a)\delta\widehat\theta_R(S,1-a)|
\right]\\
\lesssim
&\sum_a\mathbb E\!\left[
|\delta\widehat\omega(S,a)|\left\{|\delta\widehat\theta_{\widetilde A}(S,a)|
+|\delta\widehat\theta_R(S,a)|\right\}|e_M(S,a)|\right] \\
&\quad+\sum_a\mathbb E\!\left[|\delta\widehat\omega(S,a)
\delta\widehat\theta_{\widetilde A}(S,1-a)\delta\widehat\theta_R(S,1-a)|\right],
\end{aligned}
\]
which are higher order than the bound for $|T_{2,2}|$.

Consequently, \(\mathcal R_2(\widehat\Theta)\) is bounded by
\[
\begin{aligned}
|\mathcal R_2|\lesssim
&\sum_a\mathbb E\!\left[|\delta\widehat\omega(S,a) \delta\widehat\theta_{\widetilde A}(S,1-a) \delta\widehat\theta_R(S,1-a)|\right] 
+\sum_a\mathbb E\!\left[\left\{|\delta\widehat\theta_{\widetilde A}(S,a)|
+|\delta\widehat\theta_R(S,a)|\right\}|e_M(S,a)|\right] \\
&+\sum_a\mathbb E\!\left[
|\delta\widehat\theta_{\widetilde A}(S,1-a)  \delta\widehat\theta_R(S,1-a)|\right] 
+\sum_a\mathbb E\!\left[|\delta\widehat\omega(S,a) e_M(S,a)|\right] \\
\lesssim & \sum_a\mathbb E\!\left[\left\{|\delta\widehat\theta_{\widetilde A}(S,a)|
+|\delta\widehat\theta_R(S,a)| + |\delta\widehat\omega(S,a)|\right\}|e_M(S,a)|\right] \\
& + \sum_a\mathbb E\!\left[|\delta\widehat\omega(S,a) \delta\widehat\theta_{\widetilde A}(S,1-a) \delta\widehat\theta_R(S,1-a)|\right] +\sum_a\mathbb E\!\left[
|\delta\widehat\theta_{\widetilde A}(S,a)  \delta\widehat\theta_R(S,a)|\right] 
\end{aligned}
\]

Combining the bounds for $\mathcal R_1$ and $\mathcal R_2$, and noting that
\(
\sum_a \E\!\left[
|\delta\widehat\theta_{\widetilde A}(S,a)  \delta\widehat\theta_R(S,a)|\right] 
\)
is already contained in
\[
\sum_a\mathbb E\!\left[\left\{
|\delta\widehat\theta_{\widetilde A}(S,a)| +|\delta\widehat\theta_{S'}(S,a)| + |\delta\widehat\omega(S,a)|\right\} |\delta\widehat\theta_R(S,a)|
\right] 
\]
we obtain
\[
\begin{aligned}
|\mathcal R(\widehat\Theta)|
\lesssim&\sum_a\mathbb E\!\left[
\left\{|\delta\widehat\theta_{\widetilde A}(S,a)|
+|\delta\widehat\theta_{S'}(S,a)|
+|\delta\widehat\omega(S,a)|\right\}
|\delta\widehat\theta_R(S,a)|\right] \\
&+\sum_a\mathbb E\!\left[\left\{|\delta\widehat\theta_{\widetilde A}(S,a)|
+|\delta\widehat\theta_R(S,a)|
+|\delta\widehat\omega(S,a)|
\right\}|e_M(S,a)|\right] \\
&+\sum_a\mathbb E\!\left[
|\delta\widehat\theta_{\widetilde A}(S,1-a)|
|\delta\widehat\theta_{S'}(S,1-a)|
|\delta\widehat\theta_R(S,a)|
\right] \\
&+\sum_a\mathbb E\!\left[
|\delta\widehat\omega(S,a)|
|\delta\widehat\theta_{\widetilde A}(S,1-a)|
|\delta\widehat\theta_{S'}(S,1-a)|
\right] \\
&+\sum_a\mathbb E\!\left[
|\delta\widehat\omega(S,a)|
|\delta\widehat\theta_{\widetilde A}(S,1-a)|
|\delta\widehat\theta_R(S,1-a)|
\right].
\end{aligned}
\]

Combining the preceding bound for \(L(\widehat\Theta)\) with the bound for
\(\mathcal R(\widehat\Theta)\), we obtain
\begin{align*}
|\mathcal D(\widehat\Theta)|
\le |L(\widehat\Theta)|+|\mathcal R(\widehat\Theta)|
\lesssim &\sum_a\mathbb E\!\left[
\left\{|\delta\widehat\theta_{\widetilde A}(S,a)|
+|\delta\widehat\theta_{S'}(S,a)|
+|\delta\widehat\omega(S,a)|\right\}
|\delta\widehat\theta_R(S,a)|\right] \\
&+\sum_a\mathbb E\!\left[\left\{|\delta\widehat\theta_{\widetilde A}(S,a)|
+|\delta\widehat\theta_R(S,a)|
+|\delta\widehat\omega(S,a)|
\right\}|e_M(S,a)|\right] \\
&+\sum_a\mathbb E\!\left[
|\delta\widehat\theta_{\widetilde A}(S,a)|
|\delta\widehat\theta_{S'}(S,a)|
\right] \\
&+\sum_a\mathbb E\!\left[
|\delta\widehat\theta_{\widetilde A}(S,1-a)|
|\delta\widehat\theta_{S'}(S,1-a)|
|\delta\widehat\theta_R(S,a)|
\right] \\
&+\sum_a\mathbb E\!\left[
|\delta\widehat\omega(S,a)|
|\delta\widehat\theta_{\widetilde A}(S,1-a)|
|\delta\widehat\theta_{S'}(S,1-a)|
\right] \\
&+\sum_a\mathbb E\!\left[
|\delta\widehat\omega(S,a)|
|\delta\widehat\theta_{\widetilde A}(S,1-a)|
|\delta\widehat\theta_R(S,1-a)|
\right] 
\end{align*}
Therefore, 
\begin{align*}
&I_n = \sum_{a\in\{0,1\}}\E\Big[
\left\{|\delta\widehat\theta_{\widetilde A}(S,a)|
+|\delta\widehat\theta_{S'}(S,a)|
+|\delta\widehat\omega(S,a)|\right\}
|\delta\widehat\theta_R(S,a)| \\
&+ \left\{|\delta\widehat\theta_{\widetilde A}(S,a)|
+|\delta\widehat\theta_R(S,a)|
+|\delta\widehat\omega(S,a)|
\right\}|e_M(S,a)| \\
&+ |\delta\widehat\theta_{\widetilde A}(S,a)|
|\delta\widehat\theta_{S'}(S,a)| + 
|\delta\widehat\theta_{\widetilde A}(S,1-a)|
|\delta\widehat\theta_{S'}(S,1-a)|
|\delta\widehat\theta_R(S,a)| \\
&+
|\delta\widehat\omega(S,a)|
|\delta\widehat\theta_{\widetilde A}(S,1-a)|
|\delta\widehat\theta_{S'}(S,1-a)|
+
|\delta\widehat\omega(S,a)|
|\delta\widehat\theta_{\widetilde A}(S,1-a)|
|\delta\widehat\theta_R(S,1-a)|
\Big] 
\end{align*}

By Cauchy-Schwarz, the first three terms are controlled by
\begin{align*}
& \sum_a \big[\|\delta \widehat \theta_{\widetilde A}(\cdot,a)\|
+\|\delta \widehat \theta_{S'}(\cdot,a)\|
+\|\delta \widehat \omega(\cdot,a)\|
\big]\|\delta \widehat \theta_R(\cdot,a)\|;\\
& \sum_a \big[\|\delta \widehat \theta_{\widetilde A}(\cdot,a)\|+
\|\delta \widehat \theta_R(\cdot,a)\|+
\|\delta \widehat \omega(\cdot,a)\|\big]\|e_M(\cdot,a)\|; \\
\text{and }& \sum_a \|\delta \widehat \theta_{\widetilde A}(\cdot,a)\| 
\|\delta \widehat \theta_{S'}(\cdot,a)\|.
\end{align*}
By Holder's inequality, the last three terms are controlled by
\begin{align*}
&\sum_a
\mathbb E\!\left[
|\delta\widehat\theta_{\widetilde A}(S,1-a)|
|\delta\widehat\theta_{S'}(S,1-a)|
|\delta\widehat\theta_R(S,a)|
\right] \le
\sum_a
\|\delta\widehat\theta_{\widetilde A}(\cdot,1-a)\|_{\infty}
\|\delta\widehat\theta_{S'}(\cdot,1-a)\|
\|\delta\widehat\theta_R(\cdot,a)\|, \\
&\sum_a
\mathbb E\!\left[
|\delta\widehat\omega(S,a)|
|\delta\widehat\theta_{\widetilde A}(S,1-a)|
|\delta\widehat\theta_{S'}(S,1-a)|
\right] \le
\sum_a
\|\delta\widehat\theta_{\widetilde A}(\cdot,1-a)\|_{\infty}
\|\delta\widehat\omega(\cdot,a)\|
\|\delta\widehat\theta_{S'}(\cdot,1-a)\|, \\
&\sum_a
\mathbb E\!\left[
|\delta\widehat\omega(S,a)|
|\delta\widehat\theta_{\widetilde A}(S,1-a)|
|\delta\widehat\theta_R(S,1-a)|
\right] \le
\sum_a
\|\delta\widehat\theta_{\widetilde A}(\cdot,1-a)\|_{\infty}
\|\delta\widehat\omega(\cdot,a)\|
\|\delta\widehat\theta_R(\cdot,1-a)\|.
\end{align*}

\subsection{Proof of Corollary \ref{cor: multiple robust}}
\begin{proof}
From the decomposition established in the proof of Theorem~\ref{thm:error},
\[
\begin{aligned}
\widehat V(\pi)-V(\pi)=\frac{1}{NT}\sum_{i=1}^{N}\sum_{t=1}^{T}\varphi^\pi(O_{i,t};\Theta_0) 
+\bigl\{\widehat V_{\widehat d}(\pi)-\widehat V_d(\pi)\bigr\}+O_p(I_n)+o_p\bigl((NT)^{-1/2}\bigr).
\end{aligned}
\]

The influence-function term is \(o_p(1)\) by the law of large
numbers. Moreover,
\[
\left|\widehat V_{\widehat d}(\pi)-\widehat V_d(\pi)\right|\le\left|\widehat V_1(\pi)-\widehat V_0(\pi)\right|\mathbbm{1}\{\widehat d\neq d\}.
\]
By Assumption~\ref{assump:regularity}, the candidate value estimators are \(O_p(1)\), and consistency of the label-alignment step implies
\[
\mathbbm{1}\{\widehat d\neq d\}=o_p(1).
\]
Hence,
\[
\widehat V_{\widehat d}(\pi)-\widehat V_d(\pi)=o_p(1).
\]
It therefore remains to show that \(I_n=o_p(1)\) under each of the four
stated nuisance-consistency conditions.

Assumption \ref{assump:regularity} ensures that the nuisance estimators are uniformly bounded and that the relevant separation conditions hold for both the true and estimated nuisance functions. Hence the constants appearing in the error bound are well controlled, and any product term in $I_n$ is $o_p(1)$ whenever at least one of its factors is $o_p(1)$.

Under condition (i), $\|\widehat\theta_{\widetilde A}(\cdot,a) -\theta_{\widetilde A}(\cdot,a)\|=o_p(1)$, $\|\widehat\theta_R(\cdot,a)-\theta_R(\cdot,a)\|=o_p(1)$, and $\|e_M(\cdot,a)\|=o_p(1)$ for each $a\in\{0,1\}$. Inspecting the definition of $I_n$, every term contains at least one of these three $o_p(1)$ factors. Hence $I_n=o_p(1)$.

Under condition (ii), $\|\widehat\theta_{S'}(\cdot,a) -\theta_{S'}(\cdot,a)\|=o_p(1)$,
$\|\widehat\theta_R(\cdot,a)-\theta_R(\cdot,a)\|=o_p(1)$, and $\|e_M(\cdot,a)\|=o_p(1)$ for each $a\in\{0,1\}$. Again, each term in $I_n$ contains at least one of these $o_p(1)$ factors, and therefore $I_n=o_p(1)$.

Under condition (iii), $\|\widehat\omega(\cdot,a)-\omega(\cdot,a)\|=o_p(1)$, $\|\widehat\theta_R(\cdot,a)-\theta_R(\cdot,a)\|=o_p(1)$, and $\|\widehat\theta_{\widetilde A}(\cdot,a) -\theta_{\widetilde A}(\cdot,a)\|=o_p(1)$ for each $a\in\{0,1\}$.
The first group of terms in $I_n$ contains the factor $|\delta\widehat\theta_R(S,a)|$, the second group contains at least one of $|\delta\widehat\theta_{\widetilde A}(S,a)|$, $|\delta\widehat\theta_R(S,a)|$, or $|\delta\widehat\omega(S,a)|$, and the remaining higher-order product terms also contain at least one of these factors. Thus $I_n=o_p(1)$.

Under condition (iv), $\|\widehat\omega(\cdot,a)-\omega(\cdot,a)\| =o_p(1)$, $\|\widehat\theta_{\widetilde A}(\cdot,a) -\theta_{\widetilde A}(\cdot,a)\|=o_p(1)$,
$\|\widehat\theta_{S'}(\cdot,a)-\theta_{S'}(\cdot,a)\|=o_p(1)$, and $\|e_M(\cdot,a)\|=o_p(1)$ for each $a\in\{0,1\}$. Each term in $I_n$ then contains at least one of these $o_p(1)$ factors, so $I_n=o_p(1)$.

Therefore, under any one of the four conditions, $L_n=o_p(1)$, $I_n=o_p(1)$, and $E_n=o_p(1)$. Combining these facts with Theorem \ref{thm:error} gives
\[
\widehat V(\pi)-V(\pi)=o_p(1).
\]
This proves the corollary.
\end{proof}

\subsection{Proof of Theorem \ref{thm:asymptotic_normality}}

\begin{proof}
By Theorem~\ref{thm:error}, we have
\[
\widehat V(\pi)-V(\pi)=(\mathbb P_n-\mathbb P)\varphi^\pi(O;\Theta_0)+O_p(L_n)+O_p(I_n)+E_n .
\] By Lemma \ref{lem:EP}, we have $E_n=o_p(n^{-1/2})$. It remains to show that
\(
L_n=o_p(n^{-1/2}),\ I_n=o_p(n^{-1/2}).
\) 

First, by the additional condition in Theorem~\ref{thm:asymptotic_normality},
\(
\max_{a\in\{0,1\}}\|\delta \widehat \theta_{\widetilde A}(\cdot,a)\|^\kappa=o_p(n^{-1/2}).
\)
Since
\(
L_n=\max_{a\in\{0,1\}}\|\delta \widehat \theta_{\widetilde A}(\cdot,a)\|_\infty^\kappa,
\)
we immediately have
\(
L_n=o_p(n^{-1/2}).
\)

Next, we show that \(I_n=o_p(n^{-1/2})\). Let
\(
\alpha_{\min}=\min\{\alpha_\omega,\alpha_{\widetilde A},\alpha_{S'},\alpha_R,\alpha_M\}.
\)
By assumption, \(\alpha_{\min}>1/4\), and for each \(a\in\{0,1\}\),
\[
\|\delta\widehat\omega(\cdot,a)\|=O_p(n^{-\alpha_\omega}),
\qquad
\|\delta\widehat\theta_{\widetilde A}(\cdot,a)\|=O_p(n^{-\alpha_{\widetilde A}}),
\]
\[
\|\delta\widehat\theta_R(\cdot,a)\|=O_p(n^{-\alpha_R}),
\qquad
\|\delta\widehat\theta_{S'}(\cdot,a)\|=O_p(n^{-\alpha_{S'}}),
\]
and
\[
\|e_M(\cdot,a)\|=O_p(n^{-\alpha_M}).
\]
Assume also that the nuisance estimation errors are uniformly bounded in
probability, that is,
\[
\max_{a\in\{0,1\}}\left\{
\|\delta\widehat\omega(\cdot,a)\|_\infty,
\|\delta\widehat\theta_{\widetilde A}(\cdot,a)\|_\infty,
\|\delta\widehat\theta_{S'}(\cdot,a)\|_\infty,
\|\delta\widehat\theta_R(\cdot,a)\|_\infty
\right\}=O_p(1).
\]

From the preceding bound, \(I_n\) is bounded by the sum of the following terms.
First, by Cauchy--Schwarz,
\[
\begin{aligned}
&\sum_{a\in\{0,1\}}\E\!\left[
\left\{|\delta\widehat\theta_{\widetilde A}(S,a)|
+|\delta\widehat\theta_{S'}(S,a)|
+|\delta\widehat\omega(S,a)|\right\} |\delta\widehat\theta_R(S,a)|\right] \\
&\qquad\leq\sum_{a\in\{0,1\}}
\left\{\|\delta\widehat\theta_{\widetilde A}(\cdot,a)\|
+\|\delta\widehat\theta_{S'}(\cdot,a)\|
+\|\delta\widehat\omega(\cdot,a)\|\right\} \|\delta\widehat\theta_R(\cdot,a)\| \\
&\qquad=O_p\!\left(
n^{-\alpha_R-\min\{\alpha_{\widetilde A},\alpha_{S'},\alpha_\omega\}}
\right)=o_p(n^{-1/2}).
\end{aligned}
\]

Second, again by Cauchy--Schwarz,
\[
\begin{aligned}
&\sum_{a\in\{0,1\}}\mathbb E\!\left[\left\{
|\delta\widehat\theta_{\widetilde A}(S,a)|
+|\delta\widehat\theta_R(S,a)|
+|\delta\widehat\omega(S,a)|\right\}|e_M(S,a)| \right] \\
&\qquad\leq
\sum_{a\in\{0,1\}}
\left\{\|\delta\widehat\theta_{\widetilde A}(\cdot,a)\|
+\|\delta\widehat\theta_R(\cdot,a)\|
+\|\delta\widehat\omega(\cdot,a)\|\right\} \|e_M(\cdot,a)\| \\
&\qquad=O_p\!\left(n^{-\alpha_M-\min\{\alpha_{\widetilde A},\alpha_R,\alpha_\omega\}}
\right)=o_p(n^{-1/2}).
\end{aligned}
\]

Third,
\[
\begin{aligned}
\sum_{a\in\{0,1\}} \E\left[
|\delta\widehat\theta_{\widetilde A}(S,a)|
|\delta\widehat\theta_{S'}(S,a)| \right]
&\leq \sum_{a\in\{0,1\}}
\|\delta\widehat\theta_{\widetilde A}(\cdot,a)\|
\|\delta\widehat\theta_{S'}(\cdot,a)\| \\
&= O_p(n^{-\alpha_{\widetilde A}-\alpha_{S'}})=o_p(n^{-1/2}).
\end{aligned}
\]

It remains to control the triple-product terms. By Hölder's inequality with one \(L_\infty\) factor and two \(L_2\) factors,
\[
\mathbb E|XYZ|\leq\|X\|_\infty \|Y\| \|Z\|.
\]
Together with the uniform boundedness assumption, this gives
\[
\begin{aligned}
&\sum_{a\in\{0,1\}}\E \left[
|\delta\widehat\theta_{\widetilde A}(S,1-a)|
|\delta\widehat\theta_{S'}(S,1-a)|
|\delta\widehat\theta_R(S,a)| \right] \\
&\qquad\leq\sum_{a\in\{0,1\}}
\|\delta\widehat\theta_{\widetilde A}(\cdot,1-a)\|_\infty
\|\delta\widehat\theta_{S'}(\cdot,1-a)\|
\|\delta\widehat\theta_R(\cdot,a)\| \\
&\qquad=O_p(n^{-\alpha_{S'}-\alpha_R})=o_p(n^{-1/2}).
\end{aligned}
\]

Similarly,
\[
\begin{aligned}
&\sum_{a\in\{0,1\}}\mathbb E\!\left[
|\delta\widehat\omega(S,a)|
|\delta\widehat\theta_{\widetilde A}(S,1-a)|
|\delta\widehat\theta_{S'}(S,1-a)| \right] \\
&\qquad\leq\sum_{a\in\{0,1\}}
\|\delta\widehat\theta_{\widetilde A}(\cdot,1-a)\|_\infty
\|\delta\widehat\omega(\cdot,a)\|
\|\delta\widehat\theta_{S'}(\cdot,1-a)\| \\
&\qquad=O_p(n^{-\alpha_\omega-\alpha_{S'}})=o_p(n^{-1/2}),
\end{aligned}
\]
and
\[
\begin{aligned}
&\sum_{a\in\{0,1\}}\mathbb E\!\left[
|\delta\widehat\omega(S,a)|
|\delta\widehat\theta_{\widetilde A}(S,1-a)|
|\delta\widehat\theta_R(S,1-a)| \right] \\
&\qquad\leq\sum_{a\in\{0,1\}}
\|\delta\widehat\theta_{\widetilde A}(\cdot,1-a)\|_\infty
\|\delta\widehat\omega(\cdot,a)\|
\|\delta\widehat\theta_R(\cdot,1-a)\| \\
&\qquad=O_p(n^{-\alpha_\omega-\alpha_R})=o_p(n^{-1/2}).
\end{aligned}
\]
Hence every component of \(I_n\) is \(o_p(n^{-1/2})\). Therefore,
\[
I_n=o_p(n^{-1/2}).
\]




Combining the preceding bounds with the expansion in Theorem~\ref{thm:error} gives
\[
\sqrt n\{\widehat V(\pi)-V(\pi)\}=\sqrt n(\mathbb P_n-\mathbb P)\varphi^\pi(O;\Theta_0)+o_p(1).
\]
Since \(n=NT\), it remains to establish the limiting distribution of the leading influence-function term.

Let
\[
\xi_{i,t}=\sum_{a\in\mathcal A} g_a'(O_{i,t})\omega(S_{i,t},a)\{R_{i,t}-\theta_R(S_{i,t},a)\}
+\gamma\sum_{a\in\mathcal A}g_a(O_{i,t})\omega(S_{i,t},a)\{V^\pi(S_{i,t+1})-(\mathcal M V^\pi)(S_{i,t},a)\}.
\]

By the asymptotic linear expansion established above for the cross-fitted estimator,
\[
\widehat V(\pi)-V(\pi)=\frac{1}{1-\gamma}\frac{1}{NT}\sum_{i=1}^N\sum_{t=1}^T \xi_{i,t}+R_{NT},
\qquad R_{NT}=o_p\{(NT)^{-1/2}\}.
\]
It remains to establish a central limit theorem for
\(
(NT)^{-1/2}\sum_{i=1}^N\sum_{t=1}^T \xi_{i,t}.
\)

We apply a martingale central limit theorem \citep{mcleish1974dependent} by flattening the double index \((i,t)\) into a single index. Let \(j=(i-1)T+t\), for \(j=1,\ldots,NT\), and write \(i(j)=\lceil j/T\rceil\) and \(t(j)=j-\{i(j)-1\}T\). Define \(Z_{N,T,j}=\xi_{i(j),t(j)}/\sqrt{NT}\). Then
\[
\sum_{j=1}^{NT}Z_{N,T,j}=\frac{1}{\sqrt{NT}}\sum_{i=1}^N\sum_{t=1}^T\xi_{i,t}.
\]

Let \(\mathcal H_{i,t}\) denote the sigma-field generated by all completed trajectories before trajectory \(i\), together with the history of trajectory \(i\) up to the current hidden state-action pair \((S_{i,t},A_{i,t})\). Thus \(\mathcal H_{i,t}\) contains \((S_{i,t},A_{i,t})\), but not the current \(\widetilde A_{i,t}\), \(R_{i,t}\), or \(S_{i,t+1}\). For the flattened array, define \(\mathcal F_{N,T,j-1}=\mathcal H_{i(j),t(j)}\).

We first verify the martingale difference condition. By Lemma~\ref{lemma:indicator}, for each \(a\in\mathcal A\),
\[
\mathbb E\{g_a(O_{i,t})\mid S_{i,t},A_{i,t}\}=\mathbb E\{g_a'(O_{i,t})\mid S_{i,t},A_{i,t}\}=\mathds 1(A_{i,t}=a).
\]
Moreover, the reward and transition residuals are centered conditional on the
current hidden state-action pair: \(\mathbb E\{R_{i,t}-\theta_R(S_{i,t},A_{i,t})
\mid S_{i,t},A_{i,t}\}=0\) and
\[
\mathbb E\{V^\pi(S_{i,t+1})-(\mathcal M V^\pi)(S_{i,t},A_{i,t})\mid S_{i,t},A_{i,t}\}=0.
\]
It follows that \(\mathbb E(\xi_{i,t}\mid S_{i,t},A_{i,t})=0\). By the Markov property, conditioning on \(\mathcal H_{i,t}\) is equivalent to conditioning on
\((S_{i,t},A_{i,t})\) for the current transition, and therefore
\[
\mathbb E(\xi_{i,t}\mid \mathcal H_{i,t})=0.
\]
Equivalently, \(\mathbb E(Z_{N,T,j}\mid \mathcal F_{N,T,j-1})=0\). Hence the flattened array is a martingale difference array.

We next verify the two conditions of the martingale central limit theorem. First, suppose the hidden state-action process is stationary and ergodic, or geometrically ergodic with negligible initial transient effects, and define
\[
h(s,a)=\mathbb E(\xi_{i,t}^2\mid S_{i,t}=s,A_{i,t}=a).
\]
Then
\[
\mathbb E(\xi_{i,t}^2\mid \mathcal H_{i,t})=h(S_{i,t},A_{i,t}).
\]
By the Markov-chain law of large numbers,
\[
\frac{1}{NT}\sum_{i=1}^N\sum_{t=1}^T\mathbb E(\xi_{i,t}^2\mid \mathcal H_{i,t})
=\frac{1}{NT}\sum_{i=1}^N\sum_{t=1}^Th(S_{i,t},A_{i,t})
\xrightarrow{p}\mathbb E\{h(S,A)\}=:\sigma_\xi^2.
\]
We assume the limiting variance is nondegenerate, namely
\[
0<\sigma_\xi^2<\infty.
\]

Second, the conditional Lindeberg condition follows from Assumption~\ref{assump:regularity}. Indeed, Assumption~\ref{assump:regularity} implies that \(R_{i,t}\), \(V^\pi\), \(\omega\), \(g_a\), and \(g'_a\) are uniformly bounded, where the boundedness of \(g_a\) and \(g'_a\) follows from the separation condition. Hence there exists a constant \(C_\xi<\infty\) such that, with probability tending to one,
\[
\max_{1\le i\le N,\,1\le t\le T}|\xi_{i,t}|\le C_\xi .
\]
Therefore, for every \(\varepsilon>0\), when \(NT\) is large enough,
\[
\mathds 1\{|\xi_{i,t}|>\varepsilon\sqrt{NT}\}=0
\]
uniformly in \(i,t\), with probability tending to one. Consequently,
\[
\frac{1}{NT}\sum_{i=1}^N\sum_{t=1}^T\mathbb E\left[\xi_{i,t}^2
\mathds 1\{|\xi_{i,t}|>\varepsilon\sqrt{NT}\}\mid \mathcal H_{i,t}\right]\xrightarrow{p}0.
\]

By the conditional variance convergence,
\[
\frac{1}{NT}\sum_{i=1}^N\sum_{t=1}^T\mathbb E(\xi_{i,t}^2\mid \mathcal H_{i,t})\overset{p}{\longrightarrow}\sigma_\xi^2\in(0,\infty).
\]
Since \(Z_{N,T,j}=\xi_{i(j),t(j)}/\sqrt{NT}\), this is equivalent to
\[
\sum_{j=1}^{NT}\mathbb E(Z_{N,T,j}^2\mid \mathcal F_{N,T,j-1})\overset{p}{\longrightarrow}\sigma_\xi^2.
\]

Finally, the assumed conditional Lindeberg condition says that, for every
\(\varepsilon>0\),
\[
\frac{1}{NT}\sum_{i=1}^N\sum_{t=1}^T
\mathbb E\!\left[\xi_{i,t}^2\mathds 1\{|\xi_{i,t}|>\varepsilon\sqrt{NT}\}\mid \mathcal H_{i,t}
\right]\overset{p}{\longrightarrow}0.
\]
Equivalently,
\[
\sum_{j=1}^{NT}\mathbb E\!\left[Z_{N,T,j}^2\mathds 1\{|Z_{N,T,j}|>\varepsilon\}\mid \mathcal F_{N,T,j-1}\right]\overset{p}{\longrightarrow}0.
\]

Therefore, all conditions of the martingale central limit theorem hold, and
\[
\frac{1}{\sqrt{NT}}\sum_{i=1}^N\sum_{t=1}^T \xi_{i,t}=\sum_{j=1}^{NT}Z_{N,T,j}\overset{d}{\longrightarrow}N(0,\sigma_\xi^2).
\]
Using the expansion of \(\widehat V(\pi)\), we have
\[
\sqrt{NT}\{\widehat V(\pi)-V(\pi)\}=\frac{1}{1-\gamma}\frac{1}{\sqrt{NT}}\sum_{i=1}^N\sum_{t=1}^T \xi_{i,t}+
o_p(1).
\]
By Slutsky's theorem,
\[
\sqrt{NT}\{\widehat V(\pi)-V(\pi)\}\overset{d}{\longrightarrow}N\left(0,\frac{\sigma_\xi^2}{(1-\gamma)^2}\right).
\]
Thus \(\sigma_\pi^2=\sigma_\xi^2/(1-\gamma)^2\).

Under stationarity, \(\sigma_\xi^2=\mathbb E(\xi^2)\). Hence, if \(\varphi^\pi(O)=(1-\gamma)^{-1}\xi(O)\), then \(\sigma_\pi^2=\mathbb E\{[\varphi^\pi(O)]^2\} =\operatorname{Var}\{\varphi^\pi(O)\}\), since the influence function is mean zero. Otherwise, the general variance is defined through the conditional-variance limit above.

\end{proof}

\subsection{Proof of Proposition \ref{prop:var}}

\begin{proof}
For convenience, let $n = NT$ denote the total number of transition-level observations across all trajectories. Recall that the true, unobserved centered influence function for a given transition tuple $O_{i,t} = (S_{i,t}, \widetilde{A}_{i,t}, R_{i,t}, S_{i,t+1})$ is given by:
\begin{align*}
\varphi^\pi(O_{i,t}) = & (1-\gamma)^{-1}\sum_{a\in\mathcal A} g_a'(O_{i,t})\,\omega(S_{i,t},a)\,\Bigl[R_{i,t}-\theta_R(S_{i,t},a)\Bigr] \\
& +\frac{\gamma}{1-\gamma}\sum_{a\in\mathcal A} g_a(O_{i,t})\,\omega(S_{i,t},a)\,\Bigl[V^\pi(S_{i,t+1})-\mathcal{M}V^{\pi}(S_{i,t},a)\Bigr] - V(\pi).
\end{align*}
The true asymptotic variance is defined as $\sigma_\pi^2 = \Var\big(\varphi^\pi(O)\bigr) = \E\big[(\varphi^\pi(O))^2\bigr]$. 

Let $\widehat{\varphi}^\pi(O_{i,t})$ denote the plug-in estimated centered influence function evaluated with the cross-fitted nuisance estimators $\widehat{\Theta} = (\widehat{\theta}_R, \widehat{\mathcal{M}V^\pi}, \widehat{\omega}, \widehat{\theta}_{\widetilde A}, \widehat{\theta}_{S'})$ and the value estimate $\widehat{V}(\pi)$:
\begin{align*}
\widehat{\varphi}^\pi(O_{i,t}) = & (1-\gamma)^{-1}\sum_{a\in\mathcal A} \widehat{g}_a'(O_{i,t})\,\widehat{\omega}(S_{i,t},a)\,\Bigl[R_{i,t}-\widehat{\theta}_R(S_{i,t},a)\Bigr] \\
& +\frac{\gamma}{1-\gamma}\sum_{a\in\mathcal A} \widehat{g}_a(O_{i,t})\,\widehat{\omega}(S_{i,t},a)\,\Bigl[\widehat{V}^\pi(S_{i,t+1})-\widehat{\mathcal{M}V^{\pi}}(S_{i,t},a)\Bigr] - \widehat{V}(\pi).
\end{align*}
The plug-in empirical variance estimator $\widehat{\sigma}_\pi^2$ and the oracle empirical variance estimator $\overline{\sigma}_\pi^2$ satisfy:
\begin{align*}
\widehat{\sigma}_\pi^2 = \frac{1}{n}\sum_{i=1}^N\sum_{t=1}^T \left( \widehat{\varphi}^\pi(O_{i,t}) \right)^2, \qquad \text{and} \qquad \overline{\sigma}_\pi^2 = \frac{1}{n}\sum_{i=1}^N\sum_{t=1}^T \left( \varphi^\pi(O_{i,t}) \right)^2.
\end{align*}
To establish the consistency of our variance estimator, it suffices to show that $\widehat{\sigma}_\pi^2 - \overline{\sigma}_\pi^2 = o_p(1)$.

Using the identity $x^2 - y^2 = (x-y)^2 + 2y(x-y)$, we decompose the targeted difference pointwise:
\begin{align}
\widehat{\sigma}_\pi^2 - \overline{\sigma}_\pi^2 = \frac{1}{n}\sum_{i=1}^N\sum_{t=1}^T \left( \widehat{\varphi}^\pi(O_{i,t}) - \varphi^\pi(O_{i,t}) \right)^2 + \frac{2}{n}\sum_{i=1}^N\sum_{t=1}^T \varphi^\pi(O_{i,t}) \left( \widehat{\varphi}^\pi(O_{i,t}) - \varphi^\pi(O_{i,t}) \right). \label{eq:var_decomp}
\end{align}
By the Cauchy--Schwarz inequality, the cross-product term in \eqref{eq:var_decomp} is bounded by:
\begin{align*}
\left| \frac{2}{n}\sum_{i=1}^N\sum_{t=1}^T \varphi^\pi(O_{i,t}) \left( \widehat{\varphi}^\pi(O_{i,t}) - \varphi^\pi(O_{i,t}) \right) \right| \le 2 \sqrt{\overline{\sigma}_\pi^2} \cdot \sqrt{\frac{1}{n}\sum_{i=1}^N\sum_{t=1}^T \left( \widehat{\varphi}^\pi(O_{i,t}) - \varphi^\pi(O_{i,t}) \right)^2}.
\end{align*}
It is therefore sufficient to make both terms in \eqref{eq:var_decomp} vanish. Then we have
\begin{align}
\frac{1}{n}\sum_{i=1}^N\sum_{t=1}^T \left( \widehat{\varphi}^\pi(O_{i,t}) - \varphi^\pi(O_{i,t}) \right)^2 \le 2 \left( \widehat{V}(\pi) - V(\pi) \right)^2 + \frac{2}{n}\sum_{i=1}^N\sum_{t=1}^T \left( \phi^\pi(O_{i,t}; \widehat{\Theta}) - \phi^\pi(O_{i,t}; \Theta_0) \right)^2. \label{eq:phi_bound}
\end{align}
Under Assumption~\ref{assump:regularity}, all true and estimated nuisance functions are uniformly bounded, and their denominators are bounded away from zero by a constant $c > 0$. By the proof of Theorem \ref{thm:error}, we have:
\begin{align*}
\left| \phi^\pi(O_{i,t}; \widehat{\Theta}) - \phi^\pi(O_{i,t}; \Theta_0) \right| \lesssim \sum_{a \in \mathcal{A}} \Bigl( & |\delta \widehat{\omega}(S_{i,t},a)| + |\delta \widehat{\theta}_R(S_{i,t},a)| + |\delta \widehat{\theta}_{S'}(S_{i,t},a)| \\
& + |\delta \widehat{\theta}_{\widetilde{A}}(S_{i,t},a)| + |\delta \widehat{V}^\pi(S_{i,t+1})| + |e_M(S_{i,t},a)| \Bigr).
\end{align*}

Let $\{I_k\}_{k=1}^K$ denote the cross-fitting partition folds. For observations within a held-out fold $I_k$, the nuisance functions $\widehat{\Theta}$ are trained on independent data outside of $I_k$. Taking the conditional expectation of the squared variations inside fold $I_k$ given the training dataset $\mathcal{T}_k$ gives:
\begin{align*}
\E \left[ \frac{1}{n} \sum_{(i,t) \in I_k} \left( \phi^\pi(O_{i,t}; \widehat{\Theta}) - \phi^\pi(O_{i,t}; \Theta_0) \right)^2 \;\middle|\; \mathcal{T}_k \right] \lesssim \sum_{a \in \mathcal{A}} \Bigl( & \|\delta \widehat{\omega}\|^2 + \|\delta \widehat{\theta}_R\|^2 + \|\delta \widehat{\theta}_{S'}\|^2 \\
& + \|\delta \widehat{\theta}_{\widetilde{A}}\|^2 + \|\delta \widehat{V}^\pi\|^2 + \|e_M\|^2 \Bigr).
\end{align*}
By the convergence rates given for Theorem~\ref{thm:asymptotic_normality}, each nuisance component converges under the $L_2(P)$ norm at a rate of $O_p(n^{-\alpha_{\min}})$ with $\alpha_{\min} > 1/4$. Squaring these components yields an error bound of order $O_p(n^{-2\alpha_{\min}}) = o_p(1)$. 

Summing across the finite number of folds $K$ and applying Markov's inequality shows that the second term in \eqref{eq:phi_bound} scales as $O_p(n^{-2\alpha_{\min}}) = o_p(1)$. This immediately gives $$\frac{1}{n}\sum_{i=1}^N\sum_{t=1}^T \left( \widehat{\varphi}^\pi(O_{i,t}) - \varphi^\pi(O_{i,t}) \right)^2 \xrightarrow{p} 0.$$

Returning to our target expansion \eqref{eq:var_decomp}:
\begin{align*}
\widehat{\sigma}_\pi^2 - \overline{\sigma}_\pi^2 = o_p(1) + 2 \cdot O_p(1) \cdot o_p(1) = o_p(1).
\end{align*}
Combining this with $\overline{\sigma}_\pi^2 -\sigma_\pi^2=o_p(1)$ completes the proof.
\end{proof}

\subsection{Extension to Categorical Actions} \label{sec:multi-action}
We consider binary actions in the main paper for ease of presentation. The framework can be extended to a categorical action space
\[
\mathcal A=\{1,\ldots,K\},\qquad K>2.
\]

For each proxy variable \(X\in\{\widetilde A,R,S'\},\) choose a \(K\)-dimensional feature vector
\[
h_X(X)\in\mathbb R^K.
\]
A natural choice for the categorical action surrogate is its one-hot encoding,
\[
h_{\widetilde A}(\widetilde A)=\bigl(\mathbbm 1\{\widetilde A=1\}, \ldots, \mathbbm 1\{\widetilde A=K\}\bigr)^\top.
\]
For the reward and next-state proxies, the feature vectors \(h_R(R)\) and \(h_{S'}(S')\) should be chosen so that their conditional moment vectors distinguish the \(K\) latent action classes. For example, when appropriate, one may use
\[
h_R(R)=\bigl(1,R,\ldots,R^{K-1}\bigr)^\top,
\]
and analogous basis functions of \(S'\).

For \(X\in\{\widetilde A,R,S'\}\), define the \(K\times K\)
conditional moment matrix
\[
\Gamma_X(s)=\begin{bmatrix}
\mathbb E\{h_X(X)\mid S=s,A=1\}&\cdots&\mathbb E\{h_X(X)\mid S=s,A=K\}
\end{bmatrix}.
\]
Thus, the \(b\)th column of \(\Gamma_X(s)\) is \(\mathbb E\{h_X(X)\mid S=s,A=b\}.\)
We assume the following. 
\begin{assumption}
For each \(X\in\{\widetilde A,R,S'\}\), \(\Gamma_X(s)\) is nonsingular for every \(s\).
\end{assumption}
This is the categorical analogue of Assumption \ref{assump:latent_action_id}(ii) in the binary case. 

Let \(e_a\in\mathbb R^K\) denote the \(a\)th canonical basis vector. For \(a\in\mathcal A\), define
\[
q_a^X(X,S)=e_a^\top\Gamma_X(S)^{-1}h_X(X).
\]
Then, for every \(a,b\in\mathcal A\),
\[
\begin{aligned}
\mathbb E\{q_a^X(X,S)\mid S,A=b\}
&=e_a^\top\Gamma_X(S)^{-1}\mathbb E\{h_X(X)\mid S,A=b\}\\
&=e_a^\top\Gamma_X(S)^{-1}\Gamma_X(S)e_b\\
&=\mathbbm 1\{a=b\}.
\end{aligned}
\]
Hence, \(q_a^X\) is an observed-data representation of the latent-action indicator \(\mathbbm 1\{A=a\}\).

Define
\[
\bar g_a(O)=q_a^{\widetilde A}(\widetilde A,S)q_a^R(R,S) \quad\text{ and }\quad
\bar g_a'(O)=q_a^{\widetilde A}(\widetilde A,S)q_a^{S'}(S',S).
\]
By Assumption \ref{assump:latent_action_id}(iii),
\[
\begin{aligned}
\mathbb E\{\bar g_a(O)\mid S,A=b\}
&=\mathbb E\{q_a^{\widetilde A}(\widetilde A,S)\mid S,A=b\}\mathbb E\{q_a^R(R,S)\mid S,A=b\}\\
&=\mathbbm 1\{a=b\},
\end{aligned}
\]
and similarly,
\[
\mathbb E\{\bar g_a'(O)\mid S,A=b\}=\mathbbm 1\{a=b\}.
\]
Then a valid influence function for \(V(\pi)\) is
\[
\begin{aligned}
\varphi^\pi(O)
={}&
\mathbb E_{p_e}\{V^\pi(S)\}-V(\pi)\\
&+\frac{1}{1-\gamma}\sum_{a\in\mathcal A}\bar g_a'(O)\omega(S,a)\{R-\theta_R(S,a)\}\\
&+\frac{\gamma}{1-\gamma}\sum_{a\in\mathcal A}\bar g_a(O)\omega(S,a)\{V^\pi(S')-\mathcal MV(S,a)\}.
\end{aligned}
\]

The conditional moment matrices identify the latent action classes only up to a common permutation. An additional label-alignment condition is therefore required to associate the recovered latent classes with the action labels. 
Assumption \ref{assump:latent_action_id}(iv) for the binary case should be replaced by either of the following.
\begin{assumption}
1. $\Pr(\widetilde{A}=a\mid A=a, S=s) > \Pr(\widetilde{A}=j\mid A=a, S=s)$ for $j\ne a$; or 
2. $\Pr(\widetilde{A}=1\mid A=a, S=s)$ is strictly decreasing for $a\in\{1,2,\ldots,K \}$.
\end{assumption}

When \(K=2\), the matrix-inverse construction reduces to the scalar
contrast formulas used in the main paper.

\end{document}